\documentclass{article}

\usepackage{ifthen}

\newboolean{preprint}
\setboolean{preprint}{true}

\ifthenelse{\boolean{preprint}}{
    \usepackage{arxiv}
}{
    \usepackage[final]{neurips_2020}
}  

\usepackage[utf8]{inputenc} % allow utf-8 input
\usepackage[T1]{fontenc}    % use 8-bit T1 fonts
\usepackage[final,colorlinks=true]{hyperref}       % hyperlinks
\usepackage{url}            % simple URL typesetting

\usepackage{amsfonts}       % blackboard math symbols
\usepackage{nicefrac}       % compact symbols for 1/2, etc.
\usepackage{microtype}      % microtypography
\usepackage{footmisc}

\usepackage{graphicx}
\graphicspath{{figures/}}

\usepackage[labelformat=simple]{subcaption}

\newcommand{\sidecaption}[1]% #1 = label name
{\raisebox{\abovecaptionskip}{\begin{subfigure}[t]{1.6em}
  \caption[singlelinecheck=off]{}% do not center
  \label{#1}
\end{subfigure}}\ignorespaces}

\usepackage{booktabs}
\usepackage{multirow}
\usepackage{tabu}
\usepackage{rotating}
\usepackage{siunitx}
\usepackage{threeparttable}
\usepackage[table]{xcolor}
\usepackage{etoolbox}
\robustify\bfseries

\usepackage{tepper}

\title{Procrustean Orthogonal Sparse Hashing}

\author{%
  Mariano Tepper, Dipanjan Sengupta, Ted Willke\\
  Intel Labs\\
  \texttt{\{mariano.tepper,dipanjan.sengupta,ted.willke\}@intel.com} \\
}

\begin{document}

\maketitle

\begin{abstract}
Hashing is one of the most popular methods for similarity search because of its speed and efficiency. Dense binary hashing is prevalent in the literature. Recently, insect olfaction was shown to be structurally and functionally analogous to sparse hashing~\cite{dasguptaNeuralAlgorithmFundamental2017}. Here, we prove that this biological mechanism is the solution to a well-posed optimization problem. Furthermore, we show that orthogonality increases the accuracy of sparse hashing. Next, we present a novel method, Procrustean Orthogonal Sparse Hashing (POSH), that unifies these findings, learning an orthogonal transform from training data compatible with the sparse hashing mechanism.
We provide theoretical evidence of the shortcomings of Optimal Sparse Lifting (OSL)~\cite{liFastSimilaritySearch2018} and BioHash~\cite{ryaliBioInspiredHashingUnsupervised2020}, two related olfaction-inspired methods, and propose two new methods, Binary OSL and SphericalHash, to address these deficiencies.
We compare POSH, Binary OSL, and SphericalHash to several state-of-the-art hashing methods and provide empirical results for the superiority of the proposed methods across a wide range of standard benchmarks and parameter settings.
\end{abstract}

\section{Introduction}

Similarity search is a key technique that underpins many machine learning applications like recommender systems, visual search engines, drug discovery, and genomics.
In similarity search, given a database of high-dimensional vectors and a query vector $q$ of the same dimension, we seek the database vectors that are similar or closer to $q$, based on some similarity function, e.g., cosine similarity. 
In modern applications, these vectors represent the content of images, sounds, or bioinformatic data, extracted and summarized by deep learning systems. For example, collaborative filtering in recommendation systems uses similarity search to determine the stereotypical characteristics of a new observation (either a user or a product) by finding its best match in the training set.

The sheer volume and richness of data make similarity search a challenging problem that is both compute and memory intensive. Biological systems exhibit a remarkable ability for storing and retrieving complex data. In recent years, these systems have emerged as a source of inspiration to develop artificial counterparts. It is commonly accepted that biological systems operate as hashing methods, representing input data with high-dimensional sparse hash keys~\cite{olshausenSparseCodingOvercomplete1997a,wixtedCodingEpisodicMemory2018}. Additionally, these biological systems often learn in a completely unsupervised manner.

Our work is deeply inspired by a recent discovery in computational neuroscience: the olfactory system of the Drosophila melanogaster (commonly known as the fruit fly) behaves, both structurally and functionally, like a similarity search system, generating high-dimensional sparse hash codes~\cite{dasguptaNeuralAlgorithmFundamental2017}.

% \textbf{Contributions.}
We begin with a review of fruit fly hashing and other related work in \cref{eq:related_work}.
Then, in \cref{sec:posh}, we present a novel method for similarity search: Procrustean Orthogonal Sparse Hashing (POSH).  Differently from~\cite{dasguptaNeuralAlgorithmFundamental2017}, which uses randomly computed hashing functions, POSH learns from training data unsupervisedly. This learning procedure is formally derived from two key principles.
We first prove that the fruit-fly hashing scheme is the solution to an optimization problem with a unique closed-form solution. Then, we show that orthogonality is a desired trait that drives accuracy up. POSH combines both findings into a well-defined optimization problem whose solution can be implemented efficiently.

In \cref{sec:angle_preservation}, we analyze in depth two recently proposed algorithms, Optimal Sparse Lifting (OSL)~\cite{liFastSimilaritySearch2018} and BioHash~\cite{ryaliBioInspiredHashingUnsupervised2020}, that also take inspiration from the fruit-fly hashing method~\cite{dasguptaNeuralAlgorithmFundamental2017}. For OSL, we point out critical deficiencies and propose an optimization method that overcomes them. We also prove that the objective function used in BioHash is equivalent to spherical k-means. This discovery explains the observation in~\cite{ryaliBioInspiredHashingUnsupervised2020} that BioHash performs well in extremely sparse settings. Furthermore, we use this discovery to derive a new hashing algorithm, SphericalHash.

We also present in \cref{sec:refinement} a candidate refinement technique to boost the accuracy of hashing methods.
In \cref{sec:experiments}, we show through numerous and varied numerical experiments that SphericalHash and POSH outperform other state-of-the-art unsupervised hashing methods in terms of accuracy.
Finally, we provide some concluding remarks in \cref{sec:conclusions}.

\textbf{Mathematical notation.}
Matrices and vectors are respectively denoted by lowercase and uppercase bold letters, e.g., $\vect{A} \in \Real^{m \times n}$ and $\vect{v} \in \Real^{n}$. Individual entries are denoted by $(\vect{A})_{ij}$ and $(\vect{v})_i$. We use a subscript to denote different samples, e.g., $\vect{x}_i$.

\section{Related work}
\label{eq:related_work}

Given the plethora of work in similarity search,\footnote{\url{https://learning2hash.github.io/papers.html}} we limit our brief review to methods that are closely related to those proposed in this work.

Locality Sensitive Hashing (LSH)~\cite{gionisSimilaritySearchHigh1999,indykApproximateNearestNeighbors1998} is a seminal method that overcomes the curse of dimensionality in similarity search through the use of randomized hashing. Many variants have been proposed over the past 20 years, e.g.,~\cite{andoniNearoptimalHashingAlgorithms2008a,charikarSimilarityEstimationTechniques2002,datarLocalitysensitiveHashingScheme2004} that also scale in time and dataset size and provide theoretical guarantees on the quality of the returned nearest neighbors.
Iterative Quantization~\cite{gongIterativeQuantizationProcrustean2013} was one of the first methods that proposed to learn a hashing function from training data. Later efforts either improve ITQ's accuracy, e.g., KNNH~\cite{heKNearestNeighborsHashing2019}, or its performance, e.g., \cite{zhangFastOrthogonalProjection2015}.

Deep learning variants have also been recently proposed, some unsupervised, e.g.,~\cite{doLearningHashBinary2016,linLearningCompactBinary2016,liongDeepHashingCompact2015}, some supervised, e.g.,~\cite{junwangSemiSupervisedHashingLargeScale2012,liongDeepHashingCompact2015}. Supervised methods use class labels to learn hashing functions such that samples from the same (different) class end up having similar (dissimilar) hash codes.

\subsection{Brain-inspired sparse hashing}
\label{sec:fruit_fly}

The structure and functionality of the olfactory system of the Drosophila melanogaster has provided inspiration for hashing methods~\cite{dasguptaNeuralAlgorithmFundamental2017}.
The sparse set of neurons that fire in presence of a certain odor can be regarded as a ``tag'' or sparse code for that odor. This tag's role is to trigger behavioral responses specific to each odor, e.g., to seek the reward associated with sugar water.
%For any given odor the system's architecture can be described by three simple steps:
%\begin{enumerate}[leftmargin=1.5em]
%	\item Feedforward connections from odorant receptor neurons (ORNs) to projection neurons (PNs) ensure that the distribution of firing rates across the different PN types has approximately the same mean for all odors and all odor concentrations.
%	\item PNs are sparsely connected to Kenyon cells (KC) such that each KC receives and sums the firing rates from approximately six randomly selected PNs. There is a 40-fold increase in the number of KCs with respect to PNs.
%	\item A single inhibitory neuron suppresses the activity of KCs whose firing rate is not in the top 5\%.
%\end{enumerate}

This olfactory system can be formally described as follows~\cite{dasguptaNeuralAlgorithmFundamental2017}.
For $d$-dimensional data, choose the length $D$ of the output hash code ($D \gg d$) and the number $\alpha$ of set bits in the hash code ($\alpha \ll D$). In the biological system, $d$ is the number of olfactory receptor neurons (ORNs), $D$ is the number of Kenyon cells (KCs), and $\alpha$ is the number of KCs showing activity after inhibition (in~\cite{dasguptaNeuralAlgorithmFundamental2017}, $\alpha/D \approx 0.05$).
In the original formalization of the system~\cite{dasguptaNeuralAlgorithmFundamental2017}, the weight matrix $\mat{W} \in \Real^{D \times d}$ is created by sampling its entries from a Bernoulli distribution with parameter $p \in (0, 1)$, i.e., $(\mat{W})_{ij} \sim \operatorname{Bernoulli}(p)$. Note that other insects, such as locusts~\cite{jortnerSimpleConnectivityScheme2007}, have a similar olfactory system but with a dense matrix $\mat{W}$.
Given an input vector $\vect{x} \in \Real^{d}$, its hash code $\vect{h} \in \{ 0, 1 \}^{D}$ is created with
\begin{equation}
	\vect{h} = \operatorname{\textsc{wta}}_{\alpha} (\mat{W} \vect{x}) ,
	\label{eq:fruit_fly_hashing}
\end{equation}
where the winner-take-all non-linearity $\operatorname{\textsc{wta}}_{\alpha}$ is defined for each entry $i$ by
\begin{equation}
	(\operatorname{\textsc{wta}}_{\alpha} (\vect{y}))_i
	=
	\begin{cases}
	1 & \text{if $(\vect{y})_i$ is among the $\alpha$ largest values of $\vect{y}$;} \\
	0 & \text{otherwise.}
	\end{cases}
	\label{eq:wta_nl}
\end{equation}

The transformation in \cref{eq:fruit_fly_hashing} exhibits the following characteristics. First, Hamming distances between hash codes preserve the cosine similarity between the data points that produced them. Second, these distances can be written as
\begin{equation}
	\norm{\vect{h}_q - \vect{h}_i}{H} = 2 \alpha - 2 \operatorname{popcount}(\vect{h}_q \wedge \vect{h}_i) .
	\label{eq:hamming_fixed_bits}
\end{equation}
Since $\alpha$ is a hyperparameter, we only need to compute the non-constant second term.
This computation can be performed in $\alpha$ operations. As $\alpha \ll D$, this represents a meaningful computational advantage.
Finally, \Cref{eq:fruit_fly_hashing} substantially improves the accuracy of similarity search~\cite{dasguptaNeuralAlgorithmFundamental2017} when compared to the traditional LSH~\cite{gionisSimilaritySearchHigh1999} defined by $\operatorname{sign}(\mat{W}' \vect{x})$, with $\mat{W}' \in \Real^{D' \times d}$ and commonly $D' \leq d$.

Both Optimal Sparse Lifting (OSL)~\cite{liFastSimilaritySearch2018} and BioHash~\cite{ryaliBioInspiredHashingUnsupervised2020} propose to use learning and leverage, a posteriori, \cref{eq:fruit_fly_hashing} for similarity search. See \cref{sec:osl} and \cref{sec:biohash} for in-depth analyses of these methods. Although not based on \cref{eq:fruit_fly_hashing},~\cite{masciSparseSimilaritypreservingHashing2014} also proposes the use of sparse hash codes.

\section{Procrustean Orthogonal Sparse Hashing}
\label{sec:posh}

Motivated by the appealing features of the hashing scheme presented above, we ask and answer in the affirmative the question: Can we improve the hashing scheme by replacing a randomly sampled matrix $\mat{W}$ by a matrix that is specifically constructed for a given database?

The first ingredient in our approach is casting \cref{eq:fruit_fly_hashing} as the solution to an optimization problem.\footnote{Proofs are in \cref{sec:posh_additional}.\label{posh_proofs}}

\begin{restatable}{proposition}{fruitflyoptimization}
	Let $\operatorname{\textsc{wta}}_{\alpha}$ be defined as in \cref{eq:fruit_fly_hashing,eq:wta_nl}.
	The operation $\vect{h} = \operatorname{\textsc{wta}}_{\alpha} (\mat{W} \vect{x})$ is the unique solution to the optimization problem
	\begin{equation}
	\min_{\vect{h}} \norm{ \vect{h} - \mat{W} \vect{x} }{2}^2
	\quad\text{s.t.}\quad
	\begin{gathered}
	\vect{h} \in \{ 0, 1 \}^{D} ,\
	\transpose{\vect{1}} \vect{h}  = \alpha .
	\end{gathered}
	\end{equation}
	\label{theorem:fruit_fly_optimization}
\end{restatable}
 
%\proof{
%	Let $\vect{y} = \mat{W} \vect{x}$. Expanding the loss $\norm{ \vect{h} - \vect{y} }{2}^2$, removing the constant term, and using
%%	\begin{equation}
%%		\transpose{\vect{h}} \vect{h} = \sum_i (\vect{h})_i^2 = \sum_i (\vect{h})_i = \transpose{\vect{1}} \vect{h} = \alpha ,
%%	\end{equation}
%	$\transpose{\vect{h}} \vect{h} = \sum_i (\vect{h})_i^2 = \sum_i (\vect{h})_i = \transpose{\vect{1}} \vect{h} = \alpha$,
%	the problem becomes
%	\begin{equation}
%	\max_{\vect{h}} \transpose{\vect{h}} \vect{y}
%	\quad\text{s.t.}\quad
%	\begin{gathered}
%	\vect{h} \in \{ 0, 1 \}^{D} , \ 
%	\transpose{\vect{1}} \vect{h}  = \alpha .
%	\end{gathered}
%	\end{equation}
%	To maximize the quantity $\transpose{\vect{h}} \vect{y} = \sum_i (\vect{h})_i (\vect{y})_i$ under the specified constraints, we must place the allocated number $\alpha$ of ones in $\vect{h}$ where it counts the most, i.e., the $\alpha$ largest entries of $\vect{y}$.
%	\qedhere
%}

The next important ingredient is adding an orthogonality constraint on $\mat{W}$.\footref{posh_proofs}

\begin{restatable}{proposition}{fruitflyorthogonality}
	Let $\mat{W}$ be a binary matrix with entries sampled from a Bernoulli distribution with parameter $p \in (0, 1)$. The diagonal and off-diagonal entries of $\transpose{\mat{W}} \mat{W}$ follow Binomial distributions with parameters $p$ and $p^2$, respectively. Their first moments are:
	\begin{subequations}
	\begin{align}
		\expectation{(\transpose{\mat{W}} \mat{W})_{ii}} &= D p , &
		\variance{(\transpose{\mat{W}} \mat{W})_{ii}} &= D p (1 - p) , \\
		(\forall i \neq j)\ \expectation{(\transpose{\mat{W}} \mat{W})_{ij}} &= D p^2 , &
		(\forall i \neq j)\ \variance{(\transpose{\mat{W}} \mat{W})_{ij}} &= D p^2 (1 - p^2) .
	\end{align}
	\end{subequations}
	\label{theorem:fruit_fly_orthogonality}
\end{restatable}

%\proof{
%	Given two vectors $\vect{v}_1$ and $\vect{v}_2$ such that $(\vect{v}_1)_i, (\vect{v}_2)_i \sim \operatorname{Bernoulli}(p)$, we have that $(\vect{v}_1)_i^2 \sim \operatorname{Bernoulli}(p)$ and $(\vect{v}_1)_i \cdot (\vect{v}_2)_i \sim \operatorname{Bernoulli}(p^2)$. Their sum follows a Binomial distribution.
%	\qedhere
%}

Now, in expectation $\transpose{\mat{W}} \mat{W} = D p \mat{I} + D p^2 \left( \mat{E} - \mat{I} \right)$, see \cref{fig:fruit_fly_orthogonality}. This approximate orthogonality is one of the key ingredients that make \cref{eq:fruit_fly_hashing} good for similarity search, analogous to the use of a matrix with Gaussian entries in compressed sensing. Orthogonality is key because it makes the projection $\mat{W} \vect{x}$ invertible, i.e., $\transpose{\mat{W}} \mat{W} \vect{x} = \vect{x}$, ensuring that no information is lost in this step.

Furthermore, we find that replacing the sparse binary matrix by an orthogonal matrix (obtained by sampling its entries from a Gaussian distribution and orthogonalizing its columns) provides consistent improvements in accuracy, see \cref{fig:sparse_vs_orthogonal}. Thus, we consider orthogonality to be an important ingredient in our machine-learning-driven solution.

\begin{figure}[t]
	\centering
	\sidecaption{fig:fruit_fly_orthogonality}
	\hspace{0.3em}
    \raisebox{-0.9\height}{\includegraphics[width=0.4\linewidth]{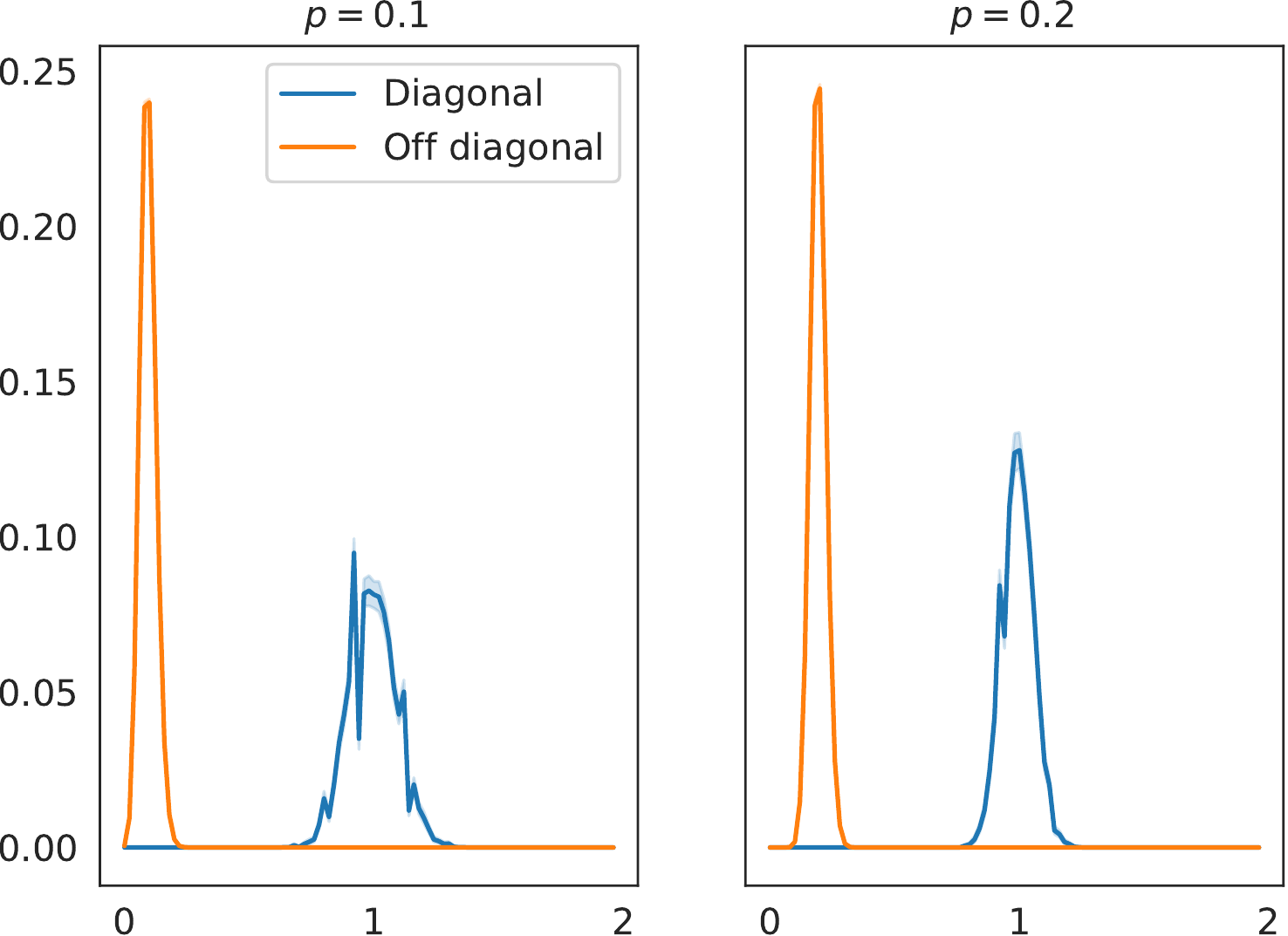}}
	\hfill%
	\sidecaption{fig:sparse_vs_orthogonal}
	\hspace{0.3em}
    \raisebox{-0.9\height}{\includegraphics[width=0.45\linewidth]{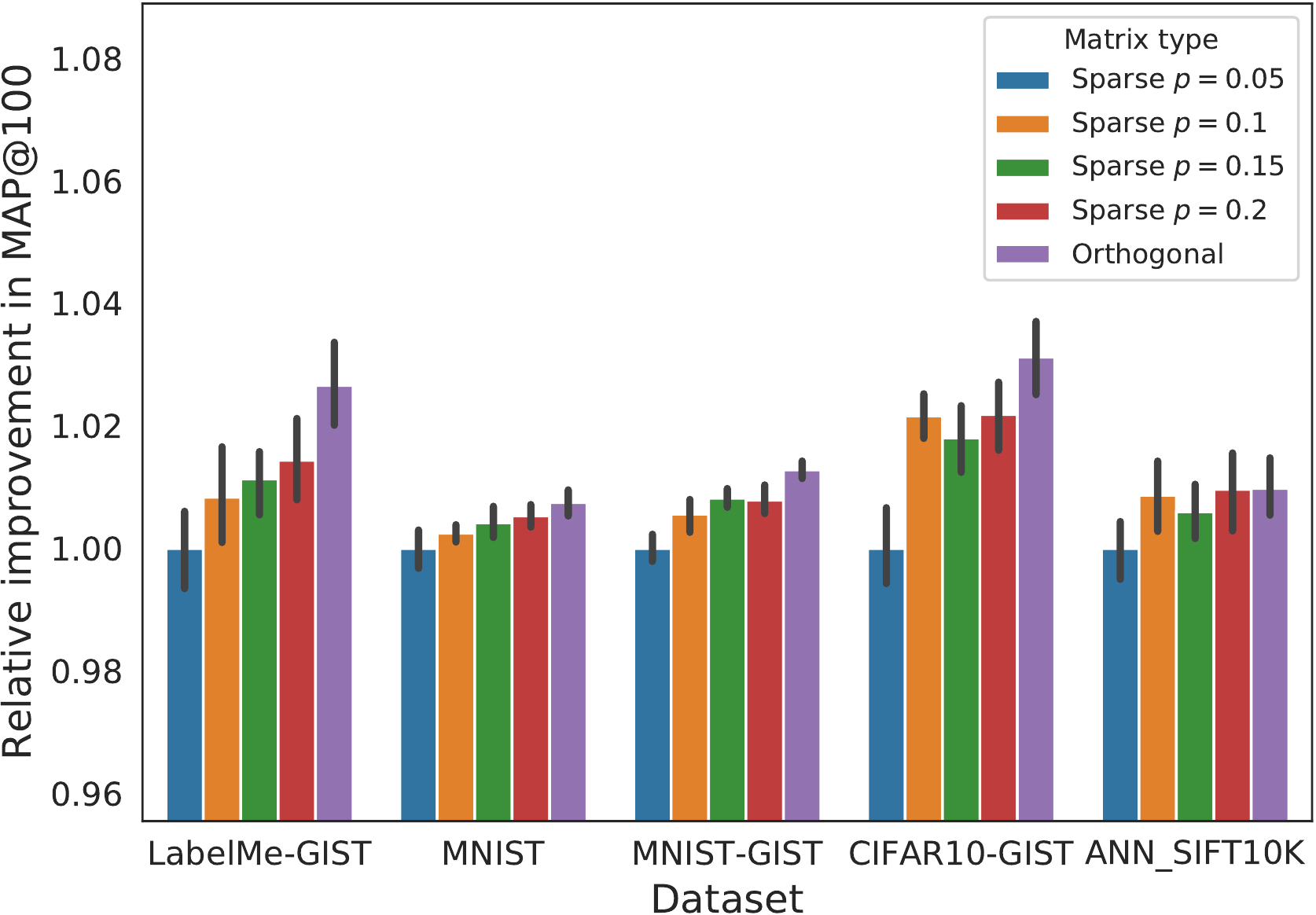}}
	
	\caption{%
		\protect\subref{fig:fruit_fly_orthogonality} Distribution of entries of the matrix $(Dp)^{-1} \transpose{\mat{W}} \mat{W}$ over 100 randomly generated matrices $\mat{W} \in \Real^{D \times d}$, where each entry is sampled from a Bernoulli distribution with parameter $p$. The shaded area represents the 95\% confidence interval (barely visible for the off diagonal values).
		\protect\subref{fig:sparse_vs_orthogonal} Comparative similarity search accuracy using matrices $\mat{W}$ with Bernoulli distributed entries versus using a dense orthogonal matrix $\mat{W}$ (bars represent the 95\% confidence interval). Orthogonality provides a consistent, albeit small, accuracy advantage.
	}
	\label{fig:orthogonality}
\end{figure}

We propose to take the problem in \cref{theorem:fruit_fly_optimization} and incorporate the orthogonality constraint. Given samples $\{ \vect{x}_i \in \Real^d \}_{i=1}^{n}$, we obtain a joint optimization for the hash codes and the weight matrix,
\begin{equation}
	\min_{\mat{W}, \{ \vect{h}_i \}_{i=1}^{n}} \sum_{i=1}^{n} \norm{ \vect{h}_i - \mat{W} \vect{x}_i }{2}^2
	\quad\text{s.t.}\quad
	\begin{gathered}
		\transpose{\mat{W}} \mat{W} = \mat{I} , \
		\vect{h}_i \in \{ 0, 1 \}^{D} , \
		\transpose{\vect{1}} \vect{h}_i  = \alpha .
	\end{gathered}
	\tag{POSH}
	\label[problem]{prob:posh}
\end{equation}
This problem shares its basic structure with~\cite{gongIterativeQuantizationProcrustean2013}, with the important differences of $D \gg d$ and the sparsity constraint.
We use an alternating optimization scheme to find a solution to \ref{prob:posh}.
First, given $\mat{W}$, we solve the subproblem for $\{ \vect{h}_i \}_{i=1}^{n}$.
%	\begin{equation}
%	\min_{\{ \vect{h}_i \}_{i=1}^{n}} \sum_{i=1}^{n} \norm{ \vect{h}_i - \mat{W} \vect{x}_i }{2}^2
%	\quad\text{s.t.}\quad
%	\begin{gathered}
%	\vect{h}_i \in \{ 0, 1 \}^{D} , \\
%	\transpose{\vect{1}} \vect{h}_i  = \alpha .
%	\end{gathered}
%	\end{equation}
For each $i$, \ref{prob:posh} decomposes into individual problems
\begin{equation}
	\min_{\vect{h}_i} \norm{ \vect{h}_i - \mat{W} \vect{x}_i }{2}^2
	\quad\text{s.t.}\quad
	\begin{gathered}
	\vect{h}_i \in \{ 0, 1 \}^{D} , \
	\transpose{\vect{1}} \vect{h}_i  = \alpha .
	\end{gathered}
	\label[problem]{prob:posh_H}
\end{equation}
The solution to this problem is $\vect{h}_i = \operatorname{\textsc{wta}}_{\alpha} (\mat{W} \vect{x}_i)$ per \cref{theorem:fruit_fly_optimization}.
Then, given $\{ \vect{h}_i \}_{i=1}^{n}$, we solve
\begin{equation}
	\min_{\mat{W}} \sum_{i=1}^{n} \norm{ \vect{h}_i - \mat{W} \vect{x}_i }{2}^2
	\quad\text{s.t.}\quad
	\begin{gathered}
		\transpose{\mat{W}} \mat{W} = \mat{I} .
	\end{gathered}
	\label[problem]{prob:posh_W}
\end{equation}
This is the classical orthogonal Procrustes problem, whose solution is $\mat{W} = \mat{U} \transpose{\mat{V}}$ where $\mat{U} \mat{S} \transpose{\mat{V}}$ is the singular value decomposition (SVD) of the matrix $\mat{M} = \sum_i \vect{h}_i \transpose{\vect{x}_i}$.

The POSH subproblems (\ref{prob:posh_H}) and (\ref{prob:posh_W}) are amenable to a stochastic (i.e., mini-batch) treatment.
This treatment is inspired by online dictionary learning techniques~\cite{mairalOnlineLearningMatrix2010}. Because of space constraints, the presentation of POSH's pseudocode is deferred to \cref{sec:posh_additional}.

%\begin{algorithm2e}[t]
%	\caption{Procrustean Orthogonal Sparse Hashing}
%	\label{algo:posh_incremental}
%	
%	\begin{small}
%	\SetKwInOut{Input}{input}
%	\SetKwInOut{Output}{output}
%	
%	\Input{Data $\{ \vect{x}_i \in \Real^d \}_{i=1}^{n}$, output dimension $D$, sparsity level $\alpha$, number $N$ of training epochs, mini-batch size $N_{\text{mini-batch}}$.}
%	\Output{Weight matrix $\mat{W} \in \Real^{D \times d}$.}
%	
%	Create a matrix $\mat{W}_0 \in \Real^{D \times d}$ by sampling its entries from the standard normal distribution\;
%	$\mat{W} \gets \mat{U} \transpose{\mat{V}}$, where $\mat{U} \mat{S} \transpose{\mat{V}}$ is the SVD of $\mat{W}_0$\;
%	
%	$\mat{M} \gets \mat{W}$\;
%
%	\ForEach{epochs $t = 1 \dots N$}{
%		\For{$i = 1 \dots n$}{
%			$\vect{s}_i \gets \operatorname{\textsc{wta}}_{\alpha} \left( \mat{W} \vect{x}_i \right)$\;
%			$\mat{M} \gets \mat{M} + \vect{s}_i \transpose{\vect{x}_i}$\;
%			
%			\If{$i \bmod{N_{\text{mini-batch}}} = 0$}{
%				$\mat{W} \gets \mat{U} \transpose{\mat{V}}$, where $\mat{U} \mat{S} \transpose{\mat{V}}$ is the SVD of $\mat{M}$\;
%			}
%		}
%	}
%	\end{small}
%\end{algorithm2e}

\section{A new perspective on Fruit Fly inspired hashing methods}
\label{sec:angle_preservation}

Following arguments similar to those in \cite{charikarSimilarityEstimationTechniques2002}, \cref{eq:fruit_fly_hashing} preserves cosines similarities, i.e., $\cos \theta(\vect{x}_i, \vect{x}_j) = \hat{\vect{x}}_i \cdot \hat{\vect{x}}_j \approx \vect{h}_i \cdot \vect{h}_j$,
% \begin{equation}
% 	\cos \theta(\vect{x}_i, \vect{x}_j) = 
% 	\hat{\vect{x}}_i \cdot \hat{\vect{x}}_j
% 	\approx
% 	\vect{h}_i \cdot \vect{h}_j ,
% \end{equation}
where $\theta$ denotes the angle between its arguments and $\hat{\vect{x}} = \vect{x} / \norm{\vect{x}}{2}$.
This is because the norm of any input $\vect{x}$ cannot be retained, as $(\forall \beta \neq 0)\, \operatorname{\textsc{wta}}_{\alpha} \left( \beta \vect{x} \right) = \operatorname{\textsc{wta}}_{\alpha} \left( \vect{x} \right)$.
However, not every matrix $\mat{W}$ possesses this property. Next, we will present three alternatives to build angle-preserving hashing methods, the last two of which will be extensively discussed in \cref{sec:biohash,sec:osl}.

First, we have the strategy followed by POSH: constrain the matrix $\mat{W}$ to be orthogonal ($\transpose{\mat{W}} \mat{W} = \mat{I}$). Orthogonality is the strictest constraint we can impose since it preserves angles, distances and norms.

Second, we can quantize the angles in any dataset $\{ \vect{x}_i \}_{i=1}^n$, approximating each point with one of $K$ representatives $\{ \vect{c}_k \}_{k=1}^K$ such that $\norm{\vect{c}_k}{2} = 1$. This quantization can be learned using spherical k-means. Then, $\cos \theta(\vect{x}_i, \vect{x}_j) \approx \vect{c}_{k(i)} \cdot \vect{c}_{k(j)}$, where $k(i)$ and $k(j)$ represent the index of closest representative for $i$ and $j$, respectively. As we will formally show in \cref{sec:biohash}, this is the strategy followed by BioHash~\cite{ryaliBioInspiredHashingUnsupervised2020} \emph{during training}. However, a different strategy is followed \emph{during hashing}: each point $i$ is associated with a set $\varkappa(i)$ of $\alpha$ representatives, resulting in $\cos \theta(\vect{x}_i, \vect{x}_j) \approx ( \sum_{k \in \varkappa(i)} \vect{c}_{k} ) \cdot ( \sum_{k' \in \varkappa(k)} \vect{c}_{k'} )$. This approximation loses its accuracy as $\alpha$ increases. See \cref{sec:biohash} for more details.

Third, we can learn to preserve angles. Given a dataset $\{ \hat{\vect{x}}_i \}_{i=1}^n$ of unit-norm vectors, we can learn $\{ \vect{y}_i \}_{i=1}^n$, with $	\vect{y}_i \in \{ 0, 1 \}^{D} , \transpose{\vect{1}} \vect{y}_i  = \alpha$, such that $\sum_{ij} (\hat{\vect{x}}_i \cdot \hat{\vect{x}}_j - \vect{y}_i \cdot \vect{y}_j )^2$ is minimized. Then, we can find the matrix $\mat{W}$ that best transforms $\hat{\vect{x}}_i$ into $\vect{y}_i$. This is the strategy chosen by Optimal Sparse Lifting (OSL)~\cite{liFastSimilaritySearch2018}. It is worth noting that this optimization is much more computationally demanding than the two previous ones as it involves $n \times n$ matrices. The OSL optimization method in~\cite{liFastSimilaritySearch2018} is broken and in \cref{sec:osl} we re-formulate it in a sound way.

\subsection{Interpreting BioHash as spherical k-means}
\label{sec:biohash}

Given data $\{ \vect{x}_i \in \Real^d \}_{i=1}^{n}$, BioHash~\cite{ryaliBioInspiredHashingUnsupervised2020} learns a matrix of weights $\mat{W} \in \Real^{D \times d}$ by minimizing the problem (we use the values $p=2$ and $\Delta = 0$ as in~\cite{ryaliBioInspiredHashingUnsupervised2020}, which simplify the formulation)
\begin{equation}
	\max_{\mat{W}}
	\sum_{i=1}^{n} \sum_{j=1}^{D} 
	\indicator{ j = \argmax_l \transpose{\vect{w}}_l \vect{x}_i ) }
	\frac{
		\transpose{\vect{w}}_j \vect{x}_i
	}{
		\norm{\vect{w}_j}{2}
	}
	\label[problem]{eq:biohash_loss}
\end{equation}
where the $\operatorname{Rank}$ returns the indices that would sort the array, in decreasing order, $\vect{w}_j$ denotes the $j$-th row of $\mat{W}$, and $\indicator{\cdot}$ denotes the indicator function.
%\begin{equation}
%	g(u) =
%	\begin{cases}
%		1 &\text{if } u = 1 ;\\
%		0 & \text{otherwise.}
%	\end{cases}
%\end{equation}
The BioHash~\cite{ryaliBioInspiredHashingUnsupervised2020} learning dynamic, given learning rate $\tau$, is
\begin{equation}
	\tau \frac{d (\vect{w}_j)_k}{dt}
	=
	\indicator{ j = \argmax_l \transpose{\vect{w}}_l \vect{x}_i ) }
	\left( (\vect{x}_i)_k - \transpose{\vect{w}}_j \vect{x}_i (\vect{w}_j)_k \right) .
	\label{eq:biohash_dynamics}
\end{equation}
Under this dynamic, the rows $\vect{w}_j$ converge to have unit norm~\cite{ryaliBioInspiredHashingUnsupervised2020}.
The connection with \cref{eq:fruit_fly_hashing} is heuristic, using it post hoc  for hashing once the optimization is complete.

Next, we cast \cref{eq:biohash_loss} into a more familiar form in the machine learning literature: spherical k-means. See, for example, its use in unsupervised feature extraction in~\cite{coatesLearningFeatureRepresentations2012}.

\begin{restatable}{proposition}{biohashspherical}
	\cref{eq:biohash_loss} is equivalent to spherical k-means, defined as follows:
	\begin{equation}
		\min_{\mat{W} , \{ \vect{s}_i \}_{i=1}^{n}}
		\sum_{i=1}^{n}
		\norm{ \transpose{\mat{W}} \vect{s}_i -  \vect{x}_i }{2}^2
		\quad\text{s.t.}\quad
		\begin{gathered}
			(\forall i)\ \vect{s}_i \in \{ 0, 1 \}^D ,\
			(\forall i)\  \norm{\vect{s}_i}{0} \leq 1 ,\
			\norm{\vect{w}_j}{2} = 1.
		\end{gathered}
		\label[problem]{eq:spherical_kmeans}
	\end{equation}
	\label{prop:biohash_spherical}
\end{restatable}

%\proof{
%	From the constraints in \cref{eq:spherical_kmeans}, we have $\transpose{\vect{s}}_i \mat{W} \transpose{\mat{W}} \vect{s}_i = 1$.
%	The problem for $\mat{W}$ becomes
%	\begin{equation}
%		\max_{\mat{W}}
%		\sum_{i=1}^{n}
%		\transpose{\vect{s}_i} \mat{W} \vect{x}_i
%		\quad\text{s.t.}\quad
%		\norm{\vect{w}_j}{2} = 1 .
%		\label[problem]{eq:spherical_kmeans_W}
%	\end{equation}
%	The problem for each $\vect{s}_i$ becomes
%	\begin{equation}
%		\max_{\vect{s}_i}
%		\transpose{\vect{s}_i} \mat{W} \vect{x}_i
%		\quad\text{s.t.}\quad
%		\begin{gathered}
%			\norm{\vect{s}_i}{0} \in \{ 0, 1 \}^D ,\
%			\norm{\vect{s}_i}{0} \leq 1,\\
%		\end{gathered}
%	\end{equation}
%	and its solution given by
%	\begin{equation}
%		\vect{s}_i^*
%		=
%		\operatorname{\textsc{wta}}_{1} \left( \mat{W} \vect{x}_i \right) ,
%%		=
%%		\begin{cases}
%%			1 & \text{if } j = \argmax_l \transpose{\vect{w}}_l \vect{x}_i ;\\
%%			0 & \text{otherwise.}
%%		\end{cases}
%		\label{eq:spherical_kmeans_S_solution}
%	\end{equation}
%	where function $\operatorname{\textsc{wta}}$ is defined in \cref{eq:wta_nl}.
%	Plugging this solution in \cref{eq:spherical_kmeans_W}, we get
%	\begin{equation}
%		\max_{\mat{W}}
%		\sum_{i=1}^{n} \sum_{j=1}^{D} 
%		\indicator{j = \argmax_{l} \transpose{\vect{w}}_l \vect{x}_i}
%		\transpose{\vect{w}}_j \vect{x}_i
%		\quad\text{s.t.}\quad
%		\norm{\vect{w}_j}{2} = 1 ,
%	\end{equation}
%	which is equivalent to \cref{eq:biohash_loss}
%	\qedhere
%}

The BioHash learning algorithm (i.e., Eq.~(1) in~\cite{ryaliBioInspiredHashingUnsupervised2020}) is relatively slow.
Fortunately, \cref{prop:biohash_spherical} offers a path to an efficient and fast alternative. We use a spherical k-means~\cite{coatesLearningFeatureRepresentations2012} to learn $\mat{W}$ and then use \cref{eq:fruit_fly_hashing} during hashing.\footnote{SphericalHash is formally described in \cref{sec:spherical_hash}.} We name this novel combination SphericalHash.

Now that we have two similar algorithms, BioHash and SphericalHash, that use the same hashing scheme and optimize the \emph{same} objective function (\cref{eq:biohash_loss}), we can contrast their optimization algorithms. SphericalHash attains significantly higher objective value than the dynamics in~\cref{eq:biohash_dynamics}  do (e.g., see \cref{fig:bio_vs_spherical}). However, a higher training objective value does not positively correlate with better similarity search accuracy, as can be observed for both algorithms in \cref{fig:bio_vs_spherical}. We hypothesize that the source of this misalignment is the different values of $\alpha$ used when training ($\alpha=1$) and hashing ($\alpha > 1$). In agreement with our thesis, when $\alpha$ is relatively small ($\alpha  \leq 32$) during hashing, guaranteeing an alignment with training, SphericalHash performs very well. In this regime, SphericalHash significantly outperforms BioHash. Finally, we remark that this discrepancy between training and hashing is not present in POSH, which uses a consistent value throughout (see \cref{fig:bio_vs_spherical}).

\begin{figure}
	\centering
	\includegraphics[width=0.95\textwidth]{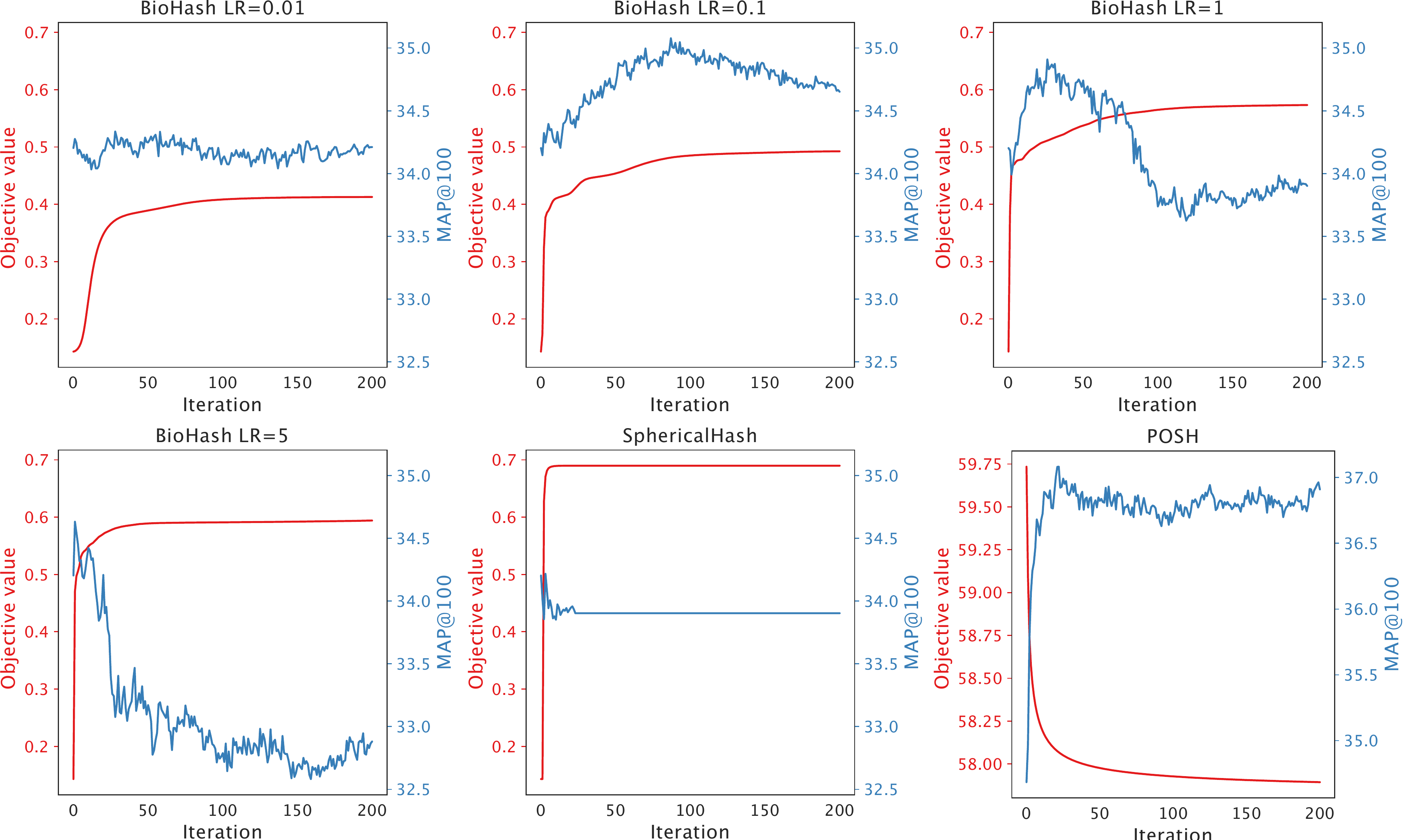}
	
	\caption{Similarity search accuracy (in blue, higher is better) is not positively correlated with the BioHash objective function (in red, higher is better) when $\alpha \gg 1$ (here, $D=1024$ and $\alpha=64$). As we increase the initial learning rate (LR) used in the original BioHash iterations (first four plots), they converge to an increasingly better objective value. However, with this training improvement, we observe worse accuracy. Running just a few optimization iterations seems to improve things, but running them all the way is clearly detrimental. SphericalHash optimizes the same objective and obtains a significantly better objective value. Again, the objective function and similarity search accuracy appear to be unrelated. For POSH, a better objective function value (in red, lower is better) correlates with improved similarity search accuracy (in blue). We use a subset of CIFAR10-GIST where, for each class, we randomly sample 1000 and 100 target and query vectors, respectively. 
	}
	\label{fig:bio_vs_spherical}
\end{figure}

\subsection{Optimal Sparse Lifting: an alternative optimization}
\label{sec:osl}

To preserve angular similarities, OSL~\cite{liFastSimilaritySearch2018} starts by learning hash codes and then learns an appropriate $\mat{W}$ (see \cref{sec:angle_preservation}). Given data $\mat{X} \in \Real^{d \times n}$, with unit-norm columns, the first problem is~\cite{liFastSimilaritySearch2018}
\begin{equation}
	\mat{Y}^* = \argmin_{\mat{Y}} \tfrac{1}{2} \norm{\transpose{\mat{X}} \mat{X} - \mat{Y} \mat{Y}}{F}^2 + \gamma \norm{\mat{Y}}{p}
	\quad\text{s.t.}\quad
	\begin{gathered}
		\transpose{\vect{Y}} \mat{1} = \alpha \vect{1} ,\
		\mat{0} \leq \mat{Y} \leq \mat{1} .
	\end{gathered}
	\label[problem]{eq:osl}
\end{equation}
where  the $\ell_p$ pseudo-norm ($0 < p < 1$) is used to promote sparsity.
The second problem is~\cite{liFastSimilaritySearch2018}
\begin{equation}
	\mat{W}^* = \argmin_{\mat{W}} \tfrac{1}{2} \norm{\mat{W} \mat{X} - \mat{Y}^*}{F}^2 + \beta \norm{\mat{W}}{p}
	\quad\text{s.t.}\quad
	\mat{W} \vect{1} = c \vect{1} ,\ \mat{0} \leq \mat{W} \leq \mat{1} .
	\label[problem]{eq:oso}
\end{equation}
In~\cite{liFastSimilaritySearch2018}, the connection to \cref{eq:fruit_fly_hashing} is heuristic, and it is used post hoc for hashing once the optimization is complete. We note that the values of $\gamma$, $\beta$, and $p$ are not specified in~\cite{liFastSimilaritySearch2018}, making their results impossible to reproduce.

The authors of~\cite{liFastSimilaritySearch2018} propose to use the Frank-Wolfe (a.k.a, conditional gradient)~\cite{frankAlgorithmQuadraticProgramming1956a,jaggiRevisitingFrankWolfeProjectionfree2013a} method to solve these two problems. However, in order to apply Frank-Wolfe, the objective function needs to be differentiable. The $\ell_p$ norms in \cref{eq:osl,eq:oso} are non-differentiable. As such, the optimization method in~\cite[Algorithm 1]{liFastSimilaritySearch2018} is not correct.

Given the shortcomings and uncertainties in the original OSL formulation, we propose a variant that is theoretically sound and, in practice, does not have \emph{critical} hyperparameters. The proposed technique starts by optimizing a problem similar to \cref{eq:osl} but eliminates the $\ell_p$ terms and returns to the binary constraints stemming from \cref{theorem:fruit_fly_optimization}. We define the problem
\begin{equation}
	\mat{Y}^* = \argmin_{\mat{Y}} \tfrac{1}{4} \norm{\transpose{\mat{X}} \mat{X} - \transpose{\mat{Y}} \mat{Y}}{F}^2
	\quad\text{s.t.}\quad
	\begin{gathered}
		\transpose{\vect{Y}} \mat{1} = \alpha \vect{1} ,\
		\mat{Y} \in \{ 0, 1 \}^{D \times n} .
	\end{gathered}
	\label[problem]{eq:osl_binary}
\end{equation}
Because of space constraints, we defer the optimization algorithm to \cref{sec:OSL_optimization}.
After finding $\mat{Y}^*$ and following~\cite{liFastSimilaritySearch2018}, \cref{eq:oso} would be solved to find $\mat{W}$.
Instead, we solve
\begin{equation}
	\mat{W}^* = \argmin_{\mat{W}} \tfrac{1}{2} \norm{\mat{W} \mat{X} - \mat{Y}_{t}}{F}^2 .
	\label[problem]{eq:oso_binary_relaxed}
\end{equation}
Removing the sparsity constraints from $\mat{W}$ simplifies the optimization by enlarging the feasible set. As such, solving the unconstrained \cref{eq:oso_binary_relaxed} provides an upper bound for the accuracy of the proposed OSL variant. Additionally, a dense $\mat{W}$ does not preclude biological plausibility~\cite{jortnerSimpleConnectivityScheme2007}. The proposed approach, Binary Optimal Sparse Lifting (BOSL), is described in \cref{sec:OSL_optimization}.

\begin{figure}[t]
	\centering
	\begin{small}
	\begin{tabu} to 0.95\textwidth {*{3}{ @{\hspace{0pt}} X[c,m] @{\hspace{0pt}} }}
		MNIST & CIFAR10-GIST & LabelMe-12-50K-GIST \\
		\includegraphics[width=\linewidth]{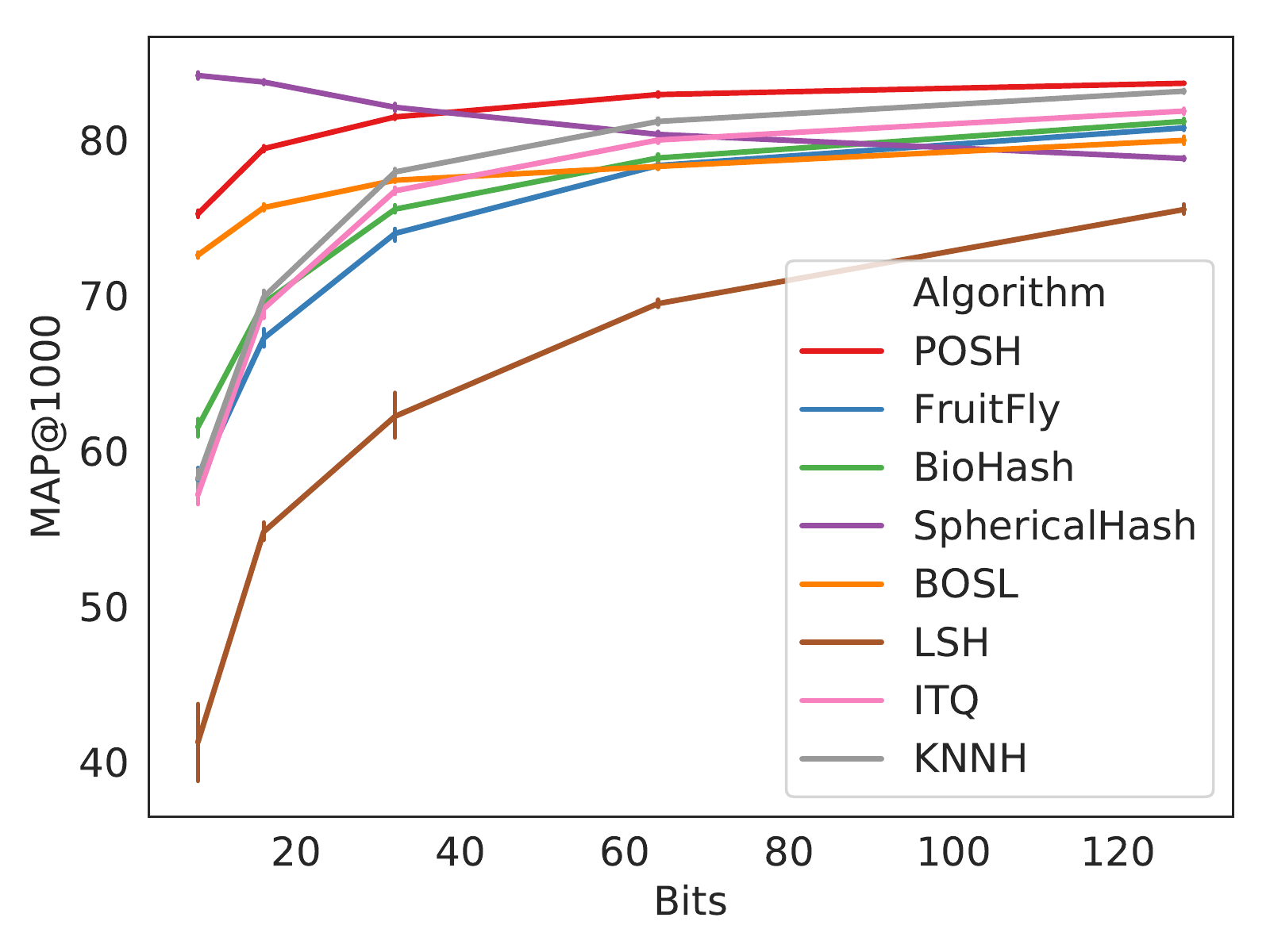} &
		\includegraphics[width=\linewidth]{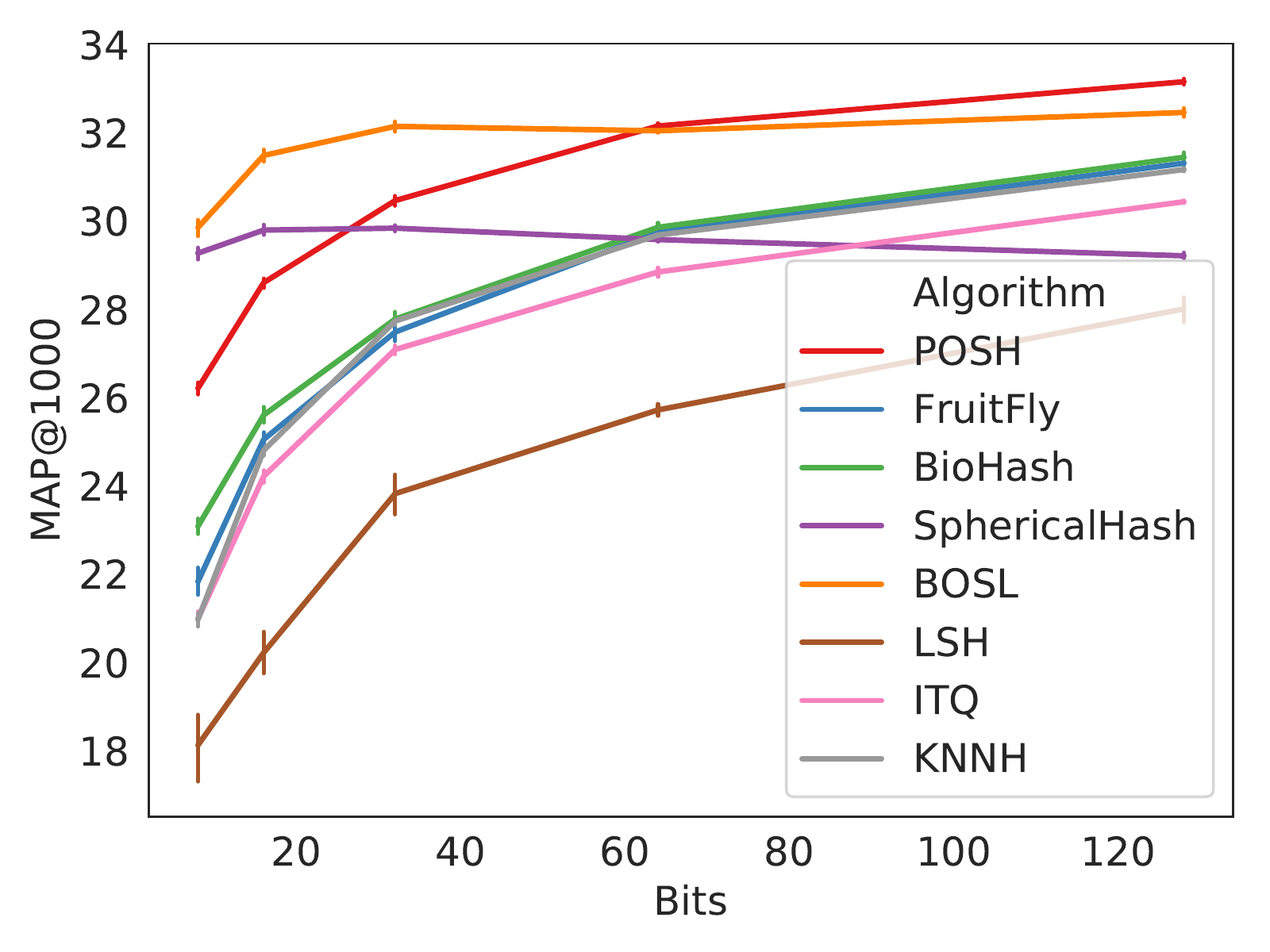} &
		\includegraphics[width=\linewidth]{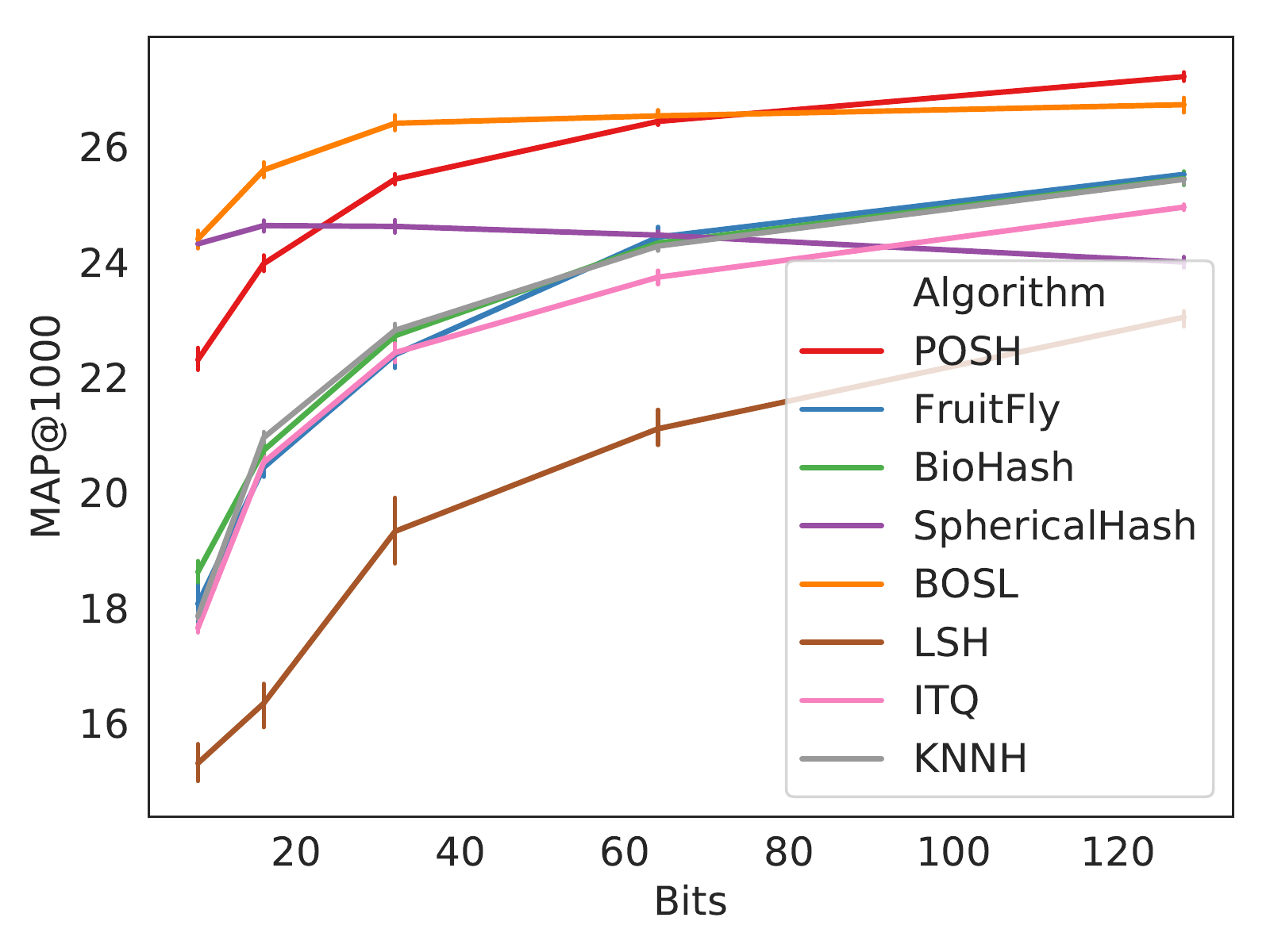} \\
		MNIST-GIST & CIFAR10-VGG & LabelMe-12-50K-VGG \\
		\includegraphics[width=\linewidth]{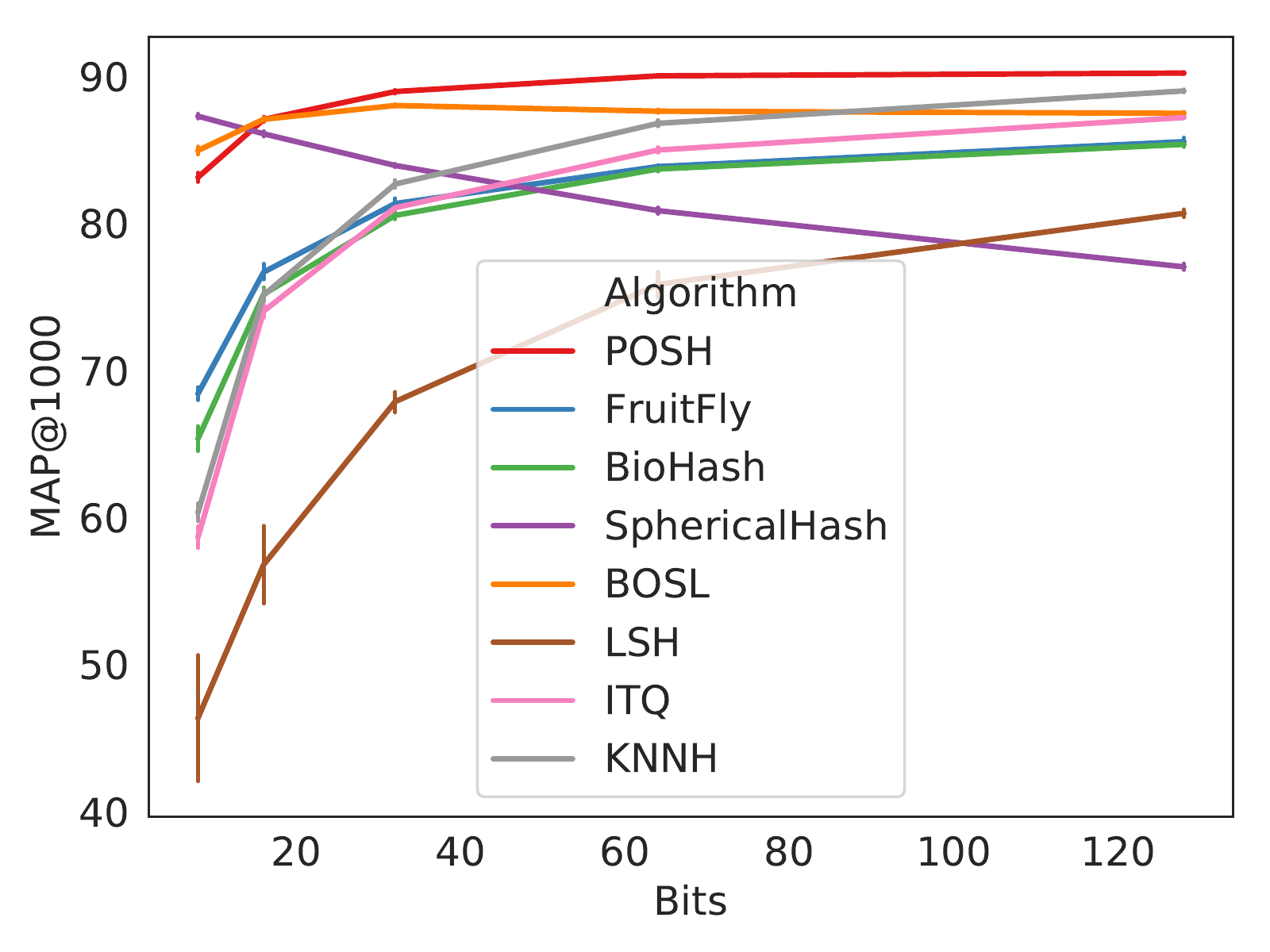} &
		\includegraphics[width=\linewidth]{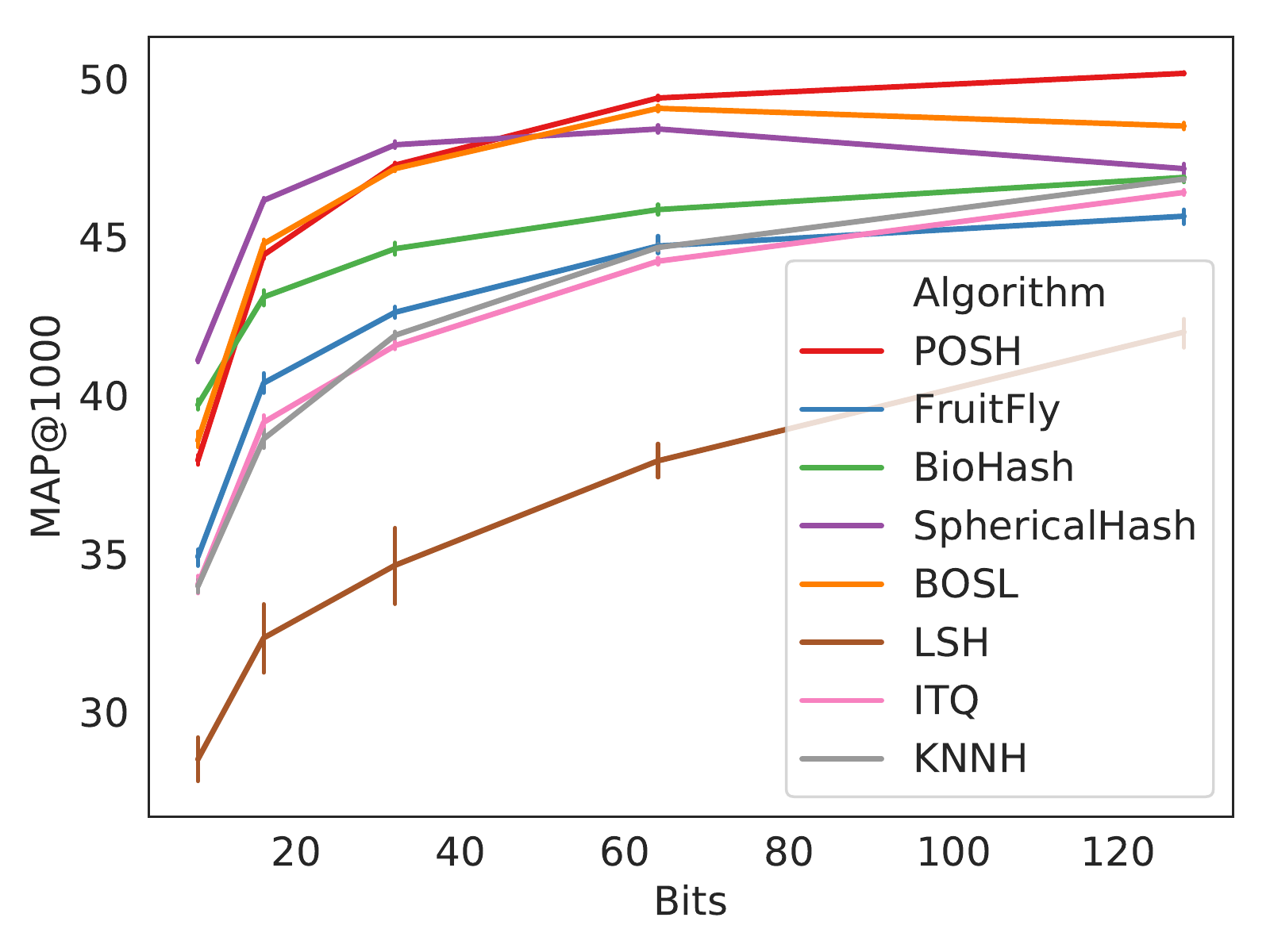} &
		\includegraphics[width=\linewidth]{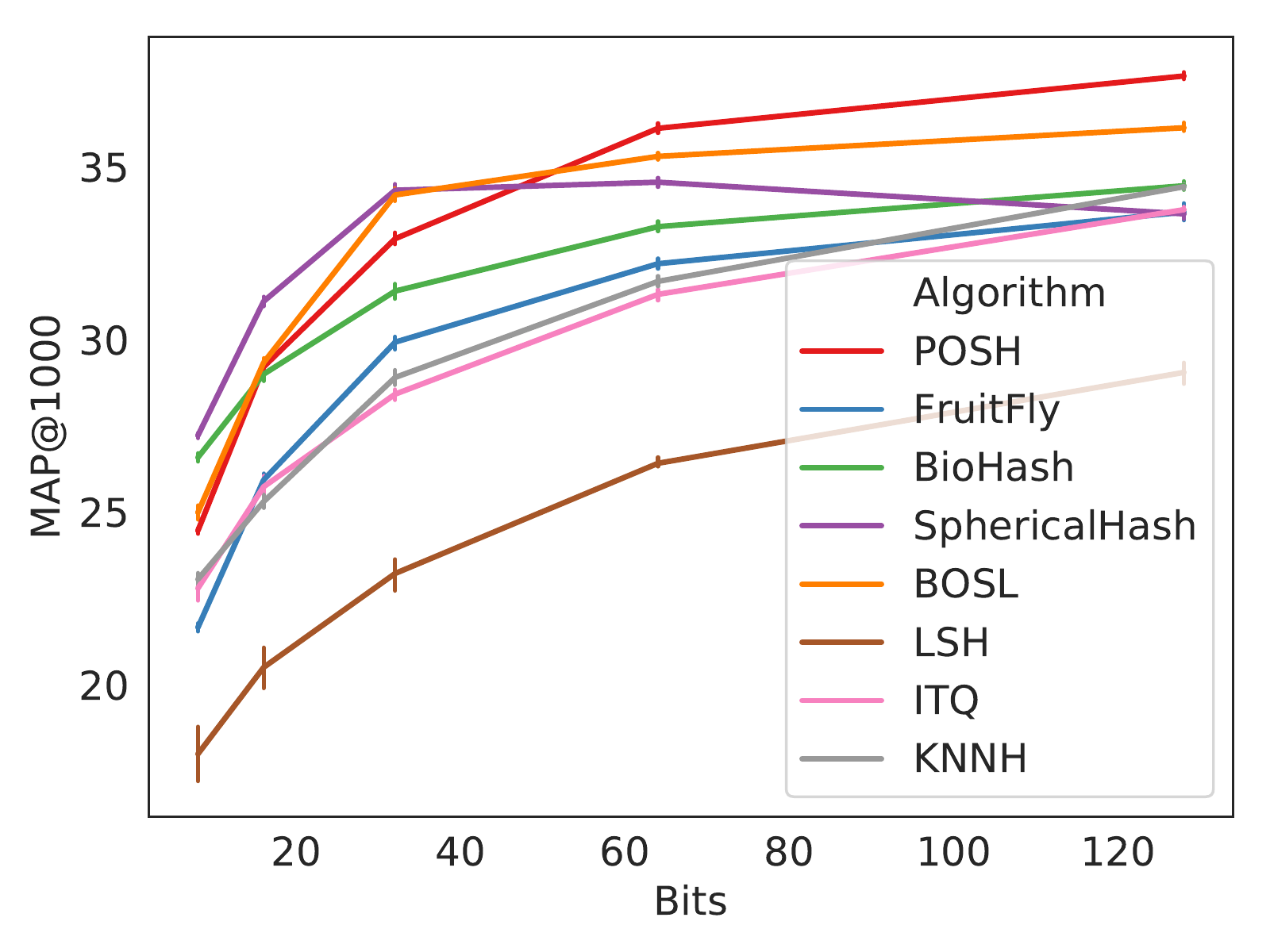} \\
	\end{tabu}
	\end{small}
	
	\caption{Comparison of different hashing methods under different configurations. The abscissa represents the hash length $D$ for dense codes (LSH, ITQ, KNNH) and the number $\alpha$ of set bits for sparse codes (POSH, FruitFly, BioHash, SphericalHash, BOSL). For sparse codes, $D=1024$. Error bars represent 95\% confidence intervals. Additional results are available in \cref{sec:additional_experiments}.}
	\label{fig:per_bits}
\end{figure}

\begin{table}[t]
	\caption{Comparison between shallow and deep hashing methods. Here, we use MAP to directly compare with the reported results of deep methods. Note that for KNNH, the results differ from those reported in~\cite{heKNearestNeighborsHashing2019}, as we run it using the experimental protocol from \cref{sec:experimental_setup}. At 64 bits, POSH results are superior to those of SphericalHash when considering MAP@1000, see \cref{fig:per_bits}.}
	\label{tab:comparison_deep_hashing}

	\centering
	\begin{small}
	\begin{tabu} to \textwidth {l @{\hspace{5pt}} c *{7}{ @{\hspace{5pt}} S}}
		\toprule
		& {\multirow{2}{*}{Bits}} & \multicolumn{3}{c}{Deep} & \multicolumn{4}{c}{Shallow} \\
		\cmidrule(r{1em}){3-5} \cmidrule(r){6-9}
		&& {Deepbit~\cite{linLearningCompactBinary2016}} & {DH~\cite{liongDeepHashingCompact2015}} & {UHBDNN~\cite{doLearningHashBinary2016}} & {KNNH~\cite{heKNearestNeighborsHashing2019}} & {FruitFly~\cite{dasguptaNeuralAlgorithmFundamental2017}} & {SphericalHash} & {POSH} \\
		\midrule
		\multirow{3}{*}{\footnotesize  \shortstack{CIFAR10\\GIST}}
		& 16 & 14.35 & 16.17 & 17.83 & 16.94 & 14.41 & 17.63 & \bfseries 17.85 \\
		& 32 & 16.33 & 16.62 & 18.52 & 17.72 & 16.20 & \bfseries 18.74 & 18.45\\
		& 64 & 17.97 & 16.96 & {-}   & 18.34 & 17.39 & 18.92 & \bfseries 19.15 \\

		\midrule

        \multirow{3}{*}{\footnotesize  \shortstack{MNIST}}
		& 16 & {-} & 43.14 & 45.38 & 42.77 & 29.68 & \bfseries 51.44 & 48.28 \\
		& 32 & {-} & 44.97 & 47.21 & 45.69 & 35.02 & \bfseries 54.03 & 49.60 \\
		& 64 & {-} & 46.74 & {-}   & 49.10 & 39.60 & \bfseries 50.92 & 50.64 \\

		\bottomrule
 	\end{tabu}
	\end{small}
\end{table}

\section{Candidate refinement}
\label{sec:refinement}

Some information is lost when performing similarity search with \cref{eq:fruit_fly_hashing}. In particular, some ties are introduced: two database points that are at different distances from the query in the input space may be at the same distance from the query in Hamming space. % Another example is found in \cref{eq:hamming_fixed_bits}, which shows that all Hamming distances are even, reducing the granularity of the estimated (dis)similarity.
%As working in a (sparse) Hamming space is critical to enable fast querying, we need to overcome this issue without reformulating the whole approach.

\Cref{eq:fruit_fly_hashing} can be regarded as an encoder neural network, $f_{\text{enc}}(\vect{x}) = \operatorname{\textsc{wta}}_{\alpha} (\mat{W} \vect{x})$, which consists of a fully connected layer followed by a non-linearity. It produces sparse high-dimensional output vectors. In the following, we describe how, by adding a decoder $f_{\text{dec}}$, we can obtain further improvements in accuracy. In short, the decoder will be tasked with breaking artifactual ties.

The bulk similarity search will be performed in Hamming space using $f_{\text{enc}}$, but, instead of just retrieving the $k$ elements we need, we retrieve a larger number of them, say $2 \cdot k$. Then, on this enlarged subset of candidates, we refine the search using the distance
$\norm{ \vect{x}_q - f_{\text{dec}}(f_{\text{enc}}(\vect{x}_i))}{2}$ for the final top-$k$ ranking (other distances could be used instead).

To ensure that the additional decoding step encompasses a minimal detrimental effect on the search speed, we use a linear decoder $f_{\text{dec}} (\vect{h}) = \mat{D} \vect{h}$, where $\mat{D} \in \Real^{d \times D}$. For simplicity, we assume that the encoder and decoder are trained separately.\footnote{Joint end-to-end learning is left as future work.}  Given the training set $\{ \vect{x}_i \}_{i=1}^n$, we seek $\mat{D}$ such that
\begin{equation}
	\min_{\mat{D}} \sum_{i=1}^{n} \norm{\vect{x}_i - \mat{D}\vect{h}_i }{2}^2 ,
	\label[problem]{eq:decoder}
\end{equation}
where $\vect{h}_i = \operatorname{\textsc{wta}}_{\alpha} (\mat{W} \vect{x}_i)$.
This is simply a least squares problem that can be solved in a single pass over the data. In \cref{sec:sbiht}, we explore a more elaborate decoder that is interesting from a theoretical perspective. However, we did not find a meaningful improvement in accuracy over \cref{eq:decoder}.

\section{Experimental results}
\label{sec:experiments}

Because of space limitations, the experimental setup is deferred to \cref{sec:experimental_setup,sec:additional_experiments}.

In~\cref{fig:per_bits} we compare different unsupervised hashing methods. FruitFly, which uses no training, generally outperforms dense hashing methods. In turn, the proposed SphericalHash and BOSL dominate at lower values of $\alpha$, while POSH dominates at higher values.

We also compare the shallow hashing methods (one matrix multiplication, followed by a nonlinearity) with deep hashing methods in \cref{tab:comparison_deep_hashing}. SphericalHash and POSH are competitive against these more computationally demanding alternatives.

The effect of candidate refinement (see \cref{sec:refinement}) is explored in \cref{tab:candidate_refinement}. We observe a consistent increase in accuracy when using this technique. As described in \cref{sec:refinement}, a larger candidate oversampling factor would be expected to offer higher accuracy. However, the decoder does not achieve perfect reconstruction and, thus, the optimal oversampling factor needs to be computed experimentally. Moreover, a smaller oversampling factor leads to a lesser impact on the search speed.

\begin{table}[t]
	\caption{Effect in MAP@100 of candidate refinement (CR) under different oversampling factors (2, 5, and 10). CR improves the performance with little overhead, particularly when using a factor of 2. See \cref{sec:refinement} for more details about this technique.}
	\label{tab:candidate_refinement}
	
	\centering
	\begin{small}
	\begin{tabular} {l *{4}{S[table-format=2.2(2),separate-uncertainty]}}
		\toprule
		& {POSH} & {POSH+CR2} & {POSH+CR5} & {POSH+CR10} \\
		\midrule
		MNIST      & 91.36 \pm 0.21 & \bfseries 92.78 \pm 0.12 & 92.71 \pm 0.06 & 92.61 \pm 0.07 \\
		MNIST-GIST & 95.23 \pm 0.06 & \bfseries 95.56 \pm 0.04 & 95.18 \pm 0.05 & 94.96 \pm 0.05 \\
		CIFAR10-GIST & 41.34 \pm 0.11 & \bfseries 43.30 \pm 0.15 & 43.05 \pm 0.10 & 42.89 \pm 0.11 \\
		CIFAR10-VGG & 56.81 \pm 0.12 & \bfseries 57.02 \pm 0.07 & 56.54 \pm 0.11 & 56.39 \pm 0.11 \\
		LabelMe-12-50K-GIST & 33.96 \pm 0.34 & \bfseries 36.31 \pm 0.33 & 35.87 \pm 0.37 & 35.58 \pm 0.27 \\
		LabelMe-12-50K-VGG & 46.42 \pm 0.23 & \bfseries 47.01 \pm 0.23 & 46.49 \pm 0.22 & 46.16 \pm 0.24 \\
%		SIFT10K & 83.72 \pm 0.45 & 91.43 \pm 0.21 & 91.61 \pm 0.21 & \bfseries 91.68 \pm 0.14 \\
		\bottomrule
	\end{tabular}
	\end{small}	
\end{table}

We compare different methods on a large dataset with approximately 2.5M target elements. Here, SphericalHash does not outperform the training-free FruitFly. POSH clearly outperforms FruitFly, obtaining the best results. Additionally, we report the increase in accuracy due to candidate refinement.

\begin{table}[t]
	\caption{Results (MAP@1000) for the large-scale dataset Places205 (approximately 2.5M target and 20K query elements). Additional results are available in \cref{sec:additional_experiments}. We highlight the best hashing method, without considering candidate refinement (CR).}
	\label{tab:places205_map}
	
	\centering
	\begin{small}
	\begin{tabular} {l *{7}{S[table-format=2.2]}}
		\toprule
		& {LSH} & {ITQ} & KNNH & {FruitFly} & {SphericalHash} & {POSH} & {POSH+CR2} \\
		\midrule
		16 bits &  9.70 & 13.70 & 13.64 & 27.47 & 27.14 & \bfseries 29.32 & 31.11 \\
		32 bits & 18.29 & 22.52 & 22.12 & 30.31 & 27.52 & \bfseries 31.49 & 32.55 \\
		64 bits & 25.62 & 28.30 & 28.08 & 32.02 & 27.61 & \bfseries 32.83 & 33.20 \\
		\bottomrule
	\end{tabular}
	\end{small}	
\end{table}

\textbf{Performance.}
Here, we use a coarse quantizer \cite{johnsonBillionscaleSimilaritySearch2019}, which clusters the data into $\lceil n / 1000 \rceil$ groups. During querying, we only explore the clusters whose centroids are closest to the query (20 by default). The coarse quantizer causes MAP@1000 to drop from 32.83 to 30.67\% in Places205. Using a parallelized C code and efficient linear algebra techniques, our implementation computes hamming distances at approximately 14000 queries per second in Places205 (\textasciitilde2.5M target elements).

\section{Conclusions}
\label{sec:conclusions}

In this work, we presented a novel method for similarity search: Procrustean Orthogonal Sparse Hashing (POSH). POSH takes inspiration from the insect olfaction  that has been shown to be structurally and functionally analogous to sparse hashing~\cite{dasguptaNeuralAlgorithmFundamental2017}.
We also analyzed in depth two recently proposed algorithms, Optimal Sparse Lifting (OSL)~\cite{liFastSimilaritySearch2018} and BioHash~\cite{ryaliBioInspiredHashingUnsupervised2020}, that also take inspiration from the fruit-fly hashing method~\cite{dasguptaNeuralAlgorithmFundamental2017}. We characterized OSL and BioHash formally and introduced new methods, BOSL and SphericalHash, that are conceptually equal and yet superior to the original ones.
We showed, through numerous and varied numerical experiments, that POSH and SphericalHash outperform, in terms of accuracy, other state-of-the-art unsupervised hashing methods.

\ifthenelse{\not\boolean{preprint}}{
    \section*{Broader Impact}
    
    The proposed approach will carry any biases induced by the cosine similarity between input samples. As such, measures to promote fairness should be put in place as a preprocessing step (i.e., during feature building/learning) to ensure that the cosine similarity does not present any undesired biases.
}{}

\begin{ack}
The authors would like to acknowledge the valuable discussions and interactions with Sourabh Dongaonkar, Chetan Chauhan, Jawad Khan, and Rick Coulson.
\end{ack}

\bibliographystyle{plain}
\bibliography{simsearch,one-bit_compressive_sensing}

\begin{thebibliography}{10}

\bibitem{andoniNearoptimalHashingAlgorithms2008a}
Alexandr Andoni and Piotr Indyk.
\newblock Near-optimal hashing algorithms for approximate nearest neighbor in
  high dimensions.
\newblock {\em Communications of the ACM}, 51(1):117--122, January 2008.

\bibitem{baraniukOneBitCompressiveSensing2016a}
Rich Baraniuk, Simon Foucart, Deanna Needell, Yaniv Plan, and Mary Wootters.
\newblock One-{{Bit Compressive Sensing}} of {{Dictionary}}-{{Sparse Signals}}.
\newblock {\em arXiv:1606.07531 [cs, math]}, June 2016.

\bibitem{blumensathIterativeHardThresholding2009}
Thomas Blumensath and Mike~E. Davies.
\newblock Iterative hard thresholding for compressed sensing.
\newblock {\em Applied and Computational Harmonic Analysis}, 27(3):265--274,
  November 2009.

\bibitem{charikarSimilarityEstimationTechniques2002}
Moses~S. Charikar.
\newblock Similarity estimation techniques from rounding algorithms.
\newblock In {\em Proceedings of the Thiry-Fourth Annual {{ACM}} Symposium on
  {{Theory}} of Computing - {{STOC}} '02}, page 380, {Montreal, Quebec,
  Canada}, 2002. {ACM Press}.

\bibitem{coatesLearningFeatureRepresentations2012}
Adam Coates and Andrew~Y. Ng.
\newblock Learning {{Feature Representations}} with {{K}}-{{Means}}.
\newblock In Gr{\'e}goire Montavon, Genevi{\`e}ve~B. Orr, and Klaus-Robert
  M{\"u}ller, editors, {\em Neural {{Networks}}: {{Tricks}} of the {{Trade}}:
  {{Second Edition}}}, Lecture {{Notes}} in {{Computer Science}}, pages
  561--580. {Springer}, {Berlin, Heidelberg}, 2012.

\bibitem{dasguptaNeuralAlgorithmFundamental2017}
Sanjoy Dasgupta, Charles~F. Stevens, and Saket Navlakha.
\newblock A neural algorithm for a fundamental computing problem.
\newblock {\em Science}, 358(6364):793--796, November 2017.

\bibitem{datarLocalitysensitiveHashingScheme2004}
Mayur Datar, Nicole Immorlica, Piotr Indyk, and Vahab~S. Mirrokni.
\newblock Locality-sensitive hashing scheme based on p-stable distributions.
\newblock In {\em Proceedings of the Twentieth Annual Symposium on
  {{Computational}} Geometry - {{SCG}} '04}, page 253, {Brooklyn, New York,
  USA}, 2004. {ACM Press}.

\bibitem{doLearningHashBinary2016}
Thanh-Toan Do, Anh-Dzung Doan, and Ngai-Man Cheung.
\newblock Learning to {{Hash}} with {{Binary Deep Neural Network}}.
\newblock In Bastian Leibe, Jiri Matas, Nicu Sebe, and Max Welling, editors,
  {\em Computer {{Vision}} \textendash{} {{ECCV}} 2016}, Lecture {{Notes}} in
  {{Computer Science}}, pages 219--234, {Cham}, 2016. {Springer International
  Publishing}.

\bibitem{frankAlgorithmQuadraticProgramming1956a}
Marguerite Frank and Philip Wolfe.
\newblock An algorithm for quadratic programming.
\newblock {\em Naval Research Logistics Quarterly}, 3(1-2):95--110, March 1956.

\bibitem{gionisSimilaritySearchHigh1999}
Aristides Gionis, Piotr Indyk, and Rajeev Motwani.
\newblock Similarity {{Search}} in {{High Dimensions}} via {{Hashing}}.
\newblock In {\em Proceedings of the 25th {{International Conference}} on
  {{Very Large Data Bases}}}, {{VLDB}} '99, pages 518--529, {San Francisco, CA,
  USA}, 1999. {Morgan Kaufmann Publishers Inc.}

\bibitem{gongIterativeQuantizationProcrustean2013}
Yunchao Gong, Svetlana Lazebnik, Albert Gordo, and Florent Perronnin.
\newblock Iterative {{Quantization}}: {{A Procrustean Approach}} to {{Learning
  Binary Codes}} for {{Large}}-{{Scale Image Retrieval}}.
\newblock {\em IEEE Transactions on Pattern Analysis and Machine Intelligence},
  35(12):2916--2929, December 2013.

\bibitem{heKNearestNeighborsHashing2019}
Xiangyu He, Peisong Wang, and Jian Cheng.
\newblock K-{{Nearest Neighbors Hashing}}.
\newblock In {\em {{CVPR}}}, 2019.

\bibitem{indykApproximateNearestNeighbors1998}
Piotr Indyk and Rajeev Motwani.
\newblock Approximate nearest neighbors: Towards removing the curse of
  dimensionality.
\newblock In {\em Proceedings of the Thirtieth Annual {{ACM}} Symposium on
  {{Theory}} of Computing - {{STOC}} '98}, pages 604--613, {Dallas, Texas,
  United States}, 1998. {ACM Press}.

\bibitem{jaggiRevisitingFrankWolfeProjectionfree2013a}
Martin Jaggi.
\newblock Revisiting {{Frank}}-{{Wolfe}}: Projection-free sparse convex
  optimization.
\newblock In {\em {{ICML}}}, {{ICML}}'13, pages I--427--I--435, {Atlanta, GA,
  USA}, June 2013. {JMLR.org}.

\bibitem{johnsonBillionscaleSimilaritySearch2019}
Jeff Johnson, Matthijs Douze, and Herve Jegou.
\newblock Billion-scale similarity search with {{GPUs}}.
\newblock {\em IEEE Transactions on Big Data}, pages 1--1, 2019.

\bibitem{jortnerSimpleConnectivityScheme2007}
R.~A. Jortner, S.~S. Farivar, and G.~Laurent.
\newblock A {{Simple Connectivity Scheme}} for {{Sparse Coding}} in an
  {{Olfactory System}}.
\newblock {\em Journal of Neuroscience}, 27(7):1659--1669, February 2007.

\bibitem{junwangSemiSupervisedHashingLargeScale2012}
{Jun Wang}, S.~Kumar, and {Shih-Fu Chang}.
\newblock Semi-{{Supervised Hashing}} for {{Large}}-{{Scale Search}}.
\newblock {\em IEEE Transactions on Pattern Analysis and Machine Intelligence},
  34(12):2393--2406, December 2012.

\bibitem{knudsonOneBitCompressiveSensing2016}
Karin Knudson, Rayan Saab, and Rachel Ward.
\newblock One-{{Bit Compressive Sensing With Norm Estimation}}.
\newblock {\em IEEE Transactions on Information Theory}, 62(5):2748--2758, May
  2016.

\bibitem{koepBinaryIterativeHard2017}
Niklas Koep and Rudolf Mathar.
\newblock Binary {{Iterative Hard Thresholding}} for {{Frequency}}-{{Sparse
  Signal Recovery}}.
\newblock In {\em {{WSA}} 2017; 21th {{International ITG Workshop}} on {{Smart
  Antennas}}}, pages 1--7, March 2017.

\bibitem{krizhevskyLearningMultipleLayers2009}
Alex Krizhevsky.
\newblock Learning multiple layers of features from tiny images.
\newblock Technical report, {University of Toronto}, 2009.

\bibitem{lecunGradientbasedLearningApplied1998a}
Y.~Lecun, L.~Bottou, Y.~Bengio, and P.~Haffner.
\newblock Gradient-based learning applied to document recognition.
\newblock {\em Proceedings of the IEEE}, 86(11):2278--2324, Nov./1998.

\bibitem{liFastSimilaritySearch2018}
Wenye Li, Mao Jingwei, Zhang Yin, and Cui Shuguang.
\newblock Fast {{Similarity Search}} via {{Optimal Sparse Lifting}}.
\newblock In {\em {{NIPS}}}, 2018.

\bibitem{linLearningCompactBinary2016}
Kevin Lin, Jiwen Lu, Chu-Song Chen, and Jie Zhou.
\newblock Learning {{Compact Binary Descriptors}} with {{Unsupervised Deep
  Neural Networks}}.
\newblock In {\em 2016 {{IEEE Conference}} on {{Computer Vision}} and {{Pattern
  Recognition}} ({{CVPR}})}, pages 1183--1192, {Las Vegas, NV, USA}, June 2016.
  {IEEE}.

\bibitem{liongDeepHashingCompact2015}
Venice~Erin Liong, {Jiwen Lu}, {Gang Wang}, Pierre Moulin, and {Jie Zhou}.
\newblock Deep hashing for compact binary codes learning.
\newblock In {\em 2015 {{IEEE Conference}} on {{Computer Vision}} and {{Pattern
  Recognition}} ({{CVPR}})}, pages 2475--2483, {Boston, MA, USA}, June 2015.
  {IEEE}.

\bibitem{mairalOnlineLearningMatrix2010}
Julien Mairal, Francis Bach, Jean Ponce, and Guillermo Sapiro.
\newblock Online learning for matrix factorization and sparse coding.
\newblock {\em Journal of Machine Learning Research}, 11:19--60, 2010.

\bibitem{masciSparseSimilaritypreservingHashing2014}
Jonathan Masci, Alex~M. Bronstein, Michael~M. Bronstein, Pablo Sprechmann, and
  Guillermo Sapiro.
\newblock Sparse similarity-preserving hashing.
\newblock {\em arXiv:1312.5479 [cs]}, February 2014.

\bibitem{olshausenSparseCodingOvercomplete1997a}
Bruno~A. Olshausen and David~J. Field.
\newblock Sparse coding with an overcomplete basis set: {{A}} strategy employed
  by {{V1}}?
\newblock {\em Vision Research}, 37(23):3311--3325, December 1997.

\bibitem{pehlevanWhySimilarityMatching2018}
C.~Pehlevan, A.~Sengupta, and D.~Chklovskii.
\newblock Why do similarity matching objectives lead to
  {{Hebbian}}/anti-{{Hebbian}} networks?
\newblock {\em Neural Computation}, 30(1):84--124, 2018.

\bibitem{pehlevanClusteringNeuralNetwork2017b}
Cengiz Pehlevan, Alexander Genkin, and Dmitri~B. Chklovskii.
\newblock A clustering neural network model of insect olfaction.
\newblock In {\em 2017 51st {{Asilomar Conference}} on {{Signals}},
  {{Systems}}, and {{Computers}}}, pages 593--600, {Pacific Grove, CA, USA},
  October 2017. {IEEE}.

\bibitem{ryaliBioInspiredHashingUnsupervised2020}
Chaitanya~K. Ryali, John~J. Hopfield, Leopold Grinberg, and Dmitry Krotov.
\newblock Bio-{{Inspired Hashing}} for {{Unsupervised Similarity Search}}.
\newblock {\em arXiv:2001.04907}, January 2020.

\bibitem{senguptaManifoldtilingLocalizedReceptive2018}
Anirvan~M. Sengupta, Mariano Tepper, Cengiz Pehlevan, Alexander Genkin, and
  Dmitri~B. Chklovskii.
\newblock Manifold-tiling {{Localized Receptive Fields}} are {{Optimal}} in
  {{Similarity}}-preserving {{Neural Networks}}.
\newblock In {\em {{NIPS}}}, 2018.

\bibitem{uetzLargescaleObjectRecognition2009}
Rafael Uetz and Sven Behnke.
\newblock Large-scale object recognition with {{CUDA}}-accelerated hierarchical
  neural networks.
\newblock In {\em 2009 {{IEEE International Conference}} on {{Intelligent
  Computing}} and {{Intelligent Systems}}}, volume~1, pages 536--541, November
  2009.

\bibitem{wixtedCodingEpisodicMemory2018}
John~T. Wixted, Stephen~D. Goldinger, Larry~R. Squire, Joel~R. Kuhn, Megan~H.
  Papesh, Kris~A. Smith, David~M. Treiman, and Peter~N. Steinmetz.
\newblock Coding of episodic memory in the human hippocampus.
\newblock {\em Proceedings of the National Academy of Sciences},
  115(5):1093--1098, January 2018.

\bibitem{zhangFastOrthogonalProjection2015}
Xu~Zhang, Felix~X. Yu, Ruiqi Guo, Sanjiv Kumar, Shengjin Wang, and Shi-Fu
  Chang.
\newblock Fast {{Orthogonal Projection Based}} on {{Kronecker Product}}.
\newblock In {\em 2015 {{IEEE International Conference}} on {{Computer Vision}}
  ({{ICCV}})}, pages 2929--2937, {Santiago, Chile}, December 2015. {IEEE}.

\bibitem{zhouLearningDeepFeatures2014}
Bolei Zhou, Agata Lapedriza, Jianxiong Xiao, Antonio Torralba, and Aude Oliva.
\newblock Learning {{Deep Features}} for {{Scene Recognition Using Places
  Database}}.
\newblock In {\em Proceedings of the 27th {{International Conference}} on
  {{Neural Information Processing Systems}} - {{Volume}} 1}, {{NIPS}}'14, pages
  487--495, {Cambridge, MA, USA}, 2014. {MIT Press}.

\end{thebibliography}

\appendix

\ifthenelse{\not\boolean{preprint}}{
    \newpage
    \setcounter{page}{1}
    
    \begin{center}
    \begin{huge}
    Procrustean Orthogonal Sparse Hashing\\[10pt]
    Supplementary material
    \end{huge}
    \end{center}
    \vspace{0.2in}
}{}

\section{Additional material for POSH}
\label{sec:posh_additional}

\fruitflyoptimization*

\begin{proof}
	Let $\vect{y} = \mat{W} \vect{x}$. Expanding the loss $\norm{ \vect{h} - \vect{y} }{2}^2$, removing the constant term, and using
%	\begin{equation}
%		\transpose{\vect{h}} \vect{h} = \sum_i (\vect{h})_i^2 = \sum_i (\vect{h})_i = \transpose{\vect{1}} \vect{h} = \alpha ,
%	\end{equation}
	$\transpose{\vect{h}} \vect{h} = \sum_i (\vect{h})_i^2 = \sum_i (\vect{h})_i = \transpose{\vect{1}} \vect{h} = \alpha$,
	the problem becomes
	\begin{equation}
	\max_{\vect{h}} \transpose{\vect{h}} \vect{y}
	\quad\text{s.t.}\quad
	\begin{gathered}
	\vect{h} \in \{ 0, 1 \}^{D} , \ 
	\transpose{\vect{1}} \vect{h}  = \alpha .
	\end{gathered}
	\end{equation}
	To maximize the quantity $\transpose{\vect{h}} \vect{y} = \sum_i (\vect{h})_i (\vect{y})_i$ under the specified constraints, we must place the allocated number $\alpha$ of ones in $\vect{h}$ where it counts the most, i.e., the $\alpha$ largest entries of $\vect{y}$.
\end{proof}

\fruitflyorthogonality*

\begin{proof}
	Given two vectors $\vect{v}_1$ and $\vect{v}_2$ such that $(\vect{v}_1)_i, (\vect{v}_2)_i \sim \operatorname{Bernoulli}(p)$, we have that $(\vect{v}_1)_i^2 \sim \operatorname{Bernoulli}(p)$ and $(\vect{v}_1)_i \cdot (\vect{v}_2)_i \sim \operatorname{Bernoulli}(p^2)$. Their sum follows a Binomial distribution.
\end{proof}

\cref{algo:posh_incremental} formally describes the learning algorithm for POSH.

\begin{algorithm2e}[ht]
	\caption{Procrustean Orthogonal Sparse Hashing}
	\label{algo:posh_incremental}
	
	\begin{small}
	\SetKwInOut{Input}{input}
	\SetKwInOut{Output}{output}
	
	\Input{Data $\{ \vect{x}_i \in \Real^d \}_{i=1}^{n}$, output dimension $D$, sparsity level $\alpha$, number $N$ of training epochs, mini-batch size $N_{\text{mini-batch}}$.}
	\Output{Weight matrix $\mat{W} \in \Real^{D \times d}$.}
	
	Create a matrix $\mat{W}_0 \in \Real^{D \times d}$ by sampling its entries from the standard normal distribution\;
	$\mat{W} \gets \mat{U} \transpose{\mat{V}}$, where $\mat{U} \mat{S} \transpose{\mat{V}}$ is the SVD of $\mat{W}_0$\;
	
	$\mat{M} \gets \mat{W}$\;

	\ForEach{epochs $t = 1 \dots N$}{
		\For{$i = 1 \dots n$}{
			$\vect{s}_i \gets \operatorname{\textsc{wta}}_{\alpha} \left( \mat{W} \vect{x}_i \right)$\;
			$\mat{M} \gets \mat{M} + \vect{s}_i \transpose{\vect{x}_i}$\;
			
			\If{$i \bmod{N_{\text{mini-batch}}} = 0$}{
				$\mat{W} \gets \mat{U} \transpose{\mat{V}}$, where $\mat{U} \mat{S} \transpose{\mat{V}}$ is the SVD of $\mat{M}$\;
			}
		}
	}
	\end{small}
\end{algorithm2e}

\section{Additional material for SphericalHash}
\label{sec:spherical_hash}

\biohashspherical*

\begin{proof}
	From the constraints in \cref{eq:spherical_kmeans}, we have $\transpose{\vect{s}}_i \mat{W} \transpose{\mat{W}} \vect{s}_i = 1$.
	The problem for $\mat{W}$ becomes
	\begin{equation}
		\max_{\mat{W}}
		\sum_{i=1}^{n}
		\transpose{\vect{s}_i} \mat{W} \vect{x}_i
		\quad\text{s.t.}\quad
		\norm{\vect{w}_j}{2} = 1 .
		\label[problem]{eq:spherical_kmeans_W}
	\end{equation}
	The problem for each $\vect{s}_i$ becomes
	\begin{equation}
		\max_{\vect{s}_i}
		\transpose{\vect{s}_i} \mat{W} \vect{x}_i
		\quad\text{s.t.}\quad
		\begin{gathered}
			\norm{\vect{s}_i}{0} \in \{ 0, 1 \}^D ,\
			\norm{\vect{s}_i}{0} \leq 1,\\
		\end{gathered}
	\end{equation}
	and its solution given by
	\begin{equation}
		\vect{s}_i^*
		=
		\operatorname{\textsc{wta}}_{1} \left( \mat{W} \vect{x}_i \right) ,
%		=
%		\begin{cases}
%			1 & \text{if } j = \argmax_l \transpose{\vect{w}}_l \vect{x}_i ;\\
%			0 & \text{otherwise.}
%		\end{cases}
		\label{eq:spherical_kmeans_S_solution}
	\end{equation}
	where function $\operatorname{\textsc{wta}}$ is defined in \cref{eq:wta_nl}.
	Plugging this solution in \cref{eq:spherical_kmeans_W}, we get
	\begin{equation}
		\max_{\mat{W}}
		\sum_{i=1}^{n} \sum_{j=1}^{D} 
		\indicator{j = \argmax_{l} \transpose{\vect{w}}_l \vect{x}_i}
		\transpose{\vect{w}}_j \vect{x}_i
		\quad\text{s.t.}\quad
		\norm{\vect{w}_j}{2} = 1 ,
	\end{equation}
	which is equivalent to \cref{eq:biohash_loss}
\end{proof}

We now present SphericalHash, a new hashing method that combines of spherical k-means for learning and \cref{eq:wta_nl} for hashing. 
The solution to \cref{eq:spherical_kmeans_W} is given by~\cite{coatesLearningFeatureRepresentations2012}
\begin{equation}
	\mat{W} = 
	\operatorname{normalize}
	\left(
	\sum_{i=1}^{n}
	\hat{\vect{s}}_i \transpose{\vect{x}_i}
	\right) ,
\end{equation}
where $\hat{\vect{s}}_i$ is given by \cref{eq:spherical_kmeans_S_solution} and the function $\operatorname{normalize}$ operates independently on each row, computing $\vect{w}_j = \vect{w}_j / \norm{\vect{w}_j}{2}$.
\Cref{algo:spherical_kmeans} formally describes the learning algorithm. 

Compared to the original BioHash learning update in \cref{eq:biohash_dynamics} in page \pageref{eq:biohash_dynamics}~\cite[Eq.~(1)]{ryaliBioInspiredHashingUnsupervised2020}, \cref{algo:spherical_kmeans} is not biologically plausible, mainly because of the non-local weight normalization step. However, biologically plausible alternatives (i.e., with Hebbian and local updates) are available in the literature~\cite{pehlevanClusteringNeuralNetwork2017b}.

\begin{algorithm2e}[h]
	\caption{Spherical k-means}
	\label{algo:spherical_kmeans}
	
	\begin{small}
	\SetKwInOut{Input}{input}
	\SetKwInOut{Output}{output}
	
	\Input{Data $\{ \vect{x}_i \in \Real^d \}_{i=1}^{n}$, output dimension $D$, number $N$ of training epochs.}
	\Output{Weight matrix $\mat{W} \in \Real^{D \times d}$.}
	
	Create a matrix $\mat{W}_0 \in \Real^{D \times d}$ by sampling its entries from the standard normal distribution\;
	$\mat{W} \gets \operatorname{normalize} (\mat{W}_0)$\tcp*[l]{Normalize each column to unit norm}
	
	\ForEach{epochs $t = 1 \dots N$}{
	
		$\mat{M} \gets \mat{0}$\;
		
		\For{$i = 1 \dots n$}{
			$\vect{x}_i' \gets \vect{x}_i / \norm{\vect{x}_i}{2}$\;
			$\vect{s}_i \gets \operatorname{\textsc{wta}}_{1} \left( \mat{W} \vect{x}_i' \right)$\;
			$\mat{M} \gets \mat{M} + \vect{s}_i \transpose{\vect{x}_i'}$\;
%			
%			\If{$i \bmod{N_{\text{mini-batch}}} = 0$}{
%	
%				$\vect{\eta} \gets \vect{0}$\tcp*[l]{$D$-dimensional learning rate vector}
%				\lForEach{$j$ such that $(\vect{c})_j > 0$}{
%					$(\vect{\eta})_j \gets 1 / (\vect{c})_j$%
%				}
%				$\mat{W} \gets \vect{1} \transpose{(\vect{1} -\vect{\eta})} \circ \mat{W} + \vect{1} \transpose{\vect{\eta}} \circ \mat{M}$\;
%				
%			}
		}
		$\mat{W} \gets \operatorname{normalize} (\mat{M})$\tcp*[l]{Normalize each column to unit norm}
	}
	\end{small}
\end{algorithm2e}

\begin{algorithm2e}[h]
	\caption{Binary Optimal Sparse Lifting (BOSL)}
	\label{algo:OSL_variant}
	
	\SetKwInOut{Input}{input}
	\SetKwInOut{Output}{output}
	
	\Input{Data $\{ \vect{x}_i \in \Real^d \}_{i=1}^{n}$, output dimension $D$, sparsity level $\alpha$.}
	\Output{Weight matrix $\hat{\mat{W}} \in \Real^{D \times d}$.}
	
	Create a matrix $\mat{W}_1$ by sampling its entries from the standard normal distribution\;
	$\mat{Y}_1 \gets \mat{W}_1 \mat{X}$\;
	$\mat{L} \gets \mat{0}$\;
	$\lambda \gets 1$\;
	\For{$k = 1 \dots K - 1$}{
		$\mat{H}_{k+1} \gets \operatorname{\textsc{wta}}_{\alpha} \left( \lambda \mat{Y}_{k} + \mat{L}_{k} \right)$\;
		$\displaystyle \mat{Y}_{k+1} \gets \argmin_{\mat{Y}} \tfrac{1}{2} \norm{\transpose{\mat{X}} \mat{X} - \transpose{\mat{Y}} \mat{Y}}{F}^2
		+ \tfrac{\lambda}{2} \norm{\mat{H}_{k+1} - \mat{Y} + \lambda^{-1} \mat{L}_{k}}{F}^2
		\enskip\text{s.t.}\enskip
		\mat{0} \leq \mat{Y} \leq \mat{1}
		$\;
		$\mat{L}_{k+1} \gets \mat{L}_{k} + \lambda (\mat{H}_{k+1} - \mat{Y}_{k+1})$\;
	}
	$\displaystyle \hat{\mat{W}} \gets \argmin_{\mat{W}} \tfrac{1}{2} \norm{\mat{W} \mat{X} - \mat{Y}_{K}}{F}^2$\;
\end{algorithm2e}

\section{Solving Binary Sparse Lifting}
\label{sec:OSL_optimization}

To solve \cref{eq:osl_binary}, we follow an ADMM approach and introduce an auxiliary variable $\mat{H}$, obtaining the problem
\begin{equation}
	\min_{\mat{Y}, \mat{Z}} \tfrac{1}{4} \norm{\transpose{\mat{X}} \mat{X} - \transpose{\mat{Y}} \mat{Y}}{F}^2
	\quad\text{s.t.}\quad
	\begin{gathered}
		\transpose{\vect{H}} \mat{1} = \alpha \vect{1} ,\
		\mat{H} \in \{ 0, 1 \}^{D \times n} ,\\
		\mat{0} \leq \mat{Y} \leq \mat{1} ,\
		\mat{H} = \mat{Y} .
	\end{gathered}
\end{equation}
This problem is non-convex and does not have a unique solution.
To find a solution, we use the augmented Lagrangian,
\begin{equation}
	\min_{\mat{Y}, \mat{Z}} \tfrac{1}{4} \norm{\transpose{\mat{X}} \mat{X} - \transpose{\mat{Y}} \mat{Y}}{F}^2
	+ \tfrac{\lambda}{2} \norm{\mat{H} - \mat{Y} + \lambda^{-1} \mat{L}}{F}^2
	\quad\text{s.t.}\quad
	\begin{gathered}
		\transpose{\mat{H}} \mat{1} = \alpha \vect{1} ,\\
		\mat{H} \in \{ 0, 1 \}^{D \times n} ,\\
		\mat{0} \leq \mat{Y} \leq \mat{1} ,
	\end{gathered}
\end{equation}
where $\mat{L}$ is the Lagrange multiplier and $\lambda$ is the penality coefficient.
We perform a series of iterations, alternatively fixing $\mat{H}$ and $\mat{Y}$ while solving for the other. The iterations are
\begin{subequations}
\begin{align}
	\mat{H}_{t+1} &= \argmin_{\mat{Y}}
	\tfrac{\lambda}{2} \norm{\mat{H} - \mat{Y}_{t} + \lambda^{-1} \mat{L}_{t}}{F}^2
	\enskip\text{s.t.}\enskip
	\begin{gathered}
		\transpose{\vect{H}} \mat{1} = \alpha \vect{1} ,\
		\mat{H} \in \{ 0, 1 \}^{D \times n} .
	\end{gathered}
	\label[problem]{eq:osl_binary_H}
	\\
	\mat{Y}_{t+1} &= \argmin_{\mat{Y}} \tfrac{1}{4} \norm{\transpose{\mat{X}} \mat{X} - \transpose{\mat{Y}} \mat{Y}}{F}^2
	+ \tfrac{\lambda}{2} \norm{\mat{H}_{t+1} - \mat{Y} + \lambda^{-1} \mat{L}_{t}}{F}^2 ,
	\enskip\text{s.t.}\enskip
	\begin{gathered}
		\mat{0} \leq \mat{Y} \leq \mat{1} ,
	\end{gathered}
	\label[problem]{eq:osl_binary_Y}
	\\
	\mat{L}_{t+1} &= \mat{L}_{t} + \lambda (\mat{H}_{t+1} - \mat{Y}_{t+1}) .
	\label[problem]{eq:osl_binary_L}
\end{align}
\end{subequations}
From the proof of \cref{theorem:fruit_fly_optimization}, \Cref{eq:osl_binary_H} has the closed-form solution
\begin{equation}
	\mat{H}_{t+1} = \operatorname{\textsc{wta}}_{\alpha} \left( \lambda \mat{Y}_{t} + \mat{L}_{t} \right) ,
\end{equation}
where, by a slight abuse of notation, we apply the function $\operatorname{\textsc{wta}}$ column-wise.
\Cref{eq:osl_binary_Y} is a symmetric NMF problem for which there are many good solvers. For simplicity, in our implementation we use the off-the-shelf L-BFGS-B method.
We set $\lambda = 1$ once and for all.

\paragraph{Observation.}
Here, we point out that \cref{eq:osl_binary_Y} is also an instance of the similarity matching framework~\cite{pehlevanWhySimilarityMatching2018,senguptaManifoldtilingLocalizedReceptive2018}. As such, it can be implemented by a biologically plausible neural network in a streaming fashion~\cite{pehlevanWhySimilarityMatching2018,senguptaManifoldtilingLocalizedReceptive2018}. Potentially, \cref{eq:osl_binary_H,eq:osl_binary_Y,eq:osl_binary_L} could be implemented in a biologically plausible way. We leave this line of work for the future.

\section{Candidate refinement: A signal recovery perspective}
\label{sec:sbiht}

An alternative viewpoint to the one proposed above would be to pose decoding as an inverse problem.
For this, we take inspiration in the compressive sensing literature, which provides algorithms and theoretical recovery results for quantized (one bit) compressive sensing (QCS) problems, e.g.,~\cite{baraniukOneBitCompressiveSensing2016a,knudsonOneBitCompressiveSensing2016,koepBinaryIterativeHard2017}. In QCS, we observe a quantized version $\vect{y} = \operatorname{sign} \left( \mat{W} \vect{x} \right)$ of the sparse signal $\vect{x}$. The objective is to recover $\vect{x}$. Our decoding problem is strikingly similar: we observe a quantized vector $\vect{h} = \operatorname{\textsc{wta}}_{\alpha} \left( \mat{W} \vect{x} \right)$ and seek to recover the vector $\vect{x}$.
Formally, we can write this recovery problem as
\begin{equation}
	\min_{\vect{x}} \norm{\vect{h} - \mat{W} \vect{x}}{2}^2
	\quad\text{s.t.}\quad
	\vect{h} = \operatorname{\textsc{wta}}_{\alpha} \left( \mat{W} \vect{x} \right) .
	\label{eq:decoder_consistent}
\end{equation}

In compressive sensing, Iterative Hard Thresholding (IHT)~\cite{blumensathIterativeHardThresholding2009} is a sparse signal recovery algorithm that consists of the iteration of two steps. First, a gradient descent to reduce the least squares objective $\norm{\vect{h} - \mat{W} \vect{x}}{2}^2$, i.e.,
\begin{equation}
	\hat{\vect{z}} \gets \hat{\vect{z}} - \tau \transpose{\mat{W}} \left( \mat{W} \hat{\vect{z}} - \vect{h}_i \right) .
	\label{eq:iht_gradient_descent}
\end{equation}
Then, $\hat{\vect{z}}$ is projected onto the “$\ell_0$ ball,” i.e., the selection of the largest in magnitude elements, to obtain a sparse estimate $\hat{\vect{x}}$. Quantized compressive sensing can be solved with Binary Iterative Hard Thresholding (BIHT)~\cite{koepBinaryIterativeHard2017}, a modification of the gradient step in IHT defined as follows
\begin{equation}
	\hat{\vect{x}} \gets \hat{\vect{x}} - \tau \transpose{\mat{W}} \left( \operatorname{sign} \left( \mat{W} \hat{\vect{x}} \right) - \vect{h}_i \right) .
	\label{eq:biht_gradient_descent}
\end{equation}

In our problem, we do not have sparsity constraints on the signal $\vect{x}$. Thus, in our algorithm we drop the $\ell_0$ projection step. Finally, we switch the sign quantization in QCS by our specific form of quantization. The resulting algorithm, that we term Sparse Binary Iterative Hard Thresholding (SBIHT), is formally specified in \cref{algo:sbiht}. 
Demonstrating theoretical convergence guarantees for SBIHT is a subject of future work. Of course, SBIHT is slower than linear decoder, which corresponds to zero iterations of SBIHT.

\begin{algorithm2e}[t]
	\caption{Sparse Binary Iterative Hard Thresholding (SBIHT)}
	\label{algo:sbiht}
	
	\SetKwInOut{Input}{input}
	\SetKwInOut{Output}{output}
	
	\Input{Sparsity level $\alpha$, hash code $\vect{h} \in \{ 0, 1 \}^D$ such that $\transpose{\vect{1}} \vect{h} = \alpha$, weight matrix $\mat{W} \in \Real^{D \times d}$, initial decoding weight matrix $\mat{D}$ (see \cref{eq:decoder}).}
	\Output{Reconstructed vector $\hat{\vect{x}} \in \Real^d$.}
	
	$\hat{\vect{x}} \gets \mat{D} \vect{h}$\;
	
	\Repeat{until convergence of $\norm{\vect{b} - \vect{h}}{2}^2$}{
		$\vect{b} \gets \operatorname{\textsc{wta}}_{\alpha} \left( \mat{W} \hat{\vect{x}} \right)$\;
		$\hat{\vect{x}} \gets \hat{\vect{x}} - \tfrac{1}{\sqrt{d D}} \transpose{\mat{W}} \left( \vect{b} - \vect{h} \right)$\;
	}
\end{algorithm2e}

For some examples, SBIHT achieves perfect recovery, i.e., $\vect{b} = \vect{h}$ , resulting in a ``lossless'' decoder, see \cref{fig:SBIHT}. Interestingly, SBIHT, despite its theoretical appeal, does not offer a significant accuracy benefit over the linear decoder from \cref{eq:decoder}. Thus, we perform the bulk of our experiments with the fast linear decoder.

\begin{figure}
	\centering
	\includegraphics[width=0.5\textwidth]{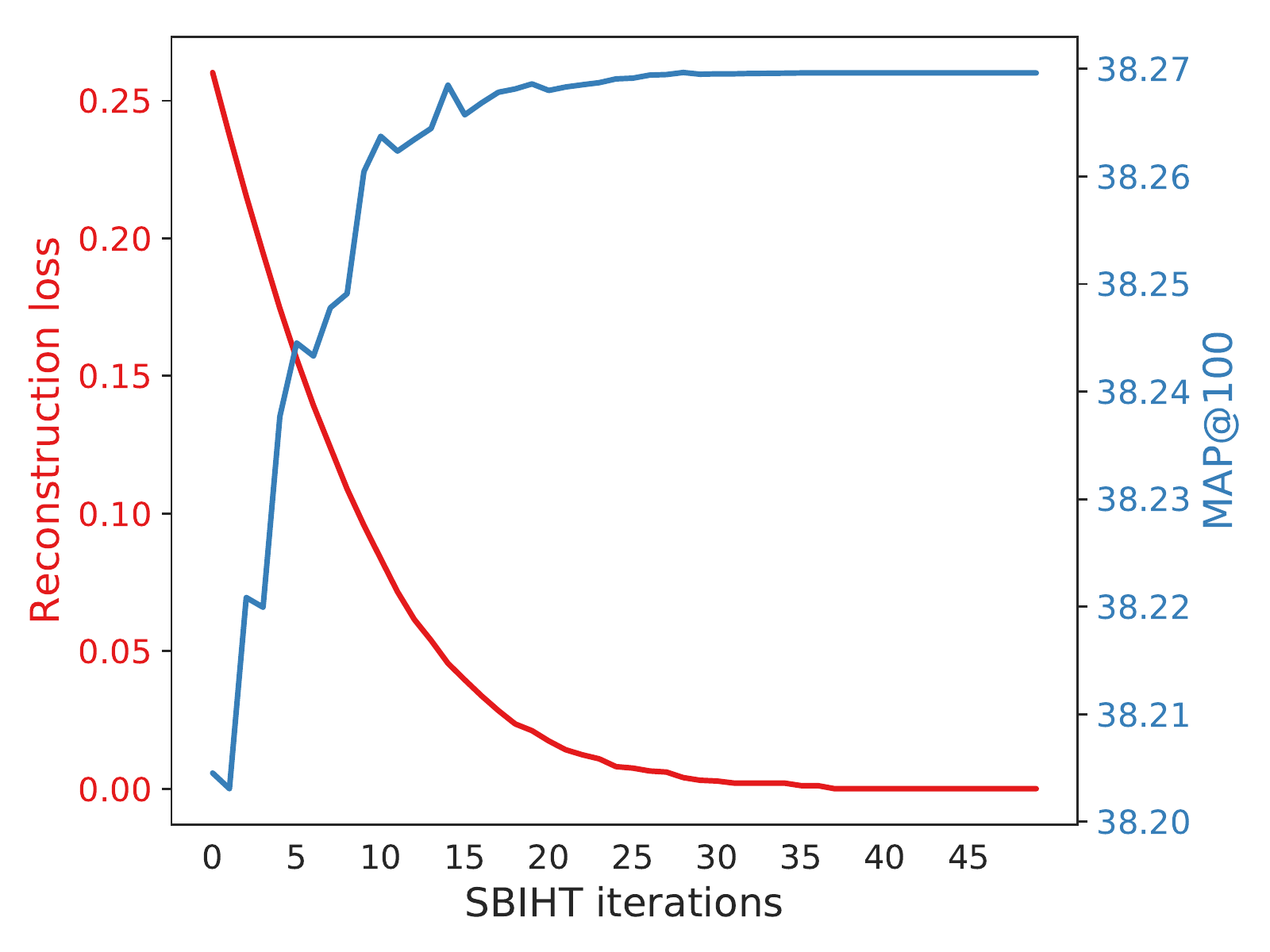}
	
	\caption{SBIHT (\cref{algo:sbiht}) monotonically decreases the objective function of \cref{eq:decoder_consistent}. Here, we plot the accumulated loss across all retrieved samples. In the last 13 iterations, the error is exactly zero for all the retrieved samples. Although not using similarity search accuracy at all in SBIHT lower objective function values correlate with improved similarity search accuracy (blue curve). However, the difference is minimal. For this plot, we use a subset of CIFAR10-GIST where, for each class, we randomly sample 1000 and 100 target and query vectors, respectively.}
	\label{fig:SBIHT}
\end{figure}

\section{Experimental setup}
\label{sec:experimental_setup}

We use datasets that are considered standard benchmarks in the similarity search literature. \Cref{tab:datsets} describes their characteristics.

LabelMe-12-50K~\cite{uetzLargescaleObjectRecognition2009} is a particularly challenging dataset because the data distribution among classes is imbalanced: the five largest classes represent 91\% of the data and and the smallest class, just 0.5\%. Additionally, 50\% of the images show a centered object while the other 50\% show a randomly selected region of a randomly selected image (i.e., visual clutter). These characteristics resemble real-world conditions.

\begin{table}
	\caption{Datasets used in the report. We include datasets of small and large scale, and with varied numbers of dimensions.}
	\label{tab:datsets}
	
	\centering
	\begin{threeparttable}[b]
		\sisetup{
			group-minimum-digits=4,
		}
		\begin{tabular}{l c S[table-format=4.0] S[table-format=7.0] S[table-format=5.0]}
			\toprule
			\multicolumn{1}{c}{Dataset} & Feature & {Dimensions} & {Target set size} & {Query set size} \\
			\midrule
			MNIST~\cite{lecunGradientbasedLearningApplied1998a}\tnote{1} & - & 784 & 69000 & 1000 \\
			MNIST~\cite{lecunGradientbasedLearningApplied1998a}\tnote{1} & GIST & 512 & 66505 & 6996 \\
			CIFAR10~\cite{krizhevskyLearningMultipleLayers2009}\tnote{1} & GIST & 512 & 54000 & 6000 \\
			CIFAR10~\cite{krizhevskyLearningMultipleLayers2009}\tnote{1} & VGG & 4096 & 54000 & 6000 \\
			LabelMe-12-50K~\cite{uetzLargescaleObjectRecognition2009}\tnote{1} & GIST & 512 & 45007 & 4993 \\
			LabelMe-12-50K~\cite{uetzLargescaleObjectRecognition2009}\tnote{1} & VGG & 4096 & 45007 & 4993 \\
			Places205~\cite{zhouLearningDeepFeatures2014}\tnote{1} & AlexNet & 128 & 2428372 & 20500 \\
%			GloVe~\cite{penningtonGloveGlobalVectors2014a}\tnote{2} & GloVe & 25 & 1183514 & 10000 \\
%			SIFT10K~\cite{jegouSearchingQuantizationApproximate2009}\tnote{3} & SIFT & 128 & 10000 & 100 \\
%			SIFT1M~\cite{jegouSearchingQuantizationApproximate2009}\tnote{3} & SIFT & 128 & 1000000 & 10000 \\
			\bottomrule
		\end{tabular}
		\begin{tablenotes}
			\item [1] \url{https://github.com/HolmesShuan/K-Nearest-Neighbors-Hashing}
%			\item [2] \url{https://github.com/erikbern/ann-benchmarks}
%			\item [3] \url{http://corpus-texmex.irisa.fr/}
		\end{tablenotes}
	\end{threeparttable}
\end{table}

\subsection{Algorithms}

Here, we briefly describe the hashing methods from the machine learning literature that are used in our comparisons. There is a breathtaking amount of research around this topic.\footnote{See the repository of papers at \url{https://learning2hash.github.io/papers.html}}
We restrict our comparisons to hashing methods that (1) are unsupervised and (2) have $O(n)$ training time. The only exception to these rules is Optimal Sparse Lifting, which has $O(n^2)$ training time. We include it in our comparisons as it is the first method derived from FruitFly tat uses machine learning.

\Cref{tab:hashing_methods} summarizes all the methods used in our work. Next, we describe the methods that have not been formally introduced in other parts of this work.
Let $\{ \vect{x}_i \in \Real^d \}_{i=1}^{n}$ be the target dataset. In all cases, $D$ is the hash code length. For FruitFly/POSH, $\alpha$ is the number of set bits.

\paragraph{LSH.}
This is the seminal work that introduced the idea of using hashing functions to perform similarity search~\cite{gionisSimilaritySearchHigh1999}. In its simplest form, the hash codes are computed with
\begin{equation}
\vect{h}_i = \operatorname{sign} ( \mat{W} \vect{x} )
,
\label{eq:lsh}
\end{equation}
where the entries of $\mat{W}$ are sampled from a normal distribution and the $\operatorname{sign}$ is applied entry-wise.

\paragraph{ITQ.}
Instead of using randomized LSH codes, this approach learns data-dependent codes~\cite{gongIterativeQuantizationProcrustean2013}. This dependence is obtained by learning a matrix $\mat{W}$ that is best suited to our data. Formally this is given by the solution to the optimization problem
\begin{equation}
\min_{\mat{W}, \{ \vect{h}_i \}_{i=1}^{n}} \sum_{i=1}^{n} \norm{ \vect{h}_i - \mat{W} \vect{x}_i }{2}^2
\quad\text{s.t.}\quad
\begin{gathered}
\transpose{\mat{W}} \mat{W} = \mat{I} , \\
\vect{h}_i \in \{ 0, 1 \}^{d} .
\end{gathered}
\label[problem]{prob:itq}
\end{equation}
The problem is solved by alternate optimization: fixing $\mat{W}$ and solving for all $\vect{h}_i$, and viceversa.
The solution for the former problem is given by \cref{eq:lsh}.
Whereas the latter is an instance of the orthogonal Procrustes Problem. We already pointed out the similarities and differences between this approach and POSH.
ITQ produces dense binary codes that are of the same dimension as the input data. Hence, high-dimensional input data will produce high-dimensional hash codes. In order to work with compact codes (which is needed when codes are dense), the data is first projected into a lower-dimensional space and decorrelated using PCA.

\paragraph{KNNH.} This algorithm incorporates a preprocessing phase to ITQ~\cite{heKNearestNeighborsHashing2019}. We denote by $\set{N}(\vect{x}_i)$ the ground truth nearest neighbors of each point $\vect{x}_i$ (obtained by brute force search). In our experiments, we always use 20 nearest neighbors, as suggested in~\cite{heKNearestNeighborsHashing2019}. KNNH creates a new dataset $\{ \tilde{\vect{x}}_i \}_{i=1}^{n}$, where
\begin{equation}
	\tilde{\vect{x}}_i = \frac{1}{|\set{N}(\vect{x}_i)|} \sum_{\vect{y} \in \set{N}(\vect{x}_i)} \vect{y} .	
	\label{eq:knnh}
\end{equation}
This effectively produces a ``denoised'' version of the original dataset. Finally, hash codes are learned using ITQ. Notice that this method cannot truly run online, as the nearest neighbor computation requires to know the entire dataset in advance and has a complexity of $O(n^2)$. Nonetheless, we include it in our comparisons as KNNH is state-of-the-art in the LSH family of algorithms.

\paragraph{FruitFly.} We use the Bernoulli parameter $p=0.2$, following the experiments in \cref{fig:sparse_vs_orthogonal}. See \cref{sec:fruit_fly} for further details about the algorithm.

\paragraph{BioHash.} We use the parameters specified in~\cite[Appendix D]{ryaliBioInspiredHashingUnsupervised2020}. The sole exception is \cref{fig:bio_vs_spherical}, where we change the initial learning rate is changed purposely. See \cref{sec:biohash} for further details about the algorithm.

\begin{table}
	\caption{Characterization of the hashing methods used for our comparisons.}
	\label{tab:hashing_methods}
	
	\centering
	\begin{threeparttable}
		\begin{tabular}{lccc}
			\toprule
			\multicolumn{1}{c}{Method} & Hash codes & Dimensionality & Definition \\
			\midrule
			LSH~\cite{gionisSimilaritySearchHigh1999} & dense & compressed & \cref{eq:lsh}) \\
			ITQ~\cite{gongIterativeQuantizationProcrustean2013} & dense & compressed & \cref{prob:itq} \\
			KNNH~\cite{heKNearestNeighborsHashing2019} & dense & compressed & \cref{eq:knnh} \\
			FruitFly~\cite{dasguptaNeuralAlgorithmFundamental2017} & sparse & expanded & \cref{eq:fruit_fly_hashing} \\
			BOSL\tnote{1} & sparse & expanded & \cref{sec:osl} \\
			BioHash~\cite{ryaliBioInspiredHashingUnsupervised2020} & sparse & expanded & \cref{sec:biohash}\\
			SphericalHash & sparse & expanded & \cref{sec:biohash}\\
			POSH  & sparse & expanded & \cref{sec:posh} \\
			\bottomrule
		\end{tabular}
		\begin{tablenotes}
			\item [1] BOSL follows the spirit of OSL~\cite{liFastSimilaritySearch2018}, but uses a different optimization technique. See \cref{sec:osl} for more details.
		\end{tablenotes}
	\end{threeparttable}	
\end{table}

\subsection{Preprocessing}

For all methods, the data is preprocessed before hashing. We compute and store the data mean over the target dataset. Before hashing target and query vectors, we subtract the stored mean.

Additionally, whenever $d > D$, we compute and store the PCA transform of the target database. For dense hashing methods (\cref{tab:hashing_methods}) we set the number of PCA components to  $D$.

The case of $d > D$ merits a discussion for sparse hashing methods (\cref{tab:hashing_methods}).
These methods are designed to expand the dimensionality, not to reduce it. In fact, the matrix $\mat{W}$ in POSH cannot be orthogonal if we are reducing the dimensionality. For these reasons, when $d > D$, we use PCA, setting the number of components to $\alpha$.

\subsection{Training}

We observe that for most methods, training with a large-scale scale training set is not necessary. This also reflects real-world conditions, where the indexed dataset might be slightly different (and probably larger) than the training set. Of course, we assume that both datasets come from the same distribution. We thus randomly and uniformly sample 5000 examples from the target set of each dataset and use them for training. We use the same training set for every method. This comparison method also allows to train BOSL on exactly the same data as the other methods, as it cannot run on large-scale traning sets (the original OSL~\cite{liFastSimilaritySearch2018} shares the same limitation).

This training protocol can be detrimental to extract the full performance out of KNNH, see~\cref{eq:knnh}. However, for the reasons stated above we believe that hashing methods should be resilient to a training set sampled from the same distribution as the target set. As such, this is a problem with the method, and not with the training setup.

We use a mini-batches of 100 elements to train all methods.

\subsection{Evaluation metrics}

We use the standard Mean Average Precision (MAP), which averages precision over different recall, as our evaluation measure. In particular, we use MAP@n, where n is the number of retrieved target elements. We believe that MAP@n is a more representative measure for real use-cases of similarity search, as a full sorting of the target elements is rarely needed in practice. We also present results using Precision@n.

For each method, we always run 10 different trials (i.e., starting from different random initializations) and report aggregated results.

\section{Additional experimental results}
\label{sec:additional_experiments}

For completeness and reproducbility, we provide numerical results in \cref{tab:per_bits_mnist,tab:per_bits_mnist-gist,tab:per_bits_cifar10-gist,tab:per_bits_cifar10-vgg,tab:per_bits_labelme-gist,tab:per_bits_labelme-vgg} for the curves in \cref{fig:per_bits}.

\begin{table}
	\caption{MAP@1000 results corresponding to the MNIST dataset in \cref{fig:per_bits}}
	\label{tab:per_bits_mnist}

	\centering
	\begin{small}
	\begin{tabular}{l *{10}{ @{\hspace{8pt}} S[table-format=2.2,round-mode=places,round-precision=2] @{\hspace{8pt}} S[table-format=1.2,round-mode=places,round-precision=2]}}
		\toprule
        & \multicolumn{2}{c}{8 bits} & \multicolumn{2}{c}{16 bits} & \multicolumn{2}{c}{32 bits} & \multicolumn{2}{c}{64 bits} & \multicolumn{2}{c}{128 bits} \\
        \cmidrule(lr){2-3} \cmidrule(lr){4-5} \cmidrule(lr){6-7} \cmidrule(lr){8-9} \cmidrule(lr){10-11}
        & {MEAN}    & {STD}    & {MEAN}    & {STD}    & {MEAN}    & {STD}    & {MEAN}    & {STD}    & {MEAN}    & {STD}    \\
		\midrule
		LSH           &       41.299227 &        2.971746 &        54.819521 &         0.662192 &        62.248258 &         1.730992 &        69.490989 &         0.286361 &         75.541584 &          0.349362 \\
		ITQ           &       57.207546 &        1.037724 &        69.169752 &         1.026098 &        76.743339 &         0.340692 &        80.008127 &         0.323895 &         81.869453 &          0.298746 \\
		KNNH          &       58.243941 &        0.956279 &        69.909346 &         0.663315 &        77.956801 &         0.374958 &        81.201798 &         0.321392 &         83.148353 &          0.168069 \\
		FruitFly      &       58.186818 &        0.768243 &        67.258457 &         0.627975 &        74.007030 &         0.485671 &        78.370552 &         0.256711 &         80.790420 &          0.194562 \\
		BOSL          &       72.611763 &        0.298237 &        75.665313 &         0.338606 &        77.424644 &         0.198480 &        78.310337 &         0.287293 &         79.987173 &          0.389141 \\		
		BioHash       &       61.560242 &        0.957880 &        69.517839 &         0.507168 &        75.564542 &         0.383517 &        78.856061 &         0.364383 &         81.205214 &          0.274274 \\
		SphericalHash & \bfseries 84.156391 &        0.338577 & \bfseries 83.736119 &         0.230584 & \bfseries 82.111057 &         0.384816 &        80.370271 &         0.284635 &         78.820698 &          0.182206 \\
		POSH          &       75.260167 &        0.348947 &        79.452880 &         0.287062 &        81.491914 &         0.242512 & \bfseries 82.929763 &         0.235357 & \bfseries 83.656549 &          0.067164 \\
		\bottomrule	              
	\end{tabular}
	\end{small}
\end{table}

\begin{table}
	\caption{MAP@1000 results corresponding to the MNIST-GIST dataset in \cref{fig:per_bits}}
	\label{tab:per_bits_mnist-gist}

	\centering
	\begin{small}
	\begin{tabular}{l *{10}{ @{\hspace{8pt}} S[table-format=2.2,round-mode=places,round-precision=2] @{\hspace{8pt}} S[table-format=1.2,round-mode=places,round-precision=2]}}
		\toprule
        & \multicolumn{2}{c}{8 bits} & \multicolumn{2}{c}{16 bits} & \multicolumn{2}{c}{32 bits} & \multicolumn{2}{c}{64 bits} & \multicolumn{2}{c}{128 bits} \\
        \cmidrule(lr){2-3} \cmidrule(lr){4-5} \cmidrule(lr){6-7} \cmidrule(lr){8-9} \cmidrule(lr){10-11}
        & {MEAN}    & {STD}    & {MEAN}    & {STD}    & {MEAN}    & {STD}    & {MEAN}    & {STD}    & {MEAN}    & {STD}    \\
		\midrule
		LSH           &       46.365747 &        4.904722 &        56.800426 &         3.166077 &        67.885235 &         0.850751 &        75.861307 &         1.022221 &         80.695031 &          0.301907 \\
		ITQ           &       58.647502 &        1.212063 &        74.039202 &         0.717892 &        81.072637 &         0.230787 &        84.998139 &         0.306214 &         87.207430 &          0.054438 \\
		KNNH          &       60.343982 &        0.986644 &        75.126294 &         0.641789 &        82.682721 &         0.387531 &        86.811426 &         0.332023 &         89.016944 &          0.154251 \\
		FruitFly      &       68.415230 &        0.503170 &        76.694714 &         0.651476 &        81.357305 &         0.427243 &        83.874345 &         0.090010 &         85.567014 &          0.340684 \\
		BOSL          &       84.963282 &        0.427866 & \bfseries 87.086767 &         0.224680 &        88.026212 &         0.129135 &        87.637339 &         0.169720 &         87.493856 &          0.170119 \\
		BioHash       &       65.358007 &        1.344298 &        75.195877 &         0.711763 &        80.552712 &         0.415548 &        83.705769 &         0.263908 &         85.379093 &          0.283965 \\
		SphericalHash & \bfseries 87.299811 &        0.270418 &        86.099416 &         0.291720 &        83.947659 &         0.171461 &        80.874773 &         0.359662 &         77.050998 &          0.329308 \\
		POSH          &       83.136181 &        0.520277 &        87.077339 &         0.286185 & \bfseries 88.972998 &         0.192691 & \bfseries 90.042332 &         0.107513 & \bfseries 90.230717 &          0.098105 \\
		\bottomrule	              
	\end{tabular}
	\end{small}
\end{table}

\begin{table}
	\caption{MAP@1000 results corresponding to the CIFAR10-GIST dataset in \cref{fig:per_bits}}
	\label{tab:per_bits_cifar10-gist}

	\centering
	\begin{small}
	\begin{tabular}{l *{10}{ @{\hspace{8pt}} S[table-format=2.2,round-mode=places,round-precision=2] @{\hspace{8pt}} S[table-format=1.2,round-mode=places,round-precision=2]}}
		\toprule
        & \multicolumn{2}{c}{8 bits} & \multicolumn{2}{c}{16 bits} & \multicolumn{2}{c}{32 bits} & \multicolumn{2}{c}{64 bits} & \multicolumn{2}{c}{128 bits} \\
        \cmidrule(lr){2-3} \cmidrule(lr){4-5} \cmidrule(lr){6-7} \cmidrule(lr){8-9} \cmidrule(lr){10-11}
        & {MEAN}    & {STD}    & {MEAN}    & {STD}    & {MEAN}    & {STD}    & {MEAN}    & {STD}    & {MEAN}    & {STD}    \\
		\midrule
		LSH           &       18.125835 &        0.907542 &        20.221060 &         0.560220 &        23.825168 &         0.523075 &        25.723192 &         0.164368 &         28.011689 &          0.309649 \\
		ITQ           &       20.988552 &        0.290260 &        24.218053 &         0.229157 &        27.091732 &         0.160561 &        28.849303 &         0.160345 &         30.440780 &          0.046250 \\
		KNNH          &       20.972742 &        0.245054 &        24.805861 &         0.187142 &        27.735256 &         0.121716 &        29.686297 &         0.167386 &         31.170373 &          0.068730 \\
		FruitFly      &       21.831603 &        0.348243 &        25.052363 &         0.175397 &        27.487172 &         0.236882 &        29.757511 &         0.143634 &         31.314356 &          0.101964 \\
		BOSL          & \bfseries 29.856044 &        0.291349 & \bfseries 31.490123 &         0.223959 & \bfseries 32.148501 &         0.187008 &        32.048648 &         0.071377 &         32.463922 &          0.151833 \\
		BioHash       &       23.083825 &        0.293249 &        25.605629 &         0.284468 &        27.780720 &         0.227816 &        29.861828 &         0.173133 &         31.447358 &          0.186902 \\
		SphericalHash &       29.274271 &        0.210321 &        29.798273 &         0.180257 &        29.841731 &         0.113176 &        29.582895 &         0.070162 &         29.215383 &          0.114321 \\
		POSH          &       26.212244 &        0.227857 &        28.606460 &         0.169747 &        30.467426 &         0.177102 & \bfseries 32.157333 &         0.103945 & \bfseries 33.160734 &          0.108774 \\
		
		\bottomrule	              
	\end{tabular}
	\end{small}
\end{table}

\begin{table}
	\caption{MAP@1000 results corresponding to the CIFAR10-VGG dataset in \cref{fig:per_bits}}
	\label{tab:per_bits_cifar10-vgg}

	\centering
	\begin{small}
	\begin{tabular}{l *{10}{ @{\hspace{8pt}} S[table-format=2.2,round-mode=places,round-precision=2] @{\hspace{8pt}} S[table-format=1.2,round-mode=places,round-precision=2]}}
		\toprule
        & \multicolumn{2}{c}{8 bits} & \multicolumn{2}{c}{16 bits} & \multicolumn{2}{c}{32 bits} & \multicolumn{2}{c}{64 bits} & \multicolumn{2}{c}{128 bits} \\
        \cmidrule(lr){2-3} \cmidrule(lr){4-5} \cmidrule(lr){6-7} \cmidrule(lr){8-9} \cmidrule(lr){10-11}
        & {MEAN}    & {STD}    & {MEAN}    & {STD}    & {MEAN}    & {STD}    & {MEAN}    & {STD}    & {MEAN}    & {STD}    \\
		\midrule
        LSH           &       28.515979 &        0.816470 &        32.345490 &         1.331963 &        34.639979 &         1.380626 &        37.938387 &         0.616784 &         42.001554 &          0.528998 \\
        ITQ           &       34.028447 &        0.447975 &        39.153728 &         0.383220 &        41.576252 &         0.194501 &        44.239824 &         0.178528 &         46.412338 &          0.124988 \\
        KNNH          &       33.980778 &        0.333276 &        38.628789 &         0.457023 &        41.901513 &         0.170581 &        44.674243 &         0.151030 &         46.833274 &          0.116909 \\
        FruitFly      &       34.909903 &        0.286323 &        40.394011 &         0.352230 &        42.627627 &         0.202109 &        44.720214 &         0.330022 &         45.661232 &          0.243765 \\
		BOSL          &       38.580720 &        0.364655 &        44.790267 &         0.249881 &        47.162788 &         0.129765 &        49.067205 &         0.153955 &         48.507609 &          0.161135 \\
        BioHash       &       39.705205 &        0.274082 &        43.109969 &         0.383667 &        44.635643 &         0.317536 &        45.868998 &         0.252035 &         46.886013 &          0.235199 \\
        SphericalHash &       41.112100 &        0.107240 &        46.160091 &         0.120986 &        47.916236 &         0.173106 &        48.415863 &         0.193181 &         47.161759 &          0.263012 \\
        POSH          &       37.953799 &        0.239184 &        44.438826 &         0.195473 &        47.274094 &         0.117879 &        49.393478 &         0.119901 &         50.171443 &          0.060018 \\
		\bottomrule	              
	\end{tabular}
	\end{small}
\end{table}

\begin{table}
	\caption{MAP@1000 results corresponding to the LabelMe-12-50K-GIST dataset in \cref{fig:per_bits}}
	\label{tab:per_bits_labelme-gist}

	\centering
	\begin{small}
	\begin{tabular}{l *{10}{ @{\hspace{8pt}} S[table-format=2.2,round-mode=places,round-precision=2] @{\hspace{8pt}} S[table-format=1.2,round-mode=places,round-precision=2]}}
	\toprule
        & \multicolumn{2}{c}{8 bits} & \multicolumn{2}{c}{16 bits} & \multicolumn{2}{c}{32 bits} & \multicolumn{2}{c}{64 bits} & \multicolumn{2}{c}{128 bits} \\
        \cmidrule(lr){2-3} \cmidrule(lr){4-5} \cmidrule(lr){6-7} \cmidrule(lr){8-9} \cmidrule(lr){10-11}
        & {MEAN}    & {STD}    & {MEAN}    & {STD}    & {MEAN}    & {STD}    & {MEAN}    & {STD}    & {MEAN}    & {STD}    \\
		\midrule
		LSH           &       15.306649 &        0.362580 &        16.344039 &         0.435870 &        19.327935 &         0.690045 &        21.107218 &         0.346002 &         23.040022 &          0.154999 \\
		ITQ           &       17.652506 &        0.120258 &        20.523762 &         0.177075 &        22.426868 &         0.266526 &        23.735533 &         0.192619 &         24.950618 &          0.079899 \\
		KNNH          &       17.855069 &        0.096692 &        20.952772 &         0.145158 &        22.814321 &         0.202137 &        24.274385 &         0.139926 &         25.431793 &          0.141145 \\
		FruitFly      &       18.070155 &        0.458770 &        20.436336 &         0.195987 &        22.391403 &         0.238587 &        24.429974 &         0.217589 &         25.517787 &          0.019855 \\
		BOSL          & \bfseries 24.398816 &        0.273531 & \bfseries 25.593448 &         0.210736 & \bfseries 26.404583 &         0.220642 & \bfseries 26.532455 &         0.167699 &         26.726800 &          0.201700 \\		
		BioHash       &       18.625815 &        0.323149 &        20.732482 &         0.191230 &        22.721094 &         0.243305 &        24.327879 &         0.144493 &         25.443924 &          0.196720 \\
		SphericalHash &       24.313733 &        0.112083 &        24.629330 &         0.160349 &        24.616089 &         0.182575 &        24.464447 &         0.121860 &         23.997034 &          0.146686 \\
		POSH          &       22.303395 &        0.308398 &        23.969311 &         0.220961 &        25.436416 &         0.146971 &        26.437014 &         0.088704 & \bfseries 27.212659 &          0.128882 \\
		
		\bottomrule	              
	\end{tabular}
	\end{small}
\end{table}

\begin{table}
	\caption{MAP@1000 results corresponding to the LabelMe-12-50K-VGG dataset in \cref{fig:per_bits}}
	\label{tab:per_bits_labelme-vgg}

	\centering
	\begin{small}
	\begin{tabular}{l *{10}{ @{\hspace{8pt}} S[table-format=2.2,round-mode=places,round-precision=2] @{\hspace{8pt}} S[table-format=1.2,round-mode=places,round-precision=2]}}
	\toprule
        & \multicolumn{2}{c}{8 bits} & \multicolumn{2}{c}{16 bits} & \multicolumn{2}{c}{32 bits} & \multicolumn{2}{c}{64 bits} & \multicolumn{2}{c}{128 bits} \\
        \cmidrule(lr){2-3} \cmidrule(lr){4-5} \cmidrule(lr){6-7} \cmidrule(lr){8-9} \cmidrule(lr){10-11}
        & {MEAN}    & {STD}    & {MEAN}    & {STD}    & {MEAN}    & {STD}    & {MEAN}    & {STD}    & {MEAN}    & {STD}    \\
		\midrule
        LSH           &       18.002353 &        0.864707 &        20.510355 &         0.695885 &        23.228842 &         0.521026 &        26.419685 &         0.172882 &         29.050196 &          0.362776 \\
        ITQ           &       22.800149 &        0.519947 &        25.739648 &         0.506300 &        28.420269 &         0.239005 &        31.310452 &         0.246265 &         33.761422 &          0.132957 \\
        KNNH          &       23.060154 &        0.281613 &        25.309556 &         0.304017 &        28.904332 &         0.339931 &        31.685577 &         0.247715 &         34.425609 &          0.120468 \\
        FruitFly      &       21.667246 &        0.141275 &        25.937632 &         0.217119 &        29.925684 &         0.212039 &        32.197673 &         0.176762 &         33.690128 &          0.280115 \\
        BOSL          &       24.994246 &        0.233777 &        29.336024 &         0.159031 &        34.187897 &         0.215678 &        35.305011 &         0.112832 &         36.132754 &          0.148445 \\
        BioHash       &       26.587151 &        0.199711 &        28.999849 &         0.327159 &        31.404585 &         0.369695 &        33.270819 &         0.242563 &         34.453950 &          0.201414 \\
        SphericalHash &       27.222900 &        0.125988 &        31.107173 &         0.227212 &        34.323932 &         0.293342 &        34.557080 &         0.210491 &         33.639905 &          0.216453 \\
        POSH          &       24.469864 &        0.126122 &        29.210659 &         0.223633 &        32.920120 &         0.274864 &        36.117830 &         0.242022 &         37.632215 &          0.173324 \\
		\bottomrule	              
	\end{tabular}
	\end{small}
\end{table}

We include in \cref{fig:per_bits} but measuring Precision@1000 instead of MAP@1000. Similar results are observed in both figures.

\begin{figure}
	\centering
	\begin{small}
		\begin{tabu} to \textwidth {*{3}{ @{\hspace{0pt}} X[c,m] @{\hspace{0pt}} }}
			MNIST & CIFAR10-GIST & LabelMe-12-50K-GIST \\
			\includegraphics[width=\linewidth]{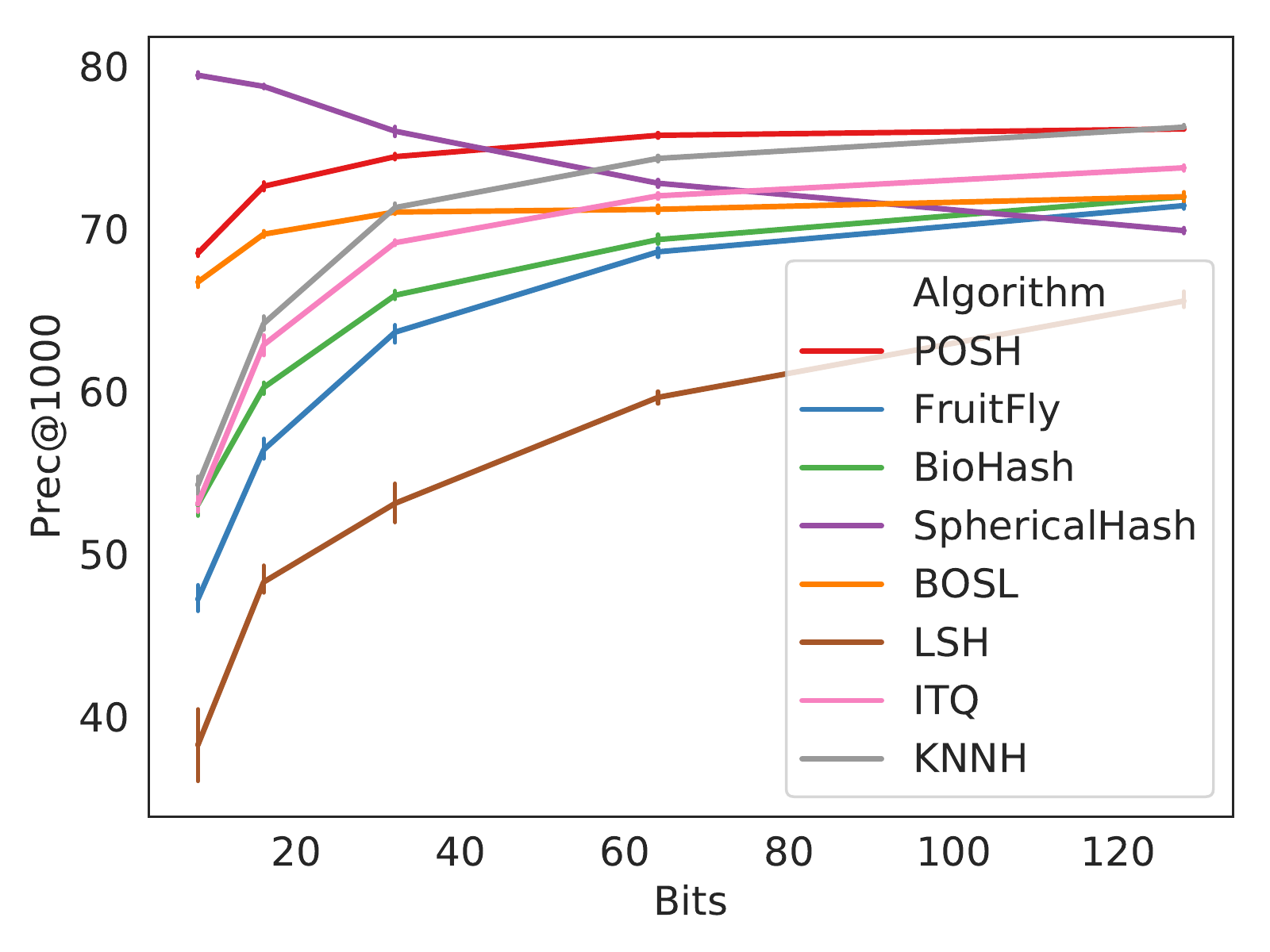} &
			\includegraphics[width=\linewidth]{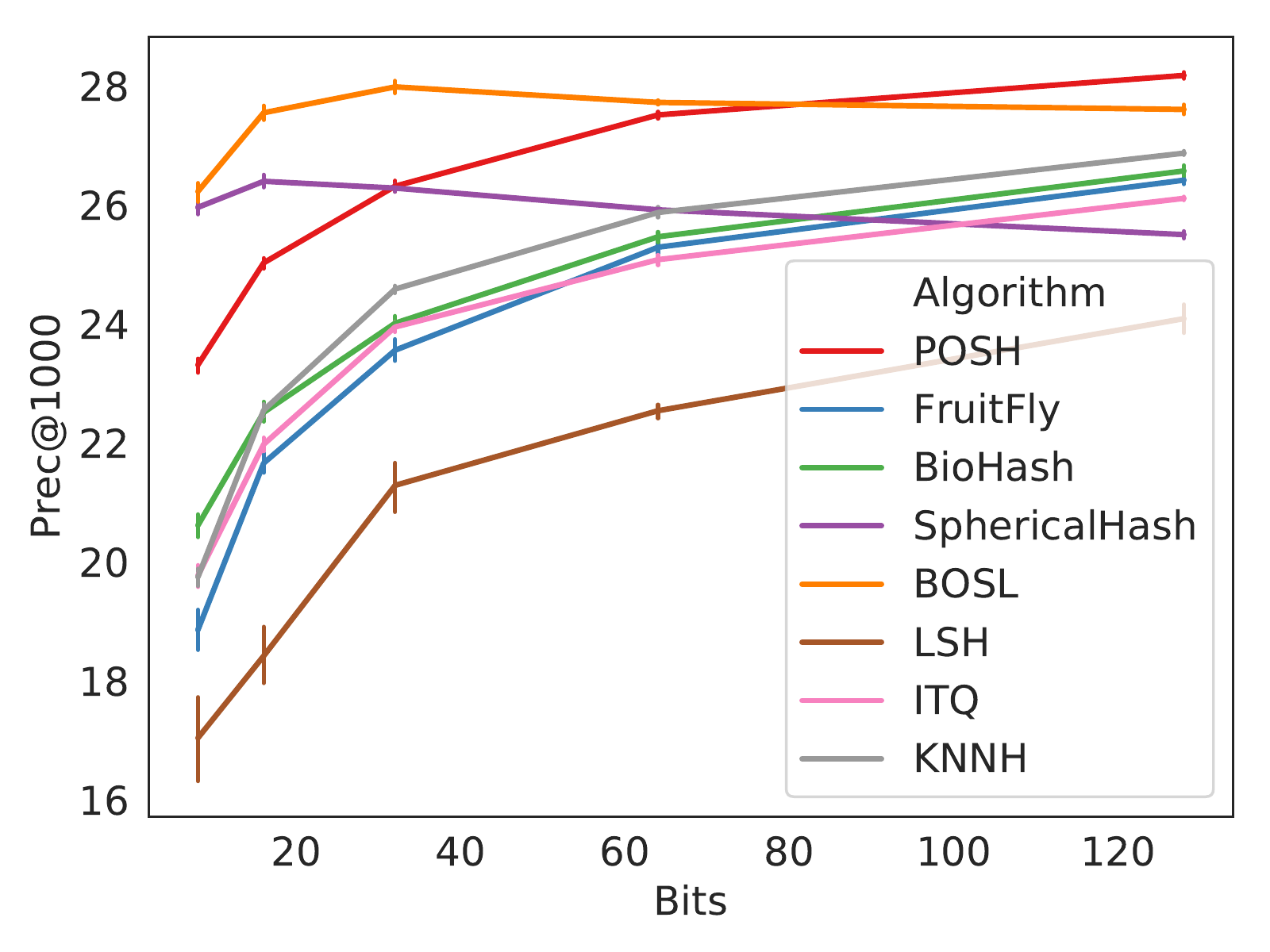} &
			\includegraphics[width=\linewidth]{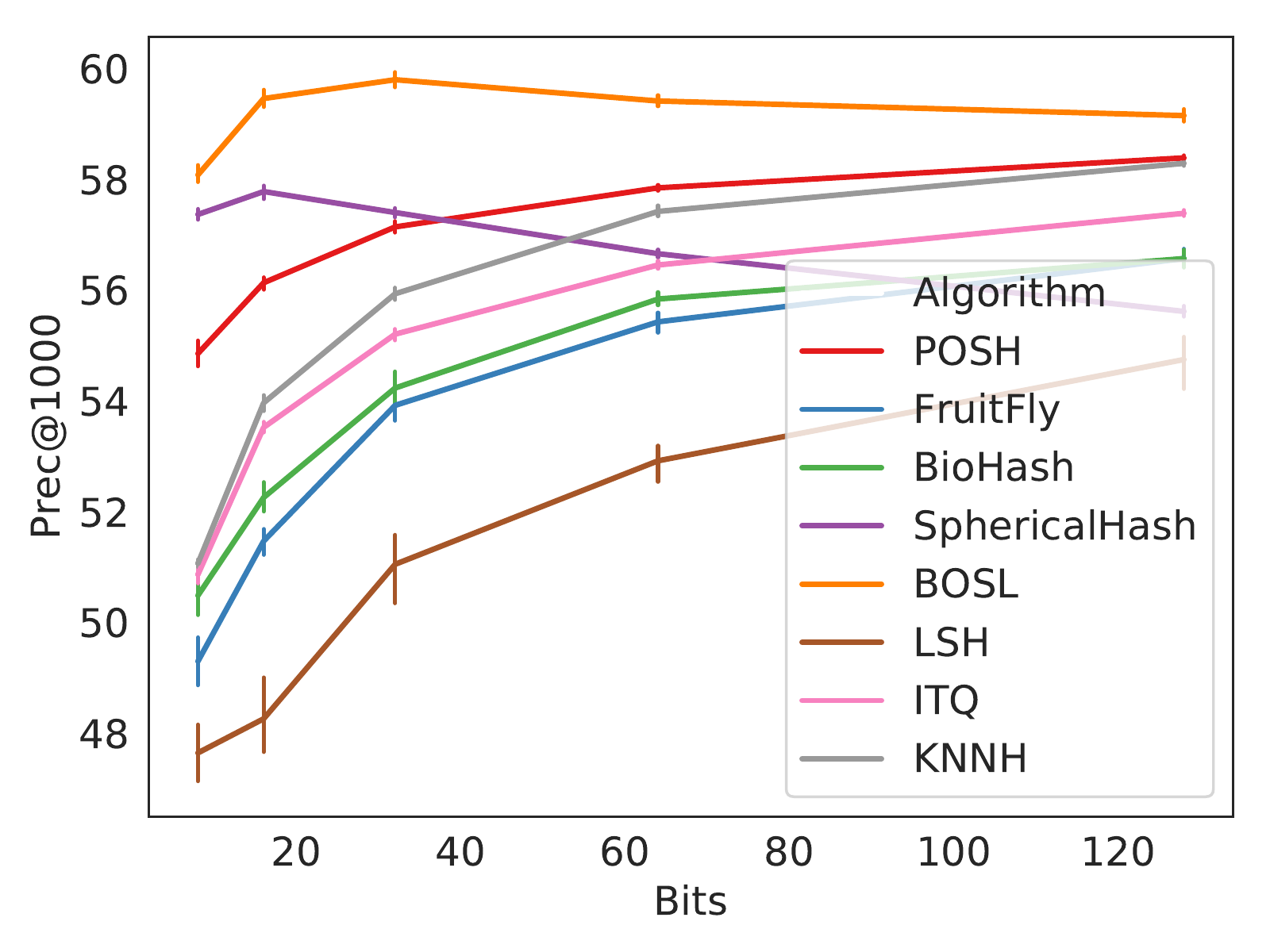} \\
			MNIST-GIST & CIFAR10-VGG & LabelMe-12-50K-VGG \\
			\includegraphics[width=\linewidth]{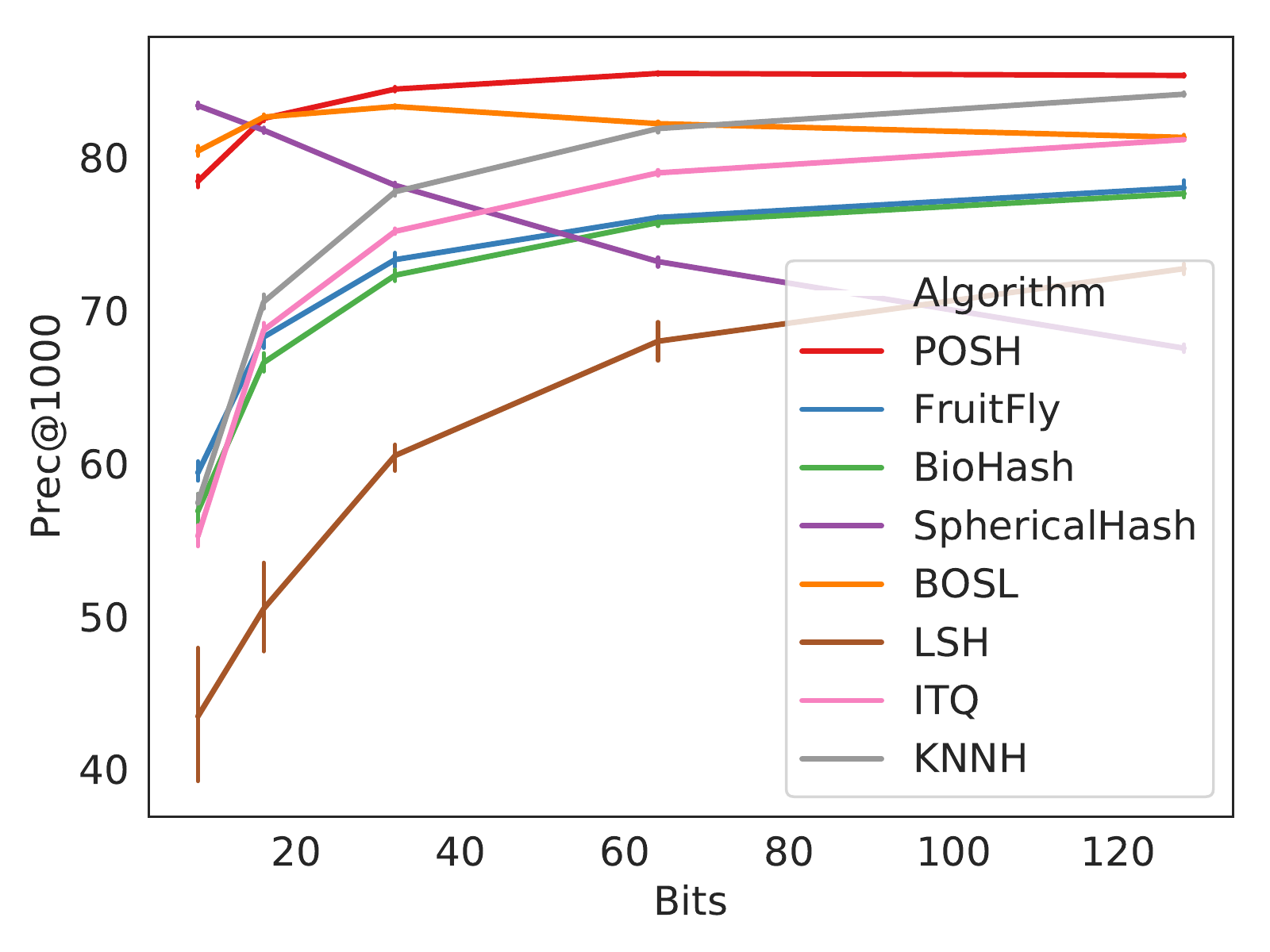} &
			\includegraphics[width=\linewidth]{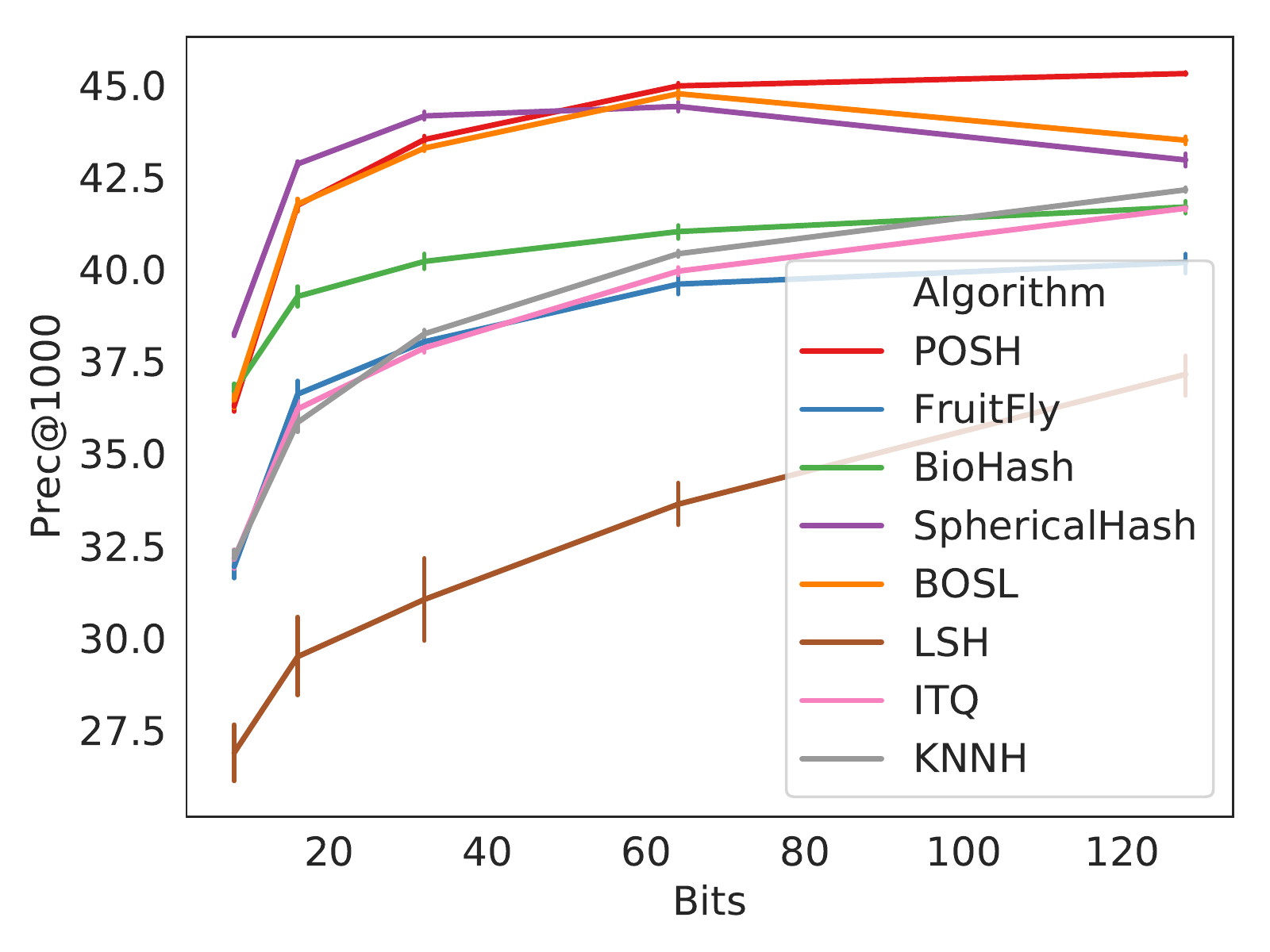} &
			\includegraphics[width=\linewidth]{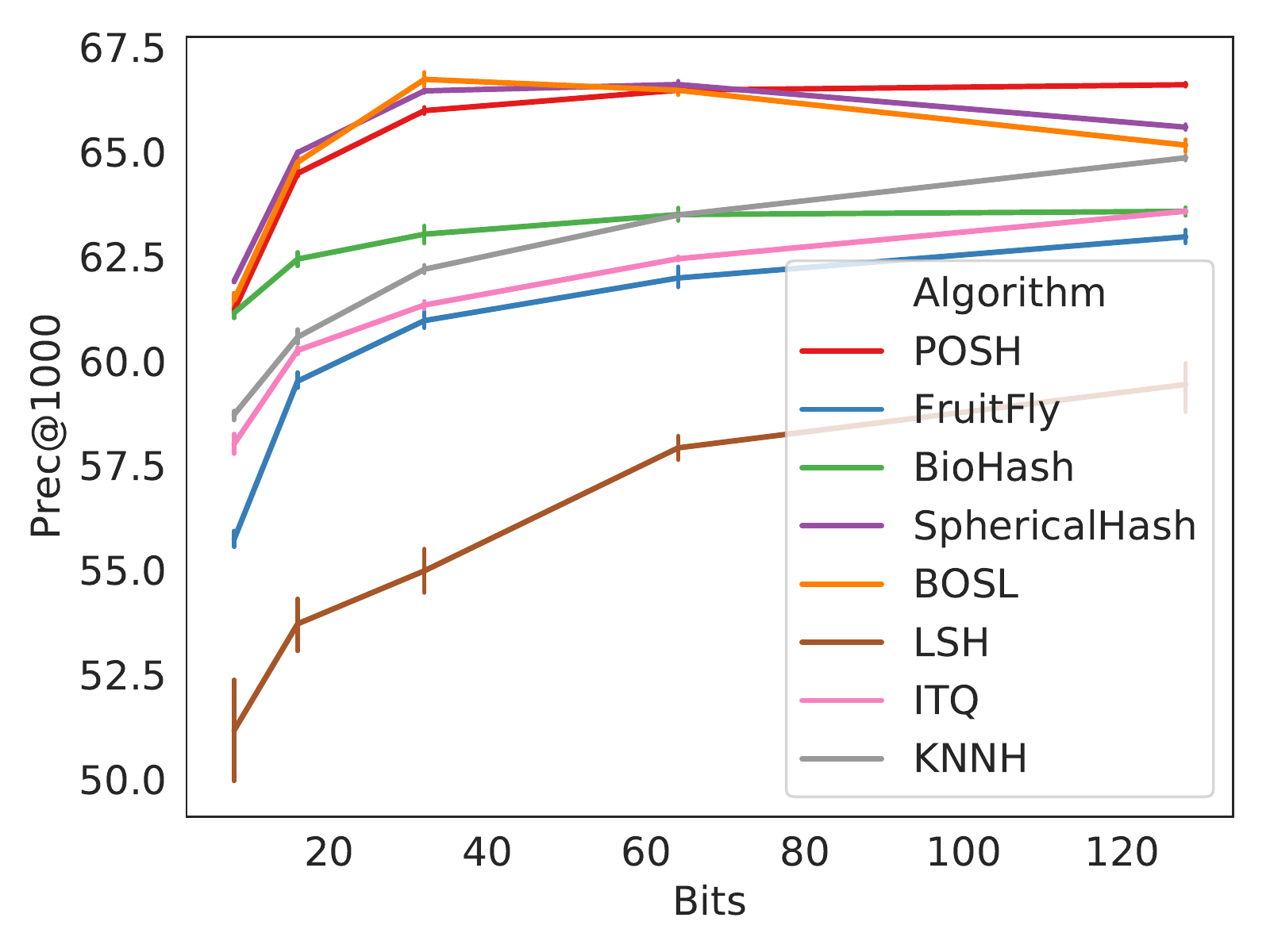} \\
		\end{tabu}
	\end{small}
	
	\caption{Comparison of different hashing methods under different configurations. The abscissa represents the hash length $D$ for dense codes (LSH, ITQ, KNNH) and the number $\alpha$ of set bits for sparse codes (POSH, FruitFly, BioHash, SphericalHash, BOSL). For sparse codes, $D=1024$. Error bars represent 95\% confidence intervals. Additional results are available in \cref{sec:additional_experiments}.}
	\label{fig:per_bits_precision}
\end{figure}

%\begin{figure}
%	\centering
%	\includegraphics[width=0.45\textwidth]{comparison_per_bits/set_bits-ANN_SIFT10K-MAP_100}
%	
%%		\begin{tabu} to .65\textwidth {*{2}{ @{\hspace{0pt}} X[c,m] @{\hspace{0pt}} }}
%%			\includegraphics[width=\linewidth]{comparison_per_bits/set_bits-ANN_SIFT10K-MAP_100} &
%%			\includegraphics[width=\linewidth]{comparison_per_bits/set_bits-ANN_SIFT10K-1-Rec_100} \\
%%		\end{tabu}
%
%	\caption{Comparison between different unsupervised methods}
%	\label{fig:per_bits_sift10k}
%\end{figure}

In \cref{fig:hash_length}, we compare the performance of FruitFly, SphericalHash, and POSH for $\alpha \in \{16, 32, and 64 \}$ across different hash lengths $D$. As in \cref{fig:per_bits}, SphericalHash performs generally better when $\alpha=16$ across different values of $D$, while POSh dominates when $\alpha=64$.

\begin{figure}[t]
	\centering
	\begin{small}
	\begin{tabu} to \textwidth { @{\hspace{0pt}} X[c,m,1] *{3}{ @{\hspace{0pt}} X[c,m,10] @{\hspace{0pt}} }}
		& $\alpha=16$ & $\alpha=32$ & $\alpha=64$ \\
		\begin{sideways}MNIST\end{sideways} &
		\includegraphics[width=\linewidth]{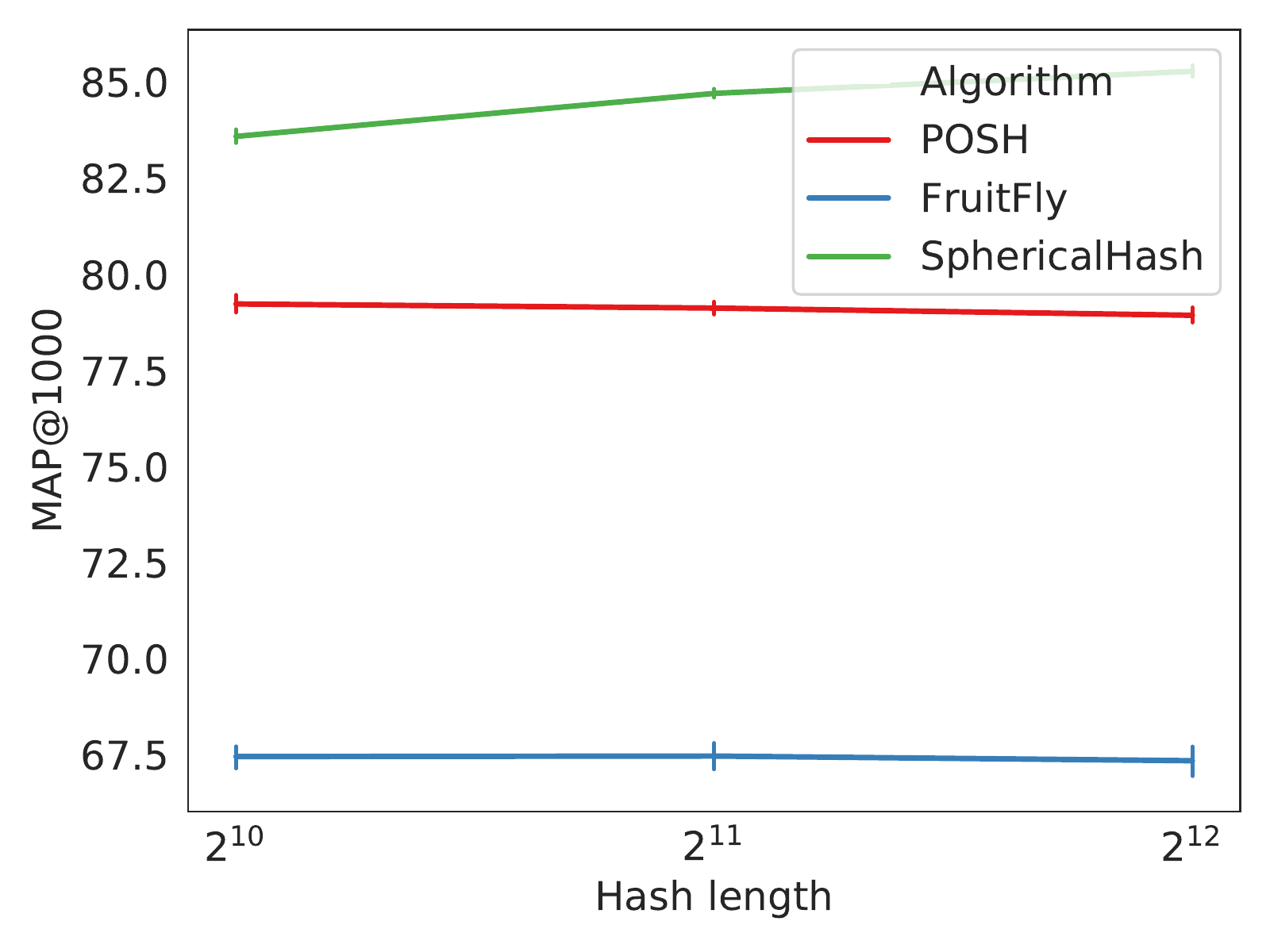} &
		\includegraphics[width=\linewidth]{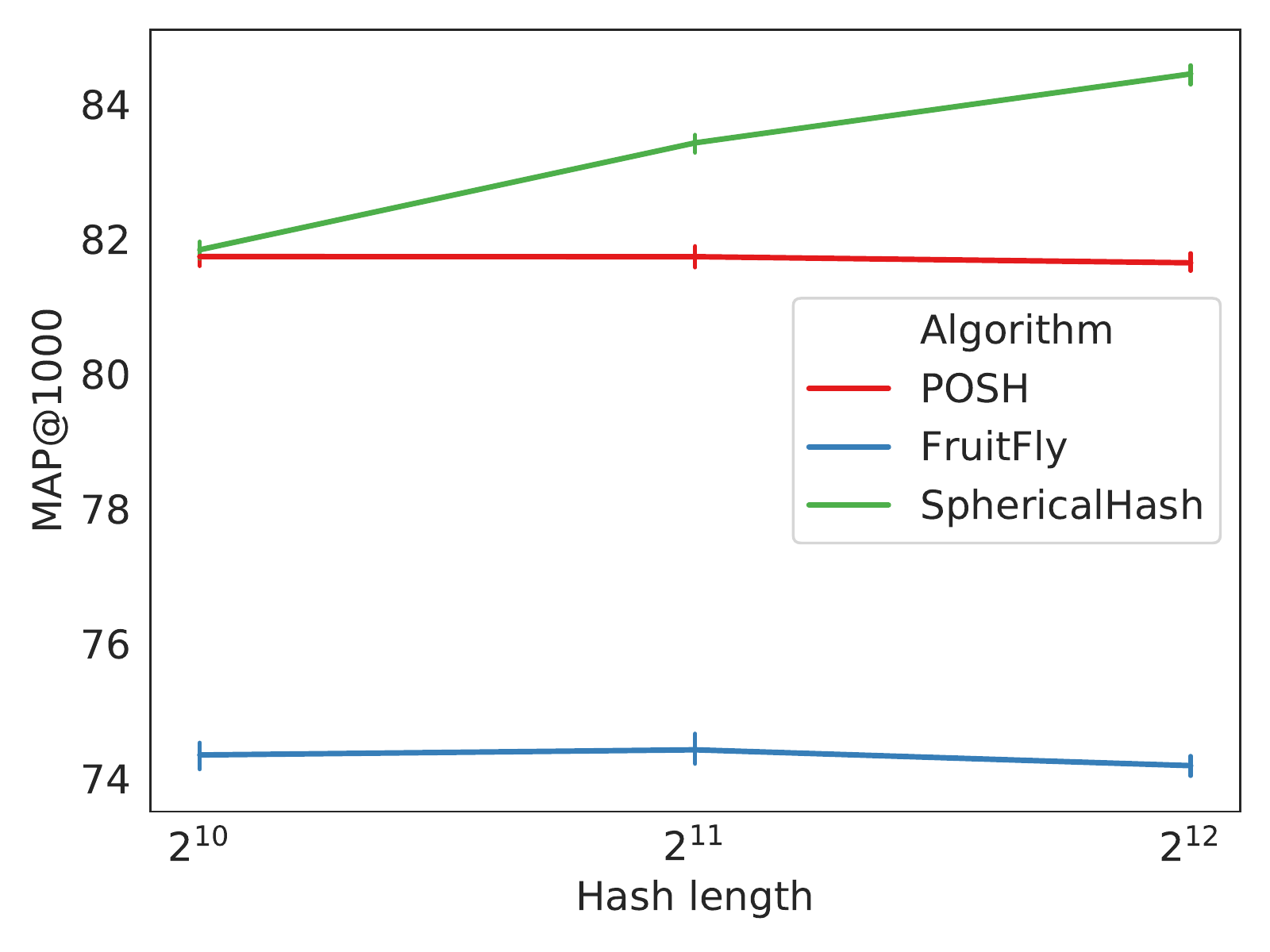} &
		\includegraphics[width=\linewidth]{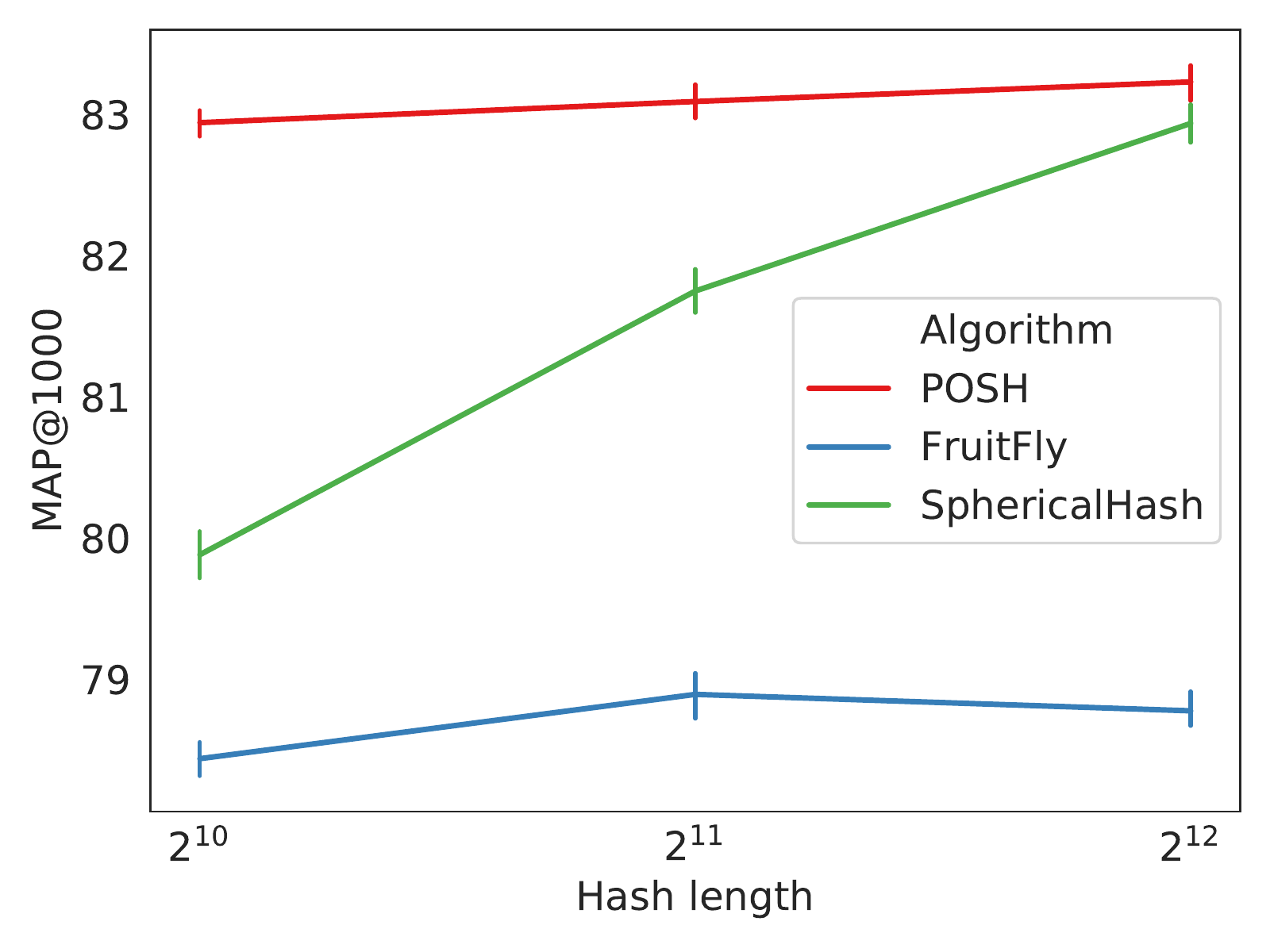} \\
		\begin{sideways}MNIST-GIST\end{sideways} &
		\includegraphics[width=\linewidth]{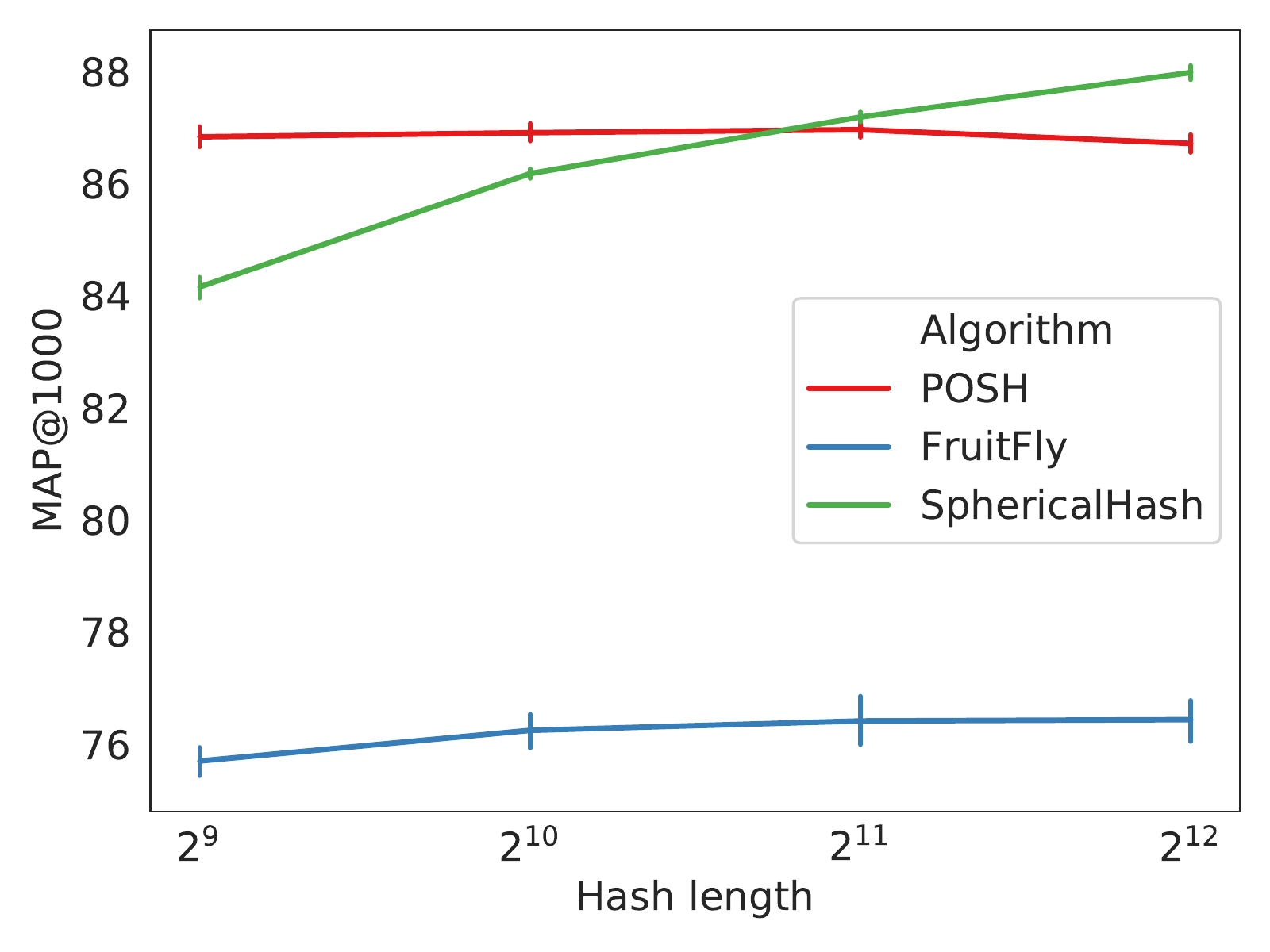} &
		\includegraphics[width=\linewidth]{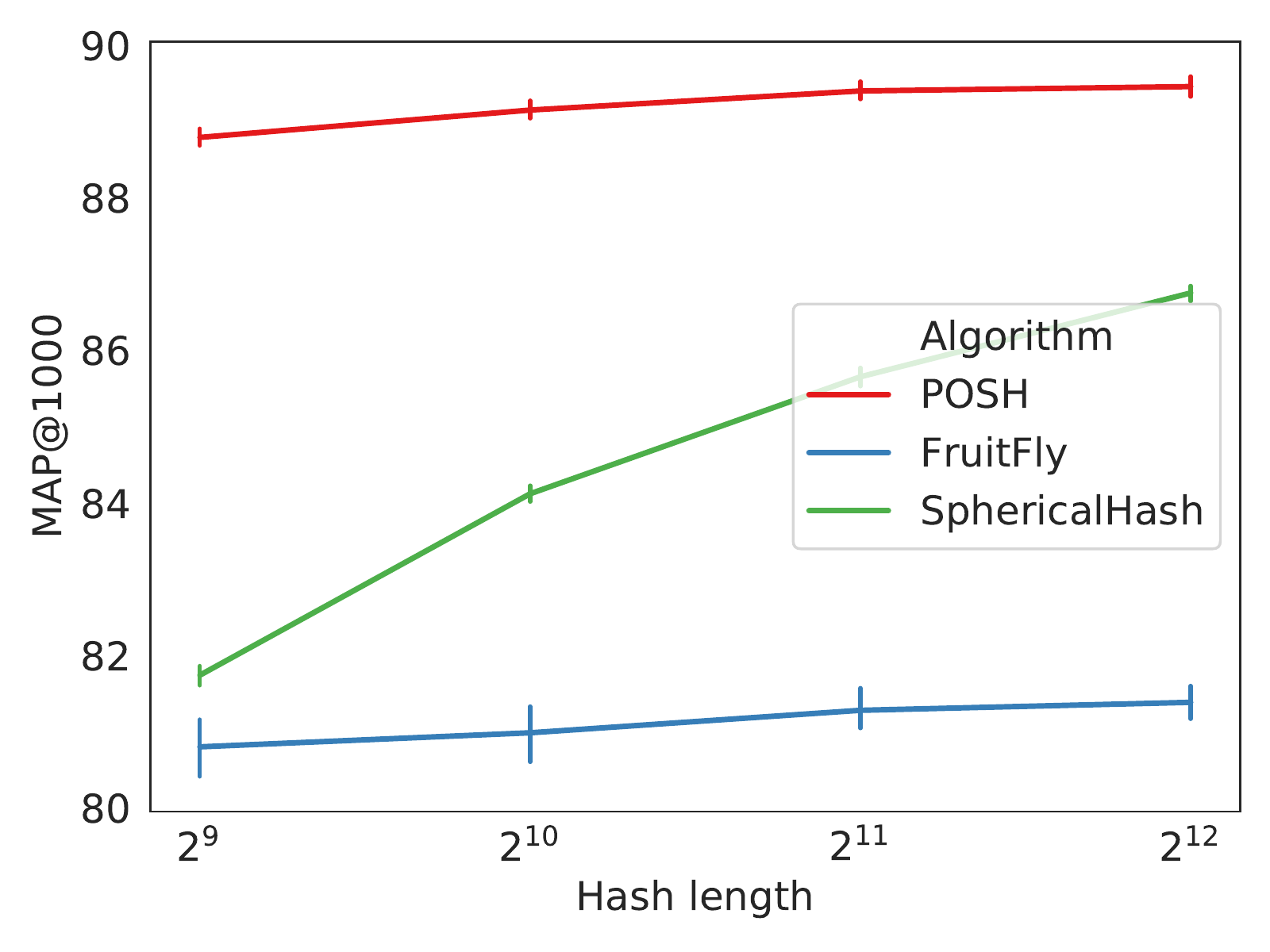} &
		\includegraphics[width=\linewidth]{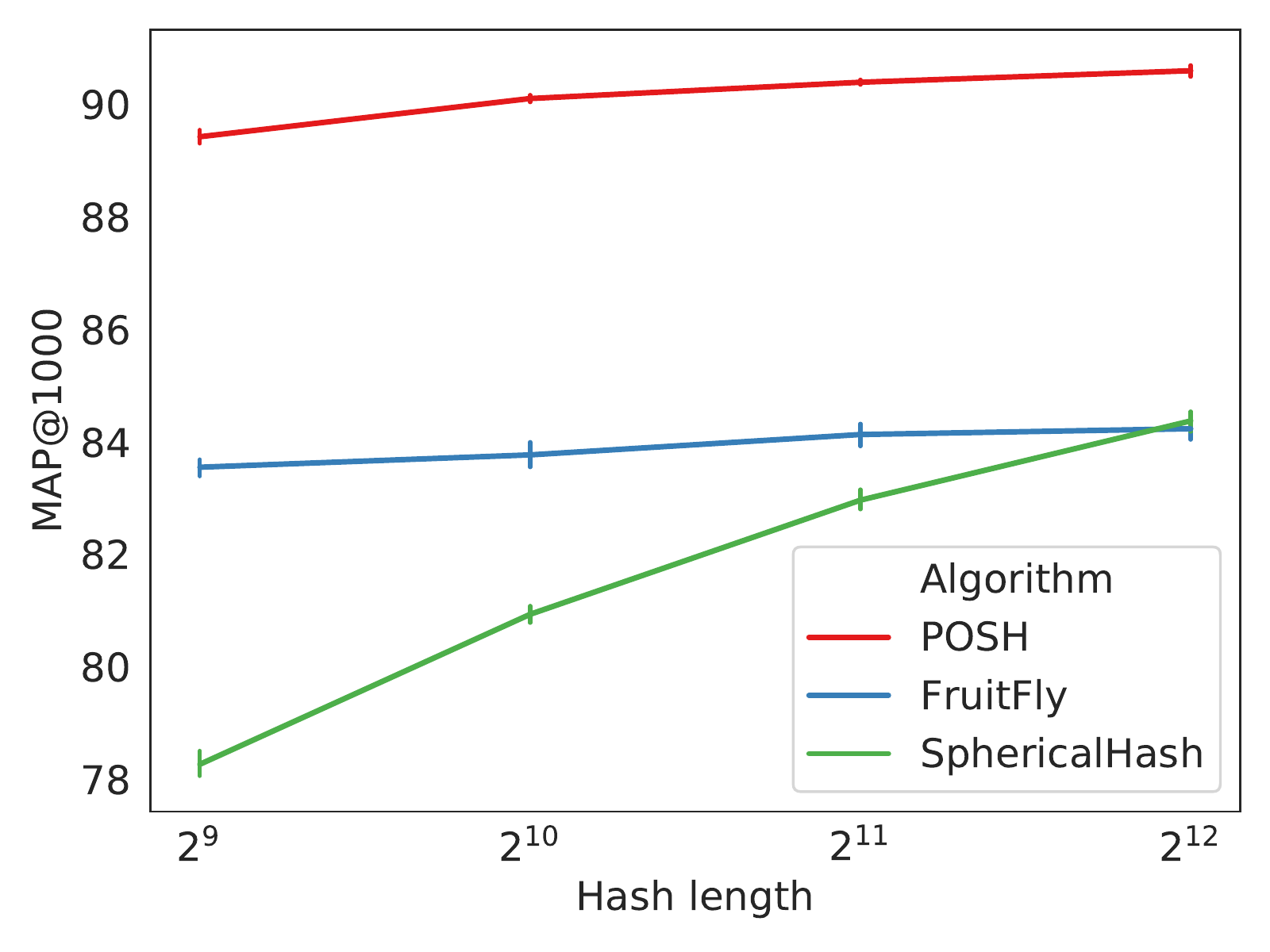} \\
		\begin{sideways}CIFAR10-GIST\end{sideways} &
		\includegraphics[width=\linewidth]{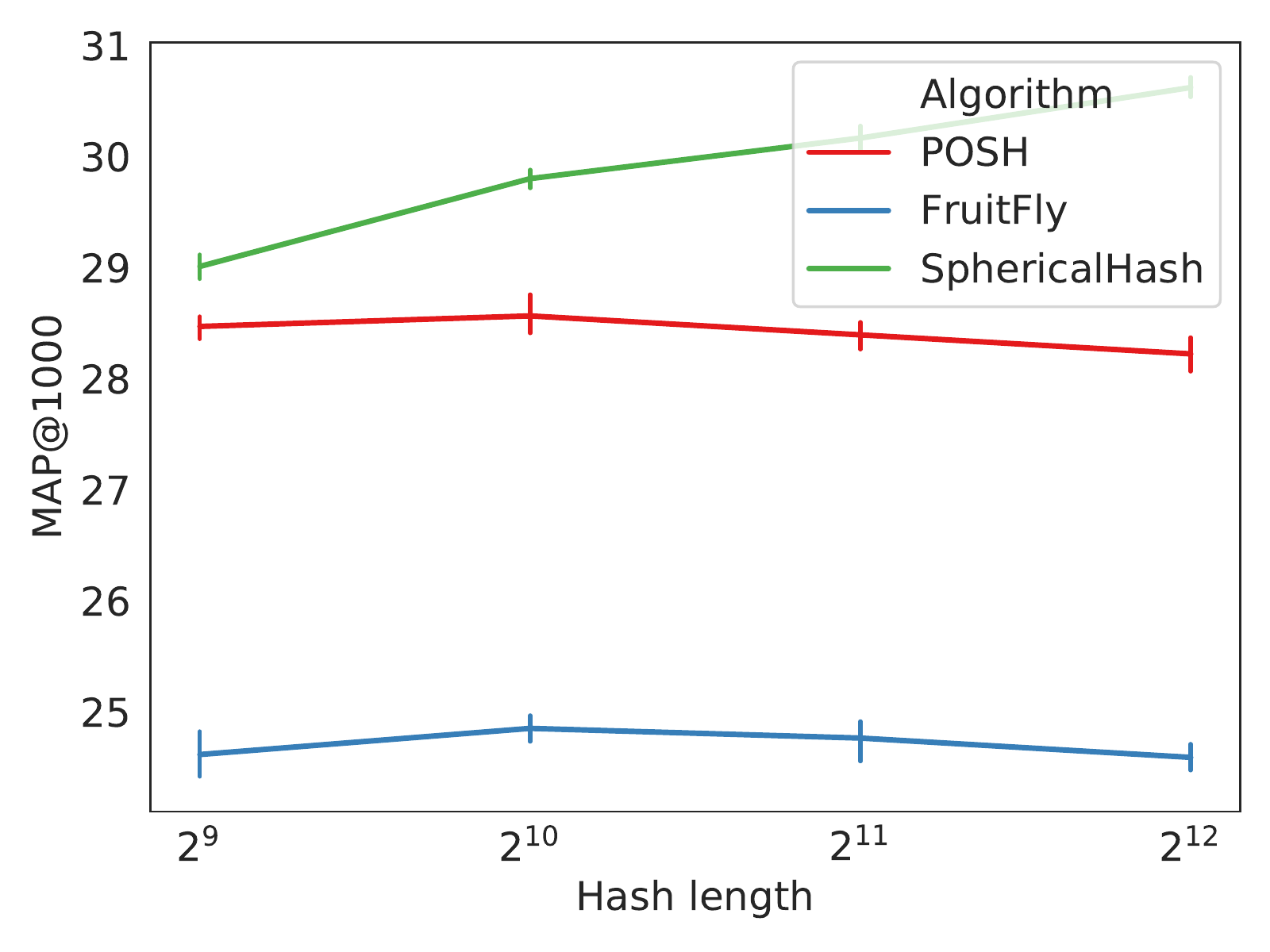} &
		\includegraphics[width=\linewidth]{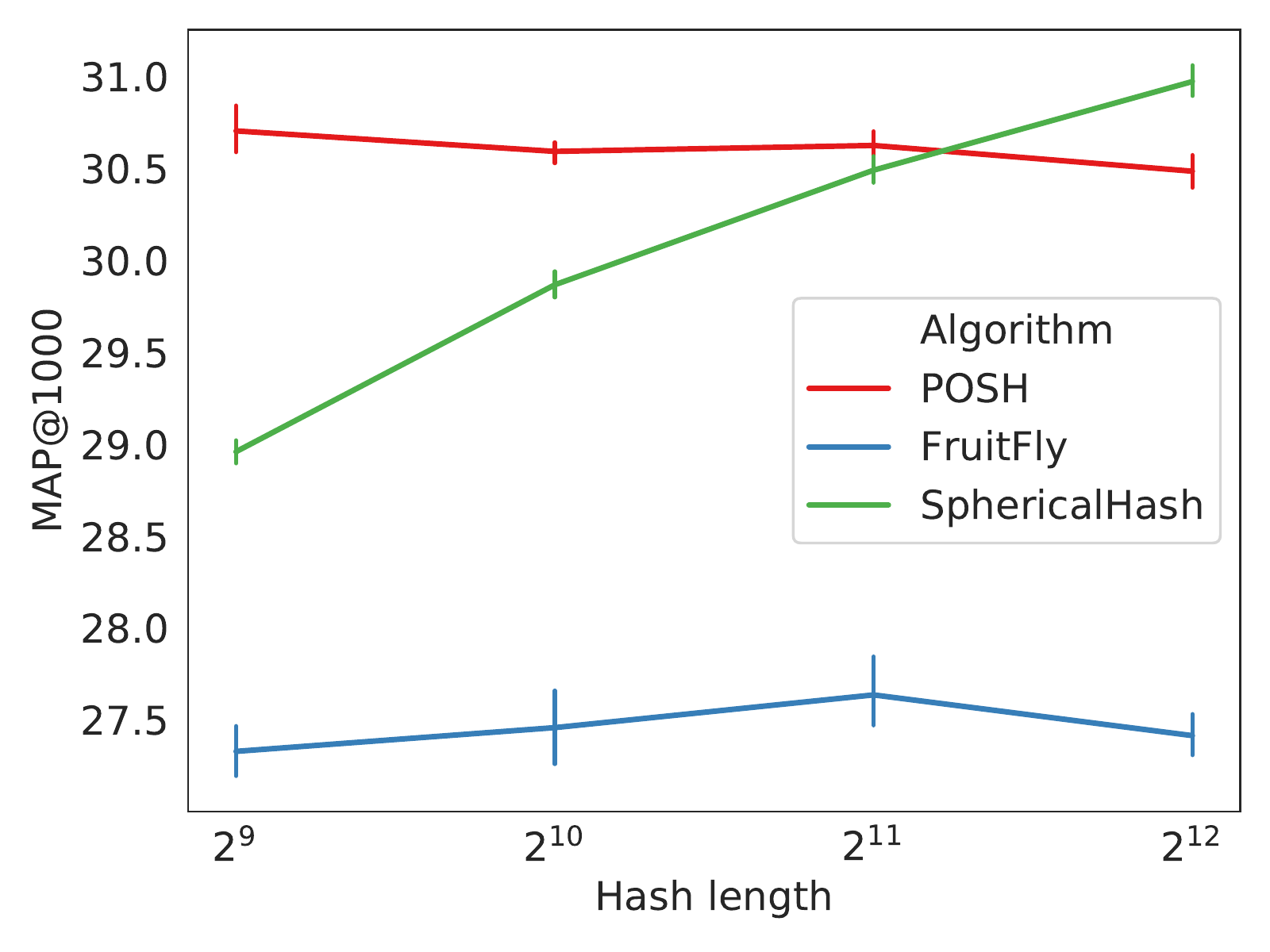} &
		\includegraphics[width=\linewidth]{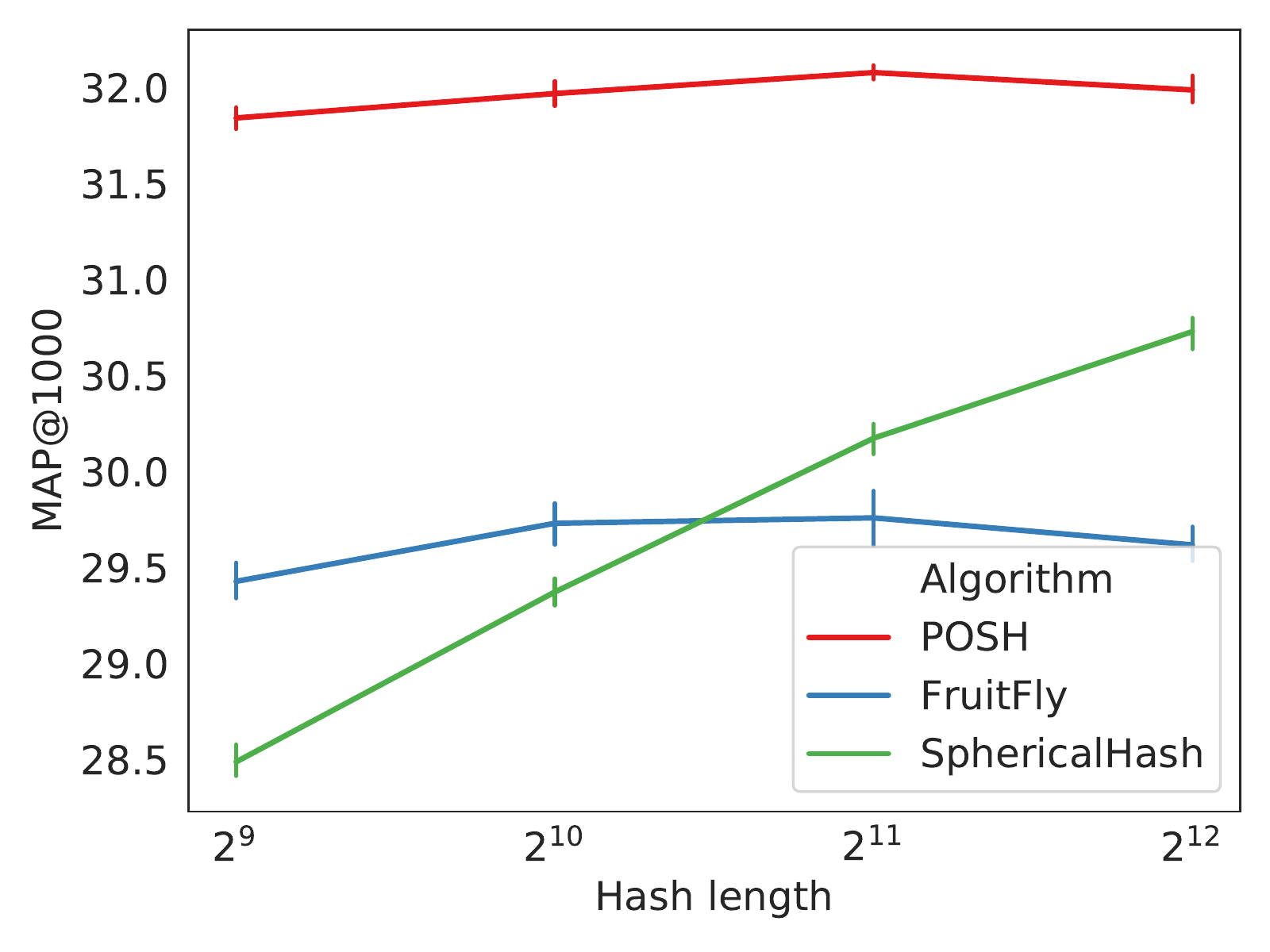} \\
		\begin{sideways}LabelMe-12-50K-GIST\end{sideways} &
		\includegraphics[width=\linewidth]{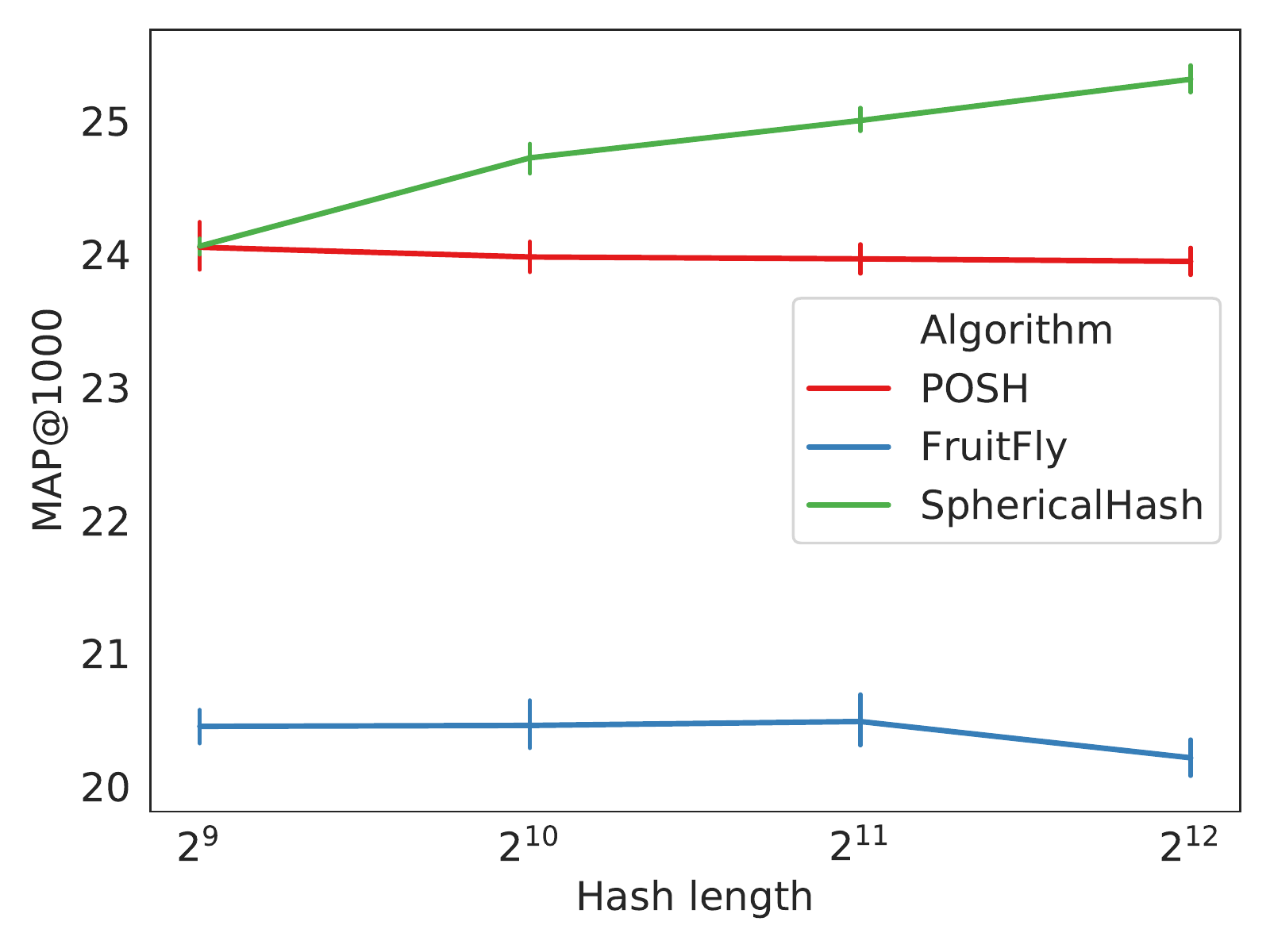} &
		\includegraphics[width=\linewidth]{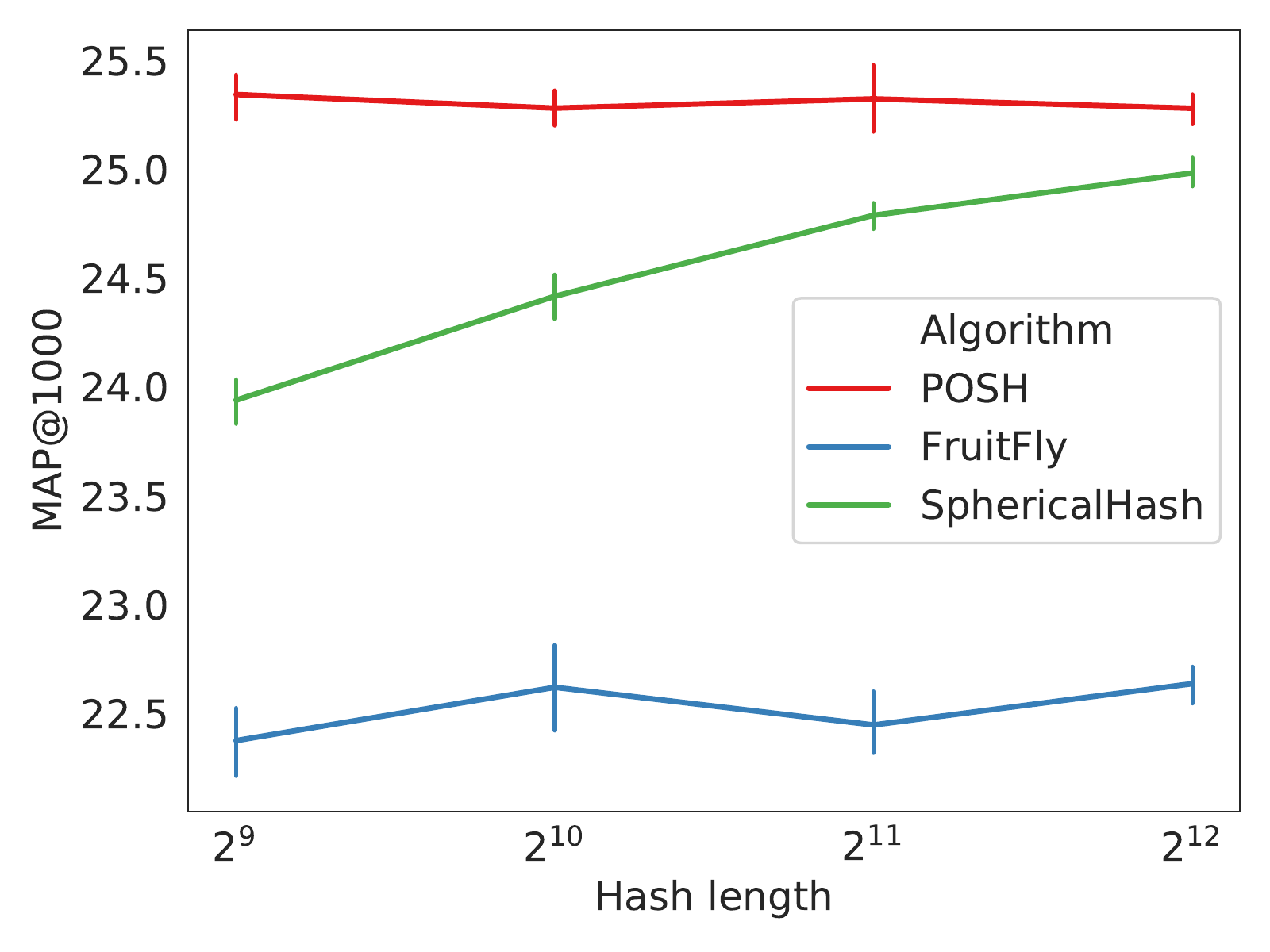} &
		\includegraphics[width=\linewidth]{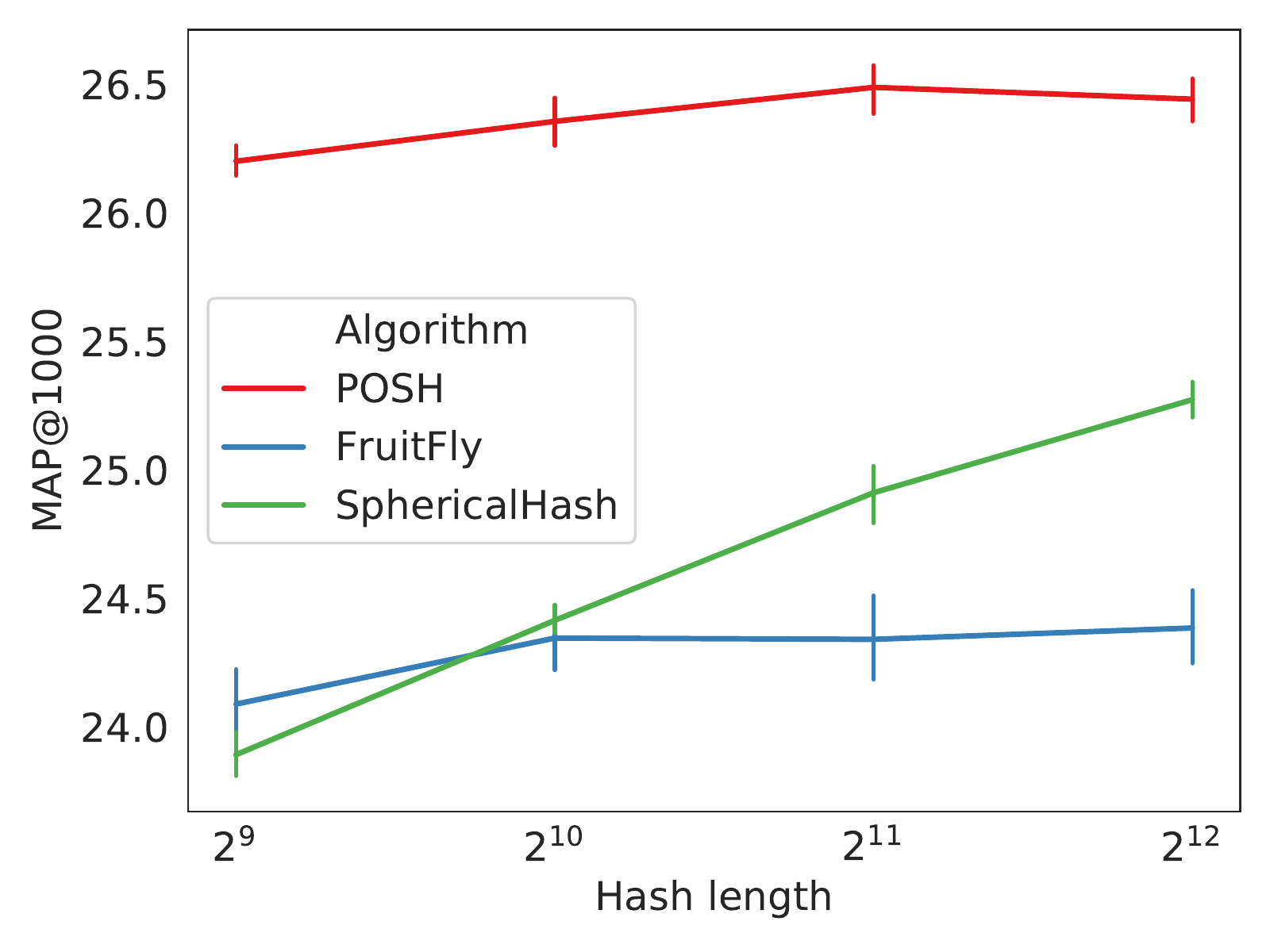} \\
	\end{tabu}
	\end{small}

	\caption{Comparison of different sparse hashing methods under different numbers $\alpha$ of set bits. The abscissa represents the hash length $D$. Error bars represent 95\% confidence intervals.}
	\label{fig:hash_length}
\end{figure}

In \cref{tab:places205_precision}, we present the same comparison as in \cref{tab:places205_map}, but measuring Precision@1000 instead of MAP@1000. Again, similar results are observed in both tables.

\begin{table}[t]
	\caption{Results (Precision@1000) for the large-scale dataset Places205 (approximately 2.5M targets and 20K queries).}
	\label{tab:places205_precision}

	\centering
	\begin{small}
	\begin{tabular} {l *{7}{S[table-format=2.2]}}
		\toprule
		& {LSH} & {ITQ} & KNNH & {FruitFly} & {SphericalHash} & {POSH} & {POSH+CR2} \\
		\midrule
		16 bits &  8.36 & 12.09 & 12.15 & 23.97 & 24.23 & \bfseries 26.81 & 27.94 \\
		32 bits & 15.75 & 20.02 & 19.63 & 26.91 & 24.26 & \bfseries 28.81 & 29.42 \\
		64 bits & 22.52 & 25.50 & 25.29 & 28.65 & 24.03 & \bfseries 29.96 & 30.13 \\
		\bottomrule
	\end{tabular}
	\end{small}	
\end{table}

\clearpage

\section{Reproducibility}

\begin{verbatim}
import numpy as np


def compute_hash(W, X, n_bits, center=None):
    """
    :param W: numpy array with shape (d, D)
    :param X: Input data, numpy array with shape (n, d)
    :param n_bits: number of set bits (alpha)
    :param center: numpy array with shape(1, D) with a vector at
    which to center the data (commonly its mean).
    If None, no re-centering occurs.
    :return: Hash codes, binary numpy array with shape (n, D)
    """
    if center is not None:
        X = X - center

    Y = X.dot(W)
    
    idx_rows = np.arange(len(Y))[:, np.newaxis]
    
    idx_cols = np.argpartition(-Y, n_bits - 1, axis=1)
    idx_cols = idx_cols[:, :n_bits]
    
    H = np.zeros_like(Y, dtype=np.int)
    H[idx_rows, idx_cols] = 1
    
    return H


def learn_POSH(X, D, n_bits, n_epochs=50, mini_batch_size=100):
    """
    :param X: Input data, numpy array with shape (n, d). We assume that
    it has been properly re-centered (de-meaned)
    :param D: total number of bits in each hash code
    :param n_bits: number of set bits (alpha)
    :param n_epochs: number of training epochs
    :param mini_batch_size: number of samples in each mini-batch
    :return: Projection matrix, numpy array with shape (d, D)
    """
    W = np.random.randn(X.shape[1], D)
    U, _, Vt = np.linalg.svd(W, full_matrices=False)
    W = U.dot(Vt)
    
    M = W
    for t in range(n_epochs):
        for i in range(0, X.shape[0], mini_batch_size)
            Y = X[i:i + batch_size]
            H = compute_hash(W, Y, n_bits)
            M += Y.T.dot(H)
    	
            U, _, V = np.linalg.svd(M, full_matrices=False)
            W = U.dot(V)
            
    return W



def spherical_kmeans(X, D, n_epochs=50, mini_batch_size=100):
    """
    :param X: Input data, numpy array with shape (n, d). We assume that
    it has been properly re-centered (de-meaned)
    :param D: total number of bits in each hash code
    :param n_epochs: number of training epochs
    :param mini_batch_size: number of samples in each mini-batch
    :return: Projection matrix, numpy array with shape (d, D)
    """
    W = np.random.randn(X.shape[1], D)
    W = W / np.linalg.norm(W, axis=0, keepdims=True)
    
    for t in range(n_epochs):
        M = 0
        for i in range(0, X.shape[0], mini_batch_size)
            Y = X[i:i + batch_size]
            Y /= np.linalg.norm(Y, axis=1, keepdims=True)
            S = compute_hash(W, Y, 1)
            M += Y.T.dot(S)
    	
        norm = np.linalg.norm(M, axis=0, keepdims=True)
        norm[norm == 0] = 1
        W = M / norm
    
    return W
\end{verbatim}

\end{document}